\documentclass{tlp}

 \submitted {5 January 2009}
 \revised {9 October 2009}
 \accepted{10 December 2009}

\newtheorem{theorem}{Theorem}

\newtheorem{corollary}{Corollary}
\newtheorem{example}{Example}
\newtheorem{remark}{Remark}

\newtheorem{lemma}{Lemma}

\usepackage{amsmath}
\usepackage{amsfonts}
\usepackage{amssymb}
\usepackage[mathscr]{eucal}
\usepackage{url}
\usepackage{times}
\usepackage{epic}
\usepackage{alg}
\usepackage{pstricks,pst-grad,pst-poly}
\usepackage{pstricks-add}
\usepackage{graphicx}
\usepackage{rotating}

\newcommand{\eoe}{ ~~~\mbox{$\Box$}}
\newcommand{\qed}{\hfill~}
\newcommand{\codetext}[1]{\mbox{\small \tt #1}}
\newcommand{\z}{\mathbb{Z}}

\newcommand{\BMV}{\mbox{$\mathcal{B}^{MV}$}}
\newcommand{\BMVi}{\mbox{$\mathcal{B}^{MV}_0$}}

\newcommand{\dom}{\mathrm{dom}}
\newcommand{\Eff}{\mathit{Eff}}

\newcommand{\op}{\mbox{\codetext{op}\hspace*{0.5ex}}}

\newcommand{\n}{\mathbf{N}}
\newcommand{\Clo}{{\tt Clo}}
\newcommand{\sito}{{\small\url{www.dimi.uniud.it/dovier/CLPASP}}}

\newcommand{\ine}{\mathsf{ine}}

\newcommand{\Sect}[1]{Section~\ref{#1}}
\newcommand{\Sects}[1]{Sections~\ref{#1}}
\newcommand{\Fig}[1]{Figure~\ref{#1}}
\newcommand{\Figs}[1]{Figures~\ref{#1}}

\newcommand{\interi}{\mathbb{Z}}
\newcommand{\codesize}{\small}
\newcommand{\FV}{\mathsf{fluents}}

\newcommand{\GFV}{\mathsf{\mbox{$\geqslant$-}fluents}}
\newcommand{\tab}{\hspace*{1.5em}}

\newcommand{\boxstandup}[1]{\begin{turn}{90}{\makebox(5,10){#1}}\end{turn}}
\newcommand{\boxstandupbis}[1]{\begin{turn}{90}{\makebox(35,10){#1}}\end{turn}}

\newenvironment{codice}[1]
{\begin{minipage}[t]{.99\textwidth}
 \tt \vspace*{-3ex}\scriptsize
 \setlength{\algrightmarginwidth}{1ex}\setlength{\algleftmarginwidth}{1ex}
 \begin{algtab}\algnonumber\setcounter{algline}{#1} \\
}
{\end{algtab}\end{minipage}}

\newenvironment{codicenonum}[1]
{\begin{minipage}[t]{.99\textwidth}
 \tt \vspace*{-3ex}\scriptsize
 \setlength{\algrightmarginwidth}{1ex}\setlength{\algleftmarginwidth}{1ex}
 \begin{algtab*}\algnonumber \\
}
{\end{algtab*}\end{minipage}}

\newcommand{\standup}[1]{\begin{turn}{90}{#1}\end{turn}}

\def\figurazero{
\psset{xunit=0.7}\psset{yunit=0.7}
\pspolygon[fillstyle=eofill,fillcolor=white,linewidth=1pt,linecolor=black](0,0)(8.0,0)(8.0,5.0)(0,5.0)
\pspolygon[fillstyle=eofill,linewidth=1pt,fillcolor=lightgray](4,1.8)(6,1.8)(6,2.8)(4,2.8)
\pscircle[fillstyle=hlines,linewidth=1.0pt](3,2.3){1}
\pscircle[fillstyle=vlines,linewidth=1.0pt](5,2.3){1}
\put(2.2,2.9){\mbox{\normalsize$\mathcal{F} \cup \neg \mathcal{F}$}}
\put(0.8,0.8){\mbox{\normalsize$S$}}
\put(4.5,0.8){\mbox{\normalsize$S'$}}
\put(3.3,1.5){\psframebox*{\mbox{\normalsize$E$}}}
\put(3.3,2.1){\psframebox*{\mbox{\normalsize$\bullet \ell$}}}
}

\def\figurauno{
\begin{psclip}{\pscircle[linestyle=none](5.0,2.5){1}}
\psframe[linestyle=none,fillstyle=solid,linecolor=lightgray,fillcolor=lightgray](2.0,0.0)(7.0,2.5)
\end{psclip}
\pscircle[linestyle=none](5.0,2.5){1}
\begin{psclip}{\pscircle[linestyle=none](3.8,2.5){1}}
\psframe[linestyle=none,fillstyle=solid,linecolor=lightgray,fillcolor=lightgray](2.0,2.5)(7.0,5)
\end{psclip}
\pscircle[linestyle=solid](5.0,2.5){1}
\pscircle[linestyle=solid](3.8,2.5){1}
\psline{-}(1.3,2.5)(7.5,2.5)
\psframe[framearc=.3](1.3,1.2)(7.5,3.8)
\put(2.9,3.3){\mbox{\small$v$}}
\put(5.7,3.3){\mbox{\small$v'$}}
\put(6.2,1.9){\mbox{\small$\mathcal{F}\setminus D$}}
\put(6.8,2.8){\mbox{\small$D$}}
\put(-0.7,3.4){\mbox{\small$\triangle(v,v',D)$\,:}}
}

\def\figuradue{
\begin{psclip}{\pscircle[linestyle=none](5.0,2.5){1}}
\psframe[linestyle=none,fillstyle=solid,linecolor=lightgray,fillcolor=lightgray](2.0,0.0)(7.0,3.0)
\end{psclip}
\pscircle[linestyle=none](5.0,2.5){1}
\begin{psclip}{\pscircle[linestyle=none](3.8,2.5){1}}
\psframe[linestyle=none,fillstyle=solid,linecolor=white,fillcolor=lightgray](2.0,3.0)(7.0,5)
\end{psclip}
\pscircle[linestyle=solid](3.8,2.5){1}
\pscircle[linestyle=solid](5.0,2.5){1}
\psline{-}(1.3,2.5)(1.6,2.5)
\psline{-}(2.5,2.5)(7.5,2.5)
\psline{-}(1.3,3.0)(6.8,3.0) \psline{-}(7.1,3.0)(7.5,3.0)
\psframe[framearc=.3](1.3,1.2)(7.5,3.8)
\put(2.9,3.3){\mbox{\small$v$}}
\put(5.7,3.3){\mbox{\small$v'$}}
\put(6.2,1.9){\mbox{\small$\mathcal{F}\setminus D$}}
\put(6.78,2.87){\mbox{\small$D$}}
\put(1.8,3.2){\mbox{\small$S$}}
\put(1.6,2.4){\mbox{\small$\mathcal{F}\setminus S$}}
\put(-0.7,3.4){\mbox{\small$\triangle(v,v',S)$}}
\put(-0.7,3.0){\mbox{\small$\mbox{for }S\subseteq D$\,:}}
}

\definecolor{rosa}{rgb}{1,0.4,0.7}
\definecolor{verde}{rgb}{0.5,1,0.5}

\def\figuratre{
\pscircle[linestyle=solid,fillstyle=solid,fillcolor=verde](1.5,4.5){1}
\begin{psclip}{\pscircle[linestyle=none,fillstyle=solid,fillcolor=verde](6.0,4.5){1}}
\pscircle[linestyle=solid,fillstyle=solid,fillcolor=rosa](2.8,15.8){11.5}
\pscircle[linestyle=solid,fillstyle=solid,fillcolor=white](10.0,-7.3){12}
\end{psclip}
\pscircle[linestyle=solid](6.0,4.5){1}
\begin{psclip}{\pscircle[linestyle=none,fillstyle=solid,fillcolor=verde](3.75,2.0){1}}
\pscircle[linestyle=solid,fillstyle=solid,fillcolor=rosa](0.55,13.3){11.5}
\pscircle[linestyle=solid](3.75,2.0){1}
\end{psclip}
\put(6.0,3.7){\mbox{$\mathcal{SL}$}}
\put(0.3,3.4){\mbox{\normalsize$u$}}
\put(7.0,3.4){\mbox{\normalsize$w$}}
\put(3.5,3.2){\mbox{\normalsize$v$}}
\put(5.3,5.0){\mbox{$E(a,u)$}}
\put(5.6,4.3){\mbox{Inertia}}
\put(3.1,2.5){\mbox{$E(a,u)$}}
\psline[linewidth=1.5pt,arrowsize=12pt,arrowinset=0.6,arrowlength=0.8]{->}(2.1,3.7)(3.05,2.75)
\psline[linewidth=1.5pt,arrowsize=12pt,arrowinset=0.6,arrowlength=0.8]{->}(4.4,2.76)(5.32,3.7)
\put(2.6,3.3){\mbox{\normalsize$a$}}
\put(4.8,2.9){\mbox{$\mathcal{SL}$}}
}

\title[Multi-valued Action Languages in CLP(FD)]
{Multi-valued Action Languages\\with Constraints in CLP(FD)\,\thanks{
This manuscript is an extended version of the paper
\emph{``Multi-valued Action Languages with Constraints in CLP(FD)''} in the
Proceedings of the International Conference on Logic Programming, pages 255--270,
Springer Verlag, 2007.}}

\author[A. Dovier, A. Formisano, E. Pontelli]
{AGOSTINO DOVIER\\
Universit{\`a} di Udine,\\
Dipartimento di Matematica e Informatica\\
{\email{dovier@dimi.uniud.it}}
\and
ANDREA FORMISANO\\
Universit{\`a} di Perugia,\\
Dipartimento di Matematica e Informatica\\
{\email{formis@dmi.unipg.it}}
\and
ENRICO PONTELLI\\
New Mexico State University,\\
 Department of Computer Science\\
{\email{epontell@cs.nmsu.edu}}
}

\begin{document}
\maketitle

\begin{abstract}
Action description languages, such as $\mathcal{A}$ and $\mathcal{B}$
\cite{GL98}, are expressive instruments introduced for
formalizing planning domains and planning problem instances.
The paper starts by proposing a methodology to encode
an action language (with conditional effects and static
causal laws), a slight variation of
$\mathcal{B}$, using
 \emph{Constraint Logic Programming over Finite Domains}.
The approach is then generalized to raise  the use of constraints
to the level of the
action language itself. A prototype implementation has been
developed, and the preliminary results are presented and discussed.
%%%%%%%%% TO BE CUT
\\
To appear in \emph{Theory and Practice of Logic Programming}\, (TPLP).
\end{abstract}

\section{Introduction}
The construction of intelligent agents that can be
effective  in real-world
environments has been a goal of  researchers from the very
first days of Artificial Intelligence.
It has long been recognized that an intelligent agent must be able to
\emph{acquire}, \emph{represent}, and \emph{reason} with knowledge. As such, a
{\em reasoning component} has been an inseparable part of most
agent architectures in the literature.

Although the underlying representations and implementations
may vary between agents, the reasoning component of an agent
is often responsible for making decisions that are critical to its
existence.

Logic programming languages offer many properties that make
them very suitable as knowledge representation languages. Their
declarative nature supports  the modular development of
provably correct reasoning modules~\cite{Baral}. Recursive
definitions can be easily expressed and reasoned upon. Control
knowledge and heuristic information can be declaratively
and incrementally introduced
in the reasoning process. Furthermore, many logic programming
languages offer a natural support for non-monotonic reasoning,
which is considered essential for common-sense reasoning~\cite{Lif99}.
These features, along
with the presence of  efficient
inference engines \cite{Apt2003,MS98,SimonsTesi,GLM04,claspcitaz}, make
logic programming an attractive paradigm for knowledge representation
and reasoning.

\medskip

In the context of knowledge representation and reasoning, a very
important application of logic programming has been in
the domain of \emph{reasoning about actions and change} and, more
specifically, \emph{planning}. Planning problems have been effectively
encoded using Answer Set Programming
(ASP)~\cite{Baral}---where distinct answer sets represent different
trajectories leading to the desired goal. Other logic programming
paradigms, e.g., \emph{Constraint Logic Programming over Finite
Domains (CLP(FD))}~\cite{CLP1,Apt2003}, have been used less frequently to handle problems in
reasoning about actions (e.g.,~\cite{rei01,thielscher}). Comparably more
emphasis has been placed in encoding planning problems as (non-logic
programming) constraint satisfaction problems~\cite{lopezbacchus}.

Recent proposals on representing and reasoning about actions and change
have relied on the use of concise and high-level languages, commonly
referred to as \emph{action description languages}; some well-known
examples include the languages $\mathcal{A}$ and  $\mathcal{B}$~\cite{GL98}
and extensions like $\mathcal{K}$~\cite{Eiter-K} and $\cal ADC$~\cite{BaralS02}.
Action languages allow one to write propositions that describe
the effects of actions on states, and to create queries to
infer properties of the underlying transition system.
An \emph{action domain description} is a specification of a
planning domain using an action language.

\medskip

The goal of this work is to explore the relevance of constraint
solving and constraint logic programming~\cite{MS98,Apt2003} in
dealing with action languages and planning. The push towards this
exploratory study came from  recent
investigations~\cite{DFPiclp05,JETAI08} aimed at  comparing the
practicality and efficiency of  answer set programming versus
constraint logic programming in solving various combinatorial and
optimization problems. The study indicated that CLP offers a valid
alternative, especially in terms of efficiency, to ASP when dealing
with planning problems. Furthermore, CLP offers the flexibility of
programmer-developed search strategies and the ability to handle
numerical constraints.

The first step, in this paper, is to illustrate a scheme that directly
processes an action description specification, in a
language similar to $\mathcal{B}$~\cite{GL98},
producing a CLP(FD) program that can be used to compute
solutions to the planning problem. Our encoding has some
similarities to the one presented by Lopez and Bacchus \cite{lopezbacchus}, although
we rely on constraint logic programming instead of
plain constraint satisfaction (CSP), and our action language supports
static causal laws and non-determinism---while the work of Lopez
and Bacchus is restricted to STRIPS-like specifications.

While the first step relies on using constraints to compute solutions to
a planning problem, the second step brings the expressive
power of constraints to the level of the action language, by
allowing multi-valued fluents and constraint-producing actions to be used
in the domain specification. The
extended action language (named $\BMV$) can be as easily supported by the
CLP(FD) framework, and it allows a declarative encoding of problems
involving actions with resources, delayed effects, and maintenance
goals. These ideas have been developed in a  prototype, and some
preliminary experiments are reported.

We believe that the use of CLP(FD) can greatly facilitate the transition
of declarative extensions of action languages to concrete and effective
implementations, overcoming some inherent limitations (e.g.,
efficiency and limited handling of numbers) of other logic-based systems
(e.g.,~ASP).

\bigskip

The presentation is organized as follows. The first part of our
paper (\Sects{Blanguage} and \ref{implemB}) provides an overview of the
action language $\mathcal{B}$ and illustrates our approach to modeling
problem specifications in $\mathcal{B}$ using constraints and constraint logic programming.
\Sect{motivat} provides motivations for the proposed multi-valued
extensions.
\Sect{BMVsyntax} introduces the full syntax of the new language \BMV.
The action language \BMV\ expands the previous language to a language with
constraints and multi-valued fluents,
that enables the use of dynamic and static causal laws (a.k.a. state constraints),
executability conditions, and non-Markovian forms of reasoning with arbitrary
relative or absolute references to past and future points in time.
The semantics and the abstract implementation of \BMV\ is incrementally developed
in \Sect{BMVsemantics}, where we first consider a sub-language
not involving non-Markovian references, and later we extend it to
the full \BMV.
A concrete implementation in CLP(FD) is described in \Sect{concreteImpleBMV}, and
an experimental evaluation is discussed in \Sect{sec:experimental}.
\Sect{sec:related} presents an overview of related efforts appeared in the literature,
while \Sect{sec:endofit} presents conclusions and the directions for future
investigation.

\section{The Action Language $\cal B$}\label{Blanguage}

\emph{``Action languages are formal models of parts of the natural language
that are used for talking about the effects of actions''}~\cite{GL98}.
Action languages are used to define \emph{action descriptions}
that embed knowledge to formalize planning problems.
In this section, we use the same  variant
of the language $\mathcal{B}$ used
in~\cite{SON01}---see also \Sect{sec:related} for a comparison.
With a slight abuse of notation, we simply refer to
this language as $\mathcal{B}$\/.

\subsection{Syntax of $\mathcal{B}$}\label{syntaxOfB}

An action signature consists of a set $\mathcal{F}$  of
\emph{fluent} names,
a set $\mathcal{A}$ of \emph{action} names,
and a set $\mathcal{V}$ of values for fluents in $\mathcal{F}$.
In this section, we consider Boolean fluents, hence
$\mathcal{V} = \{0,1\}$.\footnote{For simplicity, we use $0$ to
denote \emph{false} and $1$ to denote \emph{true}.
Consequently, we often say that a fluent is
true (resp., false)
if its value is $1$ (resp.,~$0$).}
A \emph{fluent literal} is either a fluent $f$ or its negation
$\codetext{neg}(f)$.
Fluents and actions are concretely represented by
\emph{ground} atomic formulae $p(t_1,\ldots,t_m)$
from an underlying logic language~$\mathcal{L}$.
For simplicity, we assume that the set of terms is finite---e.g., either
there are no function symbols in $\mathcal{L}$, or the use of functions
symbols is restricted, for instance by imposing a fixed maximal depth on the
nesting of terms, to avoid the creation of arbitrary complex terms.

The language $\mathcal{B}$ allows us to specify an
\emph{(action) domain description} $\mathcal{D}$. The core components of
a domain description are its \emph{fluents}---properties used to describe
the state of the world, that may dynamically change in response to execution
of actions---and \emph{actions}---denoting how an agent can affect the state
of the world.
Fluents and actions are introduced by assertions of the forms
\codetext{fluent($f$)} and \codetext{action($a$)}.
An action description $\mathcal{D}$ relates
actions, states, and fluents using  axioms of the following types
---where \codetext{[list-of-conditions]} denotes a list of fluent
literals:\footnote{We will sometimes write \texttt{true} as a synonymous for
the empty list of conditions.}
\begin{itemize}
\item
\codetext{causes($a$, $\ell$, [list-of-conditions])}: this axiom
 encodes
a dynamic causal law, describing the effect (i.e., truth assignment to
the fluent literal $\ell$)
of the execution of action $a$ in a state satisfying the
given conditions
\item
\codetext{caused([list-of-conditions],~$\ell$)}: this axiom
 describes
a static causal law---i.e., the fact that the fluent literal $\ell$
is true in any state satisfying the given preconditions.
\end{itemize}
Moreover, preconditions can be imposed on the executability of actions by means of
assertion of the forms:
\begin{itemize}
\item
\codetext{executable($a$, [list-of-conditions])}: this axiom
 asserts that, for the action $a$ to be executable, the given conditions have to be satisfied in the current state.
\end{itemize}

A \emph{domain description} is a set of
static causal laws, dynamic laws, and executability conditions.
A specific \emph{planning problem} $\langle\mathcal{D}, \mathcal{O}\rangle$ contains a domain
description $\mathcal{D}$ along with a set $\mathcal{O}$ of \emph{observations} describing
the \emph{initial state} and the \emph{desired goal}:
\begin{itemize}
\item
\codetext{initially($\ell$)}  asserts that the fluent
literal $\ell$ is true in the initial state
\item
\codetext{goal($\ell$)} asserts that the goal requires the
fluent literal $\ell$ to be true in the final state.
\end{itemize}

In the specification of an action theory,
we can take advantage of a Prolog-like syntax to express
in a more succinct manner the laws of the theory.
For instance, to assert that in the initial state all fluents are true,
we can simply write the following rule:
$$\codetext{initially(F) :- fluent(F)}.$$
instead of writing a fact \codetext{initially($f$)}
for each possible fluent~$f$. Remember that
the notation
$H \,:-\, B_1, \dots, B_k$
is a syntactic sugar for the logical formula
$$\forall X_1 \cdots X_n (B_1 \wedge \dots \wedge B_k \rightarrow H)$$
where $X_1,\dots,X_n$ are all the variables present in
$H, B_1, \dots, B_k$.

\begin{figure}[th]
\begin{center}
\begin{codicenonum}{0}
\\
\%\% Some Type Information\\
barrel(5). \\
barrel(7). \\
barrel(12).\\
liter(0).  \\
liter(1).  \\
~~\vdots ~~\\
liter(12).\\
\\
\%\% Identification of the fluents\\
fluent(cont(B,L)):- barrel(B), liter(L),  L $\leq$ B.\\
\\
\%\% Identification of the actions\\
action(fill(X,Y)):- barrel(X), barrel(Y), X $\neq$ Y.\\
\\
\%\% Dynamic causal laws\\
causes(fill(X,Y), cont(X,0), [cont(X,LX), cont(Y,LY)]) :-  \\
\tab        action(fill(X,Y)),
        fluent(cont(X,LX)),\\
\tab        fluent(cont(Y,LY)),
        Y-LY $\geq$ LX.\\
causes(fill(X,Y), cont(Y,LYnew), [cont(X,LX), cont(Y,LY)]) :-  \\
\tab        action(fill(X,Y)),
        fluent(cont(X,LX)),\\
\tab
        fluent(cont(Y,LY)),
        Y-LY $\geq$ LX,
        LYnew is LX+LY.\\
causes(fill(X,Y), cont(X,LXnew), [cont(X,LX), cont(Y,LY)]) :-  \\
\tab        action(fill(X,Y)),
        fluent(cont(X,LX)),\\
\tab
        fluent(cont(Y,LY)),
        LX $\geq$ Y-LY,
        LXnew is LX-Y+LY.\\
causes(fill(X,Y), cont(Y,Y), [cont(X,LX), cont(Y,LY)]) :-  \\
\tab        action(fill(X,Y)),
        fluent(cont(X,LX)),\\
\tab
        fluent(cont(Y,LY)),
        LX $\geq$ Y-LY.\\
\\
\%\% Executability conditions\\
executable(fill(X,Y), [cont(X,LX), cont(Y,LY)]) :-  \\
\tab        action(fill(X,Y)),
        fluent(cont(X,LX)),\\
\tab
        fluent(cont(Y,LY)),
        LX > 0,
        LY < Y.\\
\\
\%\% Static causal laws
caused([cont(X,LX)], neg(cont(X,LY))) :-  \\
\tab    fluent(cont(X,LX)),
        fluent(cont(X,LY)),\\
\tab
        barrel(X), liter(LX), liter(LY), LX$\neq$LY.\\
\\
\%\% Description of the initial and goal state\\
initially(cont(12,12)).\\
initially(cont(7,0)). \\
initially(cont(5,0)).\\
goal(cont(12,6)). \\
goal(cont(7,6)). \\
goal(cont(5,0)).
\end{codicenonum}
\end{center}
\caption{\label{Bool_Barrels}$\mathcal B$ description of the 12-7-5 barrels problem.}
\end{figure}

\begin{example}\label{exempiobarrelsB}
\Fig{Bool_Barrels} presents an encoding
of the three-barrel planning problem using the language~$\mathcal{B}$.
There are three barrels of capacity $N$ (an even number),
$N/2+1$, and $N/2-1$, respectively.
At the beginning, the largest barrel is full of wine while the other
two are empty.
We wish to reach a state in which the two larger barrels
contain the same amount of wine. The only permissible action is to
pour wine from one barrel to another,
until the latter is full or the former is empty.
\Fig{Bool_Barrels}
shows the encoding of the problem for $N=12$.
Notice that we also require that the smallest
barrel is empty at the end.
\eoe
\end{example}

\subsection{Semantics of $\mathcal{B}$}\label{semPerB}

If $f \in \mathcal{F}$ is a fluent, and $S$ is a set of fluent literals,
we say that $S \models f$ if and only if $f \in S$ and
$S \models \codetext{neg}(f)$ if and only if $\codetext{neg}(f) \in S$.
A list of literals
$L=[\ell_1,\ldots,\ell_m]$
denotes a conjunction of literals, hence
$S\models L$ if and only if
$S \models \ell_i$ for all $i \in \{1,\ldots,m\}$.
We denote with $\neg S$ the set
$\{f \in \mathcal{F}: \codetext{neg}(f) \in S\} \cup \{ \codetext{neg}(f) : f \in S \cap \mathcal{F}\}.$
A set of fluent literals is \emph{consistent} if there is no fluent
$f$ s.t.
$S \models f$ and $S \models \codetext{neg}(f)$.
If $S \cup \neg S \supseteq \mathcal{F}$ then $S$ is \emph{complete}.
A set $S$ of literals is \emph{closed} w.r.t. a set
of static laws
$\mathcal{SL} = \{\codetext{caused}(C_1,\ell_1), \ldots,
\codetext{caused}(C_m,\ell_m)\}$,
if for all $i \in \{1,\ldots,m\}$ it holds that $S \models C_i$ implies
$S \models \ell_i$.
The set $\Clo_{\mathcal{SL}}(S)$ is defined as the smallest set of
literals containing $S$ and closed w.r.t. $\mathcal{SL}$.
$\Clo_{\mathcal{SL}}(S)$ is uniquely determined
and
not necessarily consistent.

The semantics of an action language
on the action signature
$\langle \mathcal{V} ,\mathcal{F} ,\mathcal{A} \rangle$
is given in terms of a transition system $\langle \mathcal{S}, \nu, R \rangle$
\cite{GL98}, consisting of
a set $\mathcal{S}$ of states,
a total interpretation  function $\nu:\mathcal{F}\times \mathcal{S}
\rightarrow \mathcal{V}$
(in this section $\mathcal{V} = \{0,1\}$), and
 a transition relation $R \subseteq \mathcal{S} \times \mathcal{A} \times
\mathcal{S}$.

Given a transition system $\langle \mathcal{S}, \nu, R \rangle$ and a state $s
\in \mathcal{S}$,
let:
  $$\begin{array}{rcl}
  Lit(s) & = & \{ f \in \mathcal{F} \,:\,
                  \nu(f,s) = 1 \} \cup
                  \{ \codetext{neg}(f) \,:\, f \in \mathcal{F}, \, \nu(f,s) = 0\}.
  \end{array}$$
Observe that $Lit(s)$ is consistent and complete.

Given a set of dynamic laws
$\{\codetext{causes}(a, \ell_1, C_1)$, $\ldots$, $\codetext{causes}(a, \ell_m, C_m)\}$
for the action $a \in \mathcal{A}$ and a state $s \in \mathcal{S}$,
we define the \emph{(direct) effects of $a$ in $s$} as follows:
$$E(a,s)  =  \{ \ell_i : 1 \leqslant i \leqslant m, Lit(s) \models C_i \}.$$

\noindent The action $a$ is said to be \emph{executable}
in a state $s$ if it holds that
\begin{equation}\label{execCond}
Lit(s)\models\bigvee_{i=1}^{h} C_i,
\end{equation}
where
$\codetext{executable}(a, C_1)$, $\ldots$, $\codetext{executable}(a, C_h)$
for $h>0$,
are the executability axioms for the action $a$ in $\mathcal{D}$.
Observe that multiple executability axioms for the same action $a$ are
considered disjunctively.
Hence, for each action $a$, at least one executable axiom must be present in the
action description.\footnote{Observe that even if an action is
``executable'', its execution may lead to
an inconsistent state (which effectively prevents the use of
such action in that context). Even though ``enabled'' would be a better term
to use for an action that can be executed in a state, we prefer
to maintain the same terminology as used for $\mathcal{B}$ in
\cite{SON01}---see also Remark~\ref{execVSnonexec}.}

Let $\mathcal{D}$ be an action description defined on the action signature
$\langle \mathcal{V} ,\mathcal{F} ,\mathcal{A} \rangle$,
composed of dynamic laws $\mathcal{DL}$,
executability conditions $\mathcal{EL}$,
and static causal laws $\mathcal{SL}$.

The transition system $\langle \mathcal{S}, \nu, R \rangle$ \emph{described by}
$\mathcal{D}$ is
a transition system such that:
\begin{itemize}
\item
$\mathcal{S}$ is the set of all states $s$ such that
 $Lit(s)$ is closed w.r.t. $\mathcal{SL}$;
\item
$R$ is the set of all triples $\langle s,a,s'\rangle$ such that
$a$ is executable in $s$ and
\begin{eqnarray}
Lit(s') &=&\Clo_{\mathcal{SL}}(E(a,s) \cup (Lit(s) \cap Lit(s')))
\label{eqbool}
\end{eqnarray}
\end{itemize}

Let  $\langle \mathcal{D}, \mathcal{O}\rangle$ be a planning problem instance,
where $\{\ell\:|\: \codetext{initially}(\ell) \in  \mathcal{O}\}$ is
a consistent and complete set of fluent literals.
A \emph{trajectory} in $\langle \mathcal{S}, \nu, R \rangle$ is a sequence
$$\langle s_0, a_1, s_1, a_2, \cdots, a_{\n}, s_{\n}\rangle$$
such that
$\langle s_{i},a_{i+1},s_{i+1}\rangle \in R$
for all $i \in \{0,\ldots,\n-1\}$.

A sequence of actions
$\langle a_1,\ldots, a_{\n}\rangle$ is a solution (a \emph{plan})
to the planning problem $\langle \mathcal{D}, \mathcal{O}\rangle$
if there is
a trajectory $\langle s_0, a_1, s_1, \ldots, a_{\n}, s_{\n}\rangle$ in
$\langle \mathcal{S}, \nu, R\rangle$
such that:
\begin{itemize}
\item $Lit(s_0) \models r$ for each $\codetext{initially}(r) \in \mathcal{O}$,
and
\item $Lit(s_{\n}) \models \ell$ for each $\codetext{goal}(\ell)\in \mathcal{O}$.
\end{itemize}
The plans characterized in this definition are \emph{sequential}---i.e.,
we disallow concurrent actions. Observe also that the desired plan
length $\n$ is assumed to be given.

\begin{remark}\label{sequentiality}
In this paper we focus on sequential plans only. Hence, we assume that
only one action is executed in each state transition composing a given trajectory.

Note that the constraint-based encoding we will propose in the rest of this
manuscript can be easily adapted to  deal with concurrent actions.
Nevertheless, we have opted to ignore this aspect in this manuscript, to avoid
further complications of notation, and dealing with issues
of concurrency goes beyond the scope of this paper.
The interested reader is referred to~\cite{DFPlpnmr09} for some further considerations
on this matter.
\end{remark}

\begin{remark}\label{execVSnonexec}
Notice that the satisfaction of~(\ref{execCond})
is just a necessary requirement for the executability of an action
and it might not represent a sufficient precondition.
Indeed, as far as the definition of transition system is considered,
it is easy to see that, even if~(\ref{execCond}) is satisfied for certain $a$ and $s$,
the execution of $a$ in $s$ might be inhibited because of the contradictory effects
of the causal laws.
A simple example is represented by the following action
description~$\mathcal{D}$:
\begin{center}
\begin{tabular}{l}
\codetext{executable($a$,[]).}\\
\codetext{causes($a$,$f$,[]).}\\
\codetext{causes($a$,neg($f$),[]).}
\end{tabular}
\end{center}
The action  $a$ is always executable (according to its
executability law), but
the execution of $a$ would yield an inconsistent situation.
Indeed, the execution of $a$ does not
correspond to any state transition in the transition system described
by~$\mathcal{D}$.

The above example also suggests a possible extension of the action description language
that involves laws of the form
$$\codetext{nonexecutable($a$, $D$)}.$$
\noindent
The semantics for
such an extended action language can be defined by
replacing the condition~(\ref{execCond}),
with the following one:
$$
Lit(s)\models\bigvee_{i=1}^{h} C_i \wedge \neg\bigvee_{j=1}^{k} D_j,
$$
where
\codetext{executable($a$, $C_1$)}, $\ldots$, \codetext{executable($a$, $C_h$)}
and
\codetext{nonexecutable($a$, $D_1$)}, $\ldots$, \codetext{nonexecutable($a$, $D_k$)},
for $h>0$ and $k\geqslant 0$,
are defined for the action $a$.
Thus, the action $a$ is executable only if at least one of the $C_i$s is satisfied
and all $D_j$s are unsatisfied in the state $s$.

An alternative interpretation of the \codetext{nonexecutable}
axioms can be adopted. Namely, the law \codetext{nonexecutable($a$, $D$)}
can be considered simply as shorthand
for the pair of dynamic causal laws \codetext{causes($a$, $f$, $D$)} and
\codetext{causes($a$, \codetext{neg}($f$), $D$)}.
(Actually, this possibility also applies to the languages proposed in~\cite{GL98}).

This shows that (non)executability laws do not increase the
expressive power of the action language.
Nevertheless, the availability of both types of laws
permits the direct and explicit formalization of
preconditions for actions execution.
\end{remark}

\section{Modeling $\mathcal{B}$ and Planning Problems Using Constraints}\label{implemB}

Let us describe  how action descriptions
are mapped to finite domain constraints.
We will focus on how constraints can be used
to model the possible transitions from each individual state
of the transition system.

\newcommand{\puntinata}[4]{\dottedline{3}({#1},{#2})({#3},{#4})}
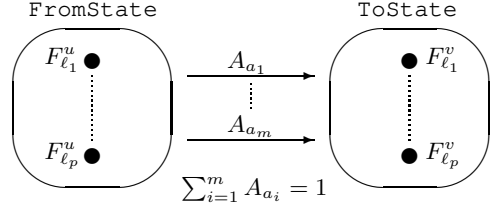
\begin{figure}
\begin{center}
\setlength{\unitlength}{0.6pt}
\begin{picture}(400,150)(0,35)
\thinlines
\put(60,158){\codetext{FromState}}
\put(100,100){\oval(100,100)}
\put(100,130){\circle*{10}}
\put(70,128){$F_{\ell_{1}}^u$}
\put(100,70){\circle*{10}}
\put(70,68){$F_{\ell_{p}}^u$}
\puntinata{100}{80}{100}{120}
\put(267,158){\codetext{ToState}}
\put(300,100){\oval(100,100)}
\put(300,130){\circle*{10}}
\put(310,128){$F_{\ell_{1}}^v$}
\put(300,70){\circle*{10}}
\put(310,68){$F_{\ell_{p}}^v$}
\puntinata{300}{80}{300}{120}
\put(160,120){\vector(1,0){80}}
\put(185,125){$A_{a_1}$}
\puntinata{200}{100}{200}{115}
\put(160,80){\vector(1,0){80}}
\put(185,85){$A_{a_m}$}
\put(156,45){$\sum_{i=1}^m A_{a_i} = 1$}
\end{picture}
\end{center}
\caption{\label{act_fig}Action constraints from state to state. (The states are described by $p$ fluents,
$\ell_{1},\ldots,\ell_{p}$, and one among $m$ possible actions is executed.)}
\end{figure}

\subsection{Modeling an Action Theory as Constraints}\label{modelingB}

Let us consider a domain description $\mathcal{D}$ and the  state transition
system described by $\mathcal{D}$.
Let us also denote with  $u$ and $v$ the starting and ending states of an
arbitrary
transition of such a system. We assert constraints that relate the
truth value of fluents in $u$ and $v$. This is intuitively
illustrated in \Fig{act_fig},
where $u=\codetext{FromState}$ and $v=\codetext{ToState}$.\footnote{For the sake
of readability, the two variables named \texttt{FromState} and \texttt{ToState}
are also used in the concrete implementation of $\mathcal{B}$
(cf., \Sect{CLPMappingB} and \Fig{setonefluent}).}

A Boolean variable is introduced to describe the
truth value of each fluent literal in a state.
The value of a fluent literal $\ell$  in $u$
is represented by the variable $F_{\ell}^u$;
analogously, its value in the destination
state $v$ is represented by the variable
$F_{\ell}^v$. For the sake of simplicity, we will
freely refer to these variables as Boolean entities---and
compose them with logical connectives to form Boolean
expressions---as well as 0/1 variables---and compose
them with arithmetic operators. Concrete CLP(FD) systems,
e.g., SICStus, ECLiPSe, and BProlog,\footnote{Web sites for some CLP(FD) systems. SICStus:~{\url{www.sics.se/sicstus.html}},
ECLiPSe:~{\url{http://87.230.22.228/}},
BProlog:~{\url{http://www.probp.com/}}}
enable this type of alternative
perspectives, providing basic primitive constraints (e.g.,
\verb"#=" and \verb"#>") and Boolean compositions of constraints.

Given a conjunction of literals $\alpha = [\ell_1, \ldots, \ell_m]$
we will denote with ${\alpha}^u$ the expression
$F_{\ell_1}^u \wedge \ldots \wedge F_{\ell_m}^u$.
We will also introduce, for each action $a$, a Boolean
variable $A_a^u$, representing whether the action
is executed or not in the transition from
$u$ to $v$ under consideration.

Given a specific fluent literal $\ell$, we develop constraints
that determine when $F_{\ell}^v$ is true and  false.
Let us consider the dynamic
causal laws that have $\ell$ as a consequence:
\begin{center}
\begin{tabular}{rcccl}
$\codetext{causes}(a_{i_{\ell,1}}, \ell, \alpha_{{\ell,1}})$
& &
$\cdots$ &  &
$\codetext{causes}(a_{i_{\ell,m_{\ell}}}, \ell, \alpha_{{\ell,m_{\ell}}})$
\end{tabular}
\end{center}
Let us also consider the static causal laws related to ${\ell}$
\begin{center}
\begin{tabular}{rcccl}
$\codetext{caused}(\gamma_{{\ell,1}},\ell)$
&  &
$\cdots$ &
 &
$\codetext{caused}(\gamma_{{\ell,h_{\ell}}},\ell)$
\end{tabular}
\end{center}
Finally, for each action $a$ we will have its executability
conditions:
\begin{center}
\begin{tabular}{rcccl}
$\codetext{executable}(a, \delta_{a,1})$ &
 & $\cdots$ & &
$\codetext{executable}(a, \delta_{a,p_{a}})$
\end{tabular}
\end{center}

\begin{figure}
\begin{center}
{\small\fbox{\begin{minipage}[c]{.87\textwidth}
\begin{eqnarray}
\codetext{Dyn}_{\ell}^u & \leftrightarrow &
\bigvee_{j=1}^{m_{\ell}}  (\alpha_{{\ell,j}}^u \wedge A_{a_{i_{\ell,j}}}^u) \label{cieffe2}\\
\codetext{Stat}_{\ell}^v & \leftrightarrow &  \bigvee_{j=1}^{h_{\ell}} \gamma_{{\ell,j}}^v \label{cieffe3}\\
\codetext{Fired}_{\ell}^{u,v} & \leftrightarrow & \codetext{Dyn}_{\ell}^u \vee
\codetext{Stat}_{\ell}^v \label{cieffe4}\\
&& \neg \codetext{Fired}_{\ell}^{u,v} \vee \neg \codetext{Fired}_{\bar{\ell}}^{u,v}
\label{per}\label{cieffe5}\\
F_{\ell}^v & \leftrightarrow & \codetext{Fired}_{\ell}^{u,v} \vee (\neg
\codetext{Fired}_{\bar{\ell}}^{u,v}
\wedge F_{\ell}^u) \label{fondamentale}\label{cieffe6}
\end{eqnarray}
\end{minipage}}}
\end{center}
\caption{\label{cieffe}The constraint $C_{\ell}^{u,v}$ for the fluent literal $\ell$ (cf.,~\Sect{modelingB}).}
\end{figure}

\noindent
\Fig{cieffe} describes the Boolean constraints
that can be used in encoding the relations that determine
the truth value of the fluent literal ${\ell}$. In the table,
we denote with $\bar{\ell}$ the complement of literal $\ell$, i.e., if $\ell$ is
the fluent $f$, then $\bar{\ell}$ is $\codetext{neg}(f)$, while if $\ell$ is the literal
$\codetext{neg}(f)$ then $\bar{\ell}$ is the fluent $f$. The intuitive meaning
of the constraints is as follows:
\begin{description}
\item[(\ref{cieffe2})] This constraint states that  dynamic
        causal laws making
        $\ell$ true  can fire if
        their conditions are satisfied and the
        corresponding actions are chosen for execution.
\item[(\ref{cieffe3})] This constraint captures the fact that
        at least one of the
        static causal laws  that make $f$ true  is applicable.
\item[(\ref{cieffe4})] This constraint
        expresses the fact
        that a fluent literal $\ell$
        can be made true during a transition form state $u$ to state $v$, either by
a dynamic causal law (determined by $\codetext{Dyn}^u_{\ell}$)
        or
a static causal law (determined by $\codetext{Stat}^v_{\ell}$).
\item[(\ref{cieffe5})] This constraint is used to guarantee consistency
        of the action theory---in no situations a fluent and its complement
        are both made true.
\item[(\ref{cieffe6})] This constraint expresses the fact that a fluent
        literal  $\ell$
        is true in the destination state if and only if it is made true (by a static
        or a dynamic causal law) or if is true in the initial state and its
        truth value is not modified by the transition (i.e., inertia).
        Observe
        the similarity between this constraint and the successor state
        axiom commonly encountered in situation calculus~\cite{sitcalc}.
\end{description}
We will denote
with $C_{\ell}^{u,v}$ the conjunction of such constraints.

Given an action domain specification over the signature
$\langle {\cal V}, {\cal F}, {\cal A}\rangle$ and two states $u$ and $v$,
we introduce the system of constraints $C_{\mathcal{F}}^{u,v}$ which includes:
\begin{itemize}
\item for each fluent literal
        $\ell$ in the language of $\mathcal{F}$, the constraints $C_{\ell}^{u,v}$.
\item the constraint
\begin{eqnarray}
\label{unica_azione}
\sum_{a \in \mathcal{A}} A_a^u= 1
\end{eqnarray}
\item for each action $a \in \mathcal{A}$, the
constraints
\begin{equation}\label{execaction}
A_a^u \rightarrow \bigvee_{j=1}^{p_{a}} \delta_{a,j}^u.
\end{equation}
\end{itemize}
Notice that the sequentiality of the plan if imposed through the constraint~(\ref{unica_azione}),
while constraint~(\ref{execaction}) reflects actions' executability conditions.

\subsection{Soundness and Completeness Results}\label{SoundCompleOfB}
Let us proceed with the soundness and completeness proofs of the constraint-based encoding.
Consider a state transition  from the state $u$ to the state $v$
and the corresponding constraint $C_f^{u,v}$ described earlier.

Let $S=Lit(u)$ and $S'=Lit(v)$ be the sets
of fluent literals that hold in $u$ and $v$, respectively.
Note that, from any specific $S$
(resp., $S'$), we can obtain a consistent assignment $\sigma_S$
(resp., $\sigma_{S'}$) of truth values for all the variables
$F_f^u$ (resp., $F_f^{v}$) of $u$
(resp., $v$). Conversely, each truth assignment $\sigma_S$
(resp., $\sigma_{S'}$) for all variables $F_f^u$ (resp.,
$F_f^{v}$)  corresponds to a
 consistent and complete set of fluents $S$ (resp., $S'$).

Regarding  the occurrence of actions,
recall that in each state transition a single action $a$ occurs
and its occurrence is encoded by a specific Boolean variable,~$A_a^u$.
Let $\sigma_a$ denote the assignment of truth values for such variables
such that $\sigma_a(A_a^u)=1$ if and only if
$a$ occurs in the state transition from $u$ to $v$.\footnote{We will
use mapping applications either as $\sigma(X)$ or in postfix
notation as $X\sigma$.}
Note that the domains of $\sigma_S$, $\sigma_{S'}$, and $\sigma_a$
are disjoint, so we can safely
denote with $\sigma_S\circ\sigma_{S'}\circ\sigma_a$ the composition of the three assignments.
With a slight abuse of notation, in what follows we will denote with $E$ the
direct effects $E(a,u)$ of an action $a$ in $u$. Observe that $E \subseteq S'$.

Theorem~\ref{Completeness} states the completeness of the system of constrains
introduced in \Sect{modelingB}.
It asserts that for any given  $\mathcal{D} = \langle \mathcal{DL},\mathcal{EL},\mathcal{SL}\rangle$,
if a triple $\langle u , a, v \rangle$ belongs to the transition system described by $\mathcal{D}$,
then the assignment $\sigma = \sigma_S\circ\sigma_{S'}\circ\sigma_a$ satisfies the
constraint $C_{\mathcal{F}}^{u,v}$.

\begin{theorem}[Completeness]\label{Completeness}
Let $\mathcal{D} = \langle \mathcal{DL},\mathcal{EL},\mathcal{SL}\rangle$.
If $\langle u , a, v \rangle$ belongs to the transition system described by $\mathcal{D}$,
then $\sigma_{S}\circ\sigma_{S'}\circ\sigma_a$
is a solution of the constraint $C_{\mathcal{F}}^{u,v}$.
\end{theorem}
\begin{proof}
In constraints~(\ref{cieffe2})--(\ref{cieffe6}) of
\Fig{cieffe} and
(\ref{unica_azione})--(\ref{execaction}) defined at the end of Subsection \ref{modelingB},
a number of auxiliary constraint variables
are defined, whose values are uniquely determined once the values of the fluents are
assessed. In other words, when $S$, $S'$, and $a$ are fixed, the substitution
$\sigma_S\circ\sigma_{S'}\circ\sigma_a$ uniquely determines the value of the
right-hand sides of the constraints $(\ref{cieffe2})$--$(\ref{cieffe4})$.
To prove the theorem, we need to verify that
if $S' = \Clo_{\mathcal{SL}}(E \cup (S \cap S'))$, then the
constraints $(\ref{cieffe5})$ and $(\ref{cieffe6})$ along with the constraints about the
action variables $A_a^u$ (i.e., constraints of the form~(\ref{unica_azione})
and~(\ref{execaction})) are satisfied for every fluent $f$.

\medskip

Let us observe that (\ref{unica_azione}) is equivalent to say that
if $A_a$ is true ($A_a=1$) then $A_b$ is false for all $b \neq a$.
Moreover, it also states that if all $A_b$ for $b \neq a$ are false
then $A_a$ is true.
Namely, (\ref{unica_azione}) is equivalent to the conjunction,
for $a \in \mathcal{A}$ of:
$$A_a \leftrightarrow \bigwedge_{b \in \mathcal{A} \setminus \{b\} } \neg A_b$$

Let us start by looking at the action occurrence. Let $a$ be the action executed
in state $u$, thus $\sigma_a = \{A_a^u/1\}\cup \{A_b^u/0\:|\: b\neq a\}$.
Hence, (\ref{unica_azione}) is satisfied by $\sigma_a$.

Similarly, since the semantics require that actions are executed only
if the executability conditions are satisfied, it holds that $S\models \delta_{a,h}$
(for at least one $h\in\{1,\ldots,p_{a}\}$, corresponding to a condition
$\codetext{executable}(a, \delta_{{a},h})$ in $\mathcal{SL}$).
This quickly leads to $ \bigvee_{j=1}^{p_{a}}\delta_{a,j}^u$  is true, and this
allows us to conclude that (\ref{execaction}) is satisfied by $\sigma_S\circ\sigma_a$.

\medskip

Let us now consider the constraints dealing with fluents.
We recall that
$S'$ is a set of fluent literals that is
consistent, complete, and closed w.r.t. $\mathcal{SL}$.
Let us consider a fluent $f$ and let us
prove that constraint $(\ref{cieffe5})$ of \Fig{cieffe} is satisfied.
Assume, by contradiction, that $\codetext{Fired}_f^{u,v}\sigma$ and
$\codetext{Fired}_{\texttt{neg}(f)}^{u,v}\sigma$
are both true. Four cases must be considered:
\begin{enumerate}
\item $\codetext{Dyn}_f^u\sigma$ and $\codetext{Dyn}_{\texttt{neg}(f)}^u\sigma$ are true.
Since these values are determined by $u,a,v$, this means that
both $f$ and $\codetext{neg}(f)$ belong to $E(a,u)$. Since the closure under $\mathcal{SL}$
is monotonic
this means that $Lit(v)=S'$ is inconsistent, representing a contradiction.

\item $\codetext{Dyn}_f^u\sigma$ and $\codetext{Stat}_{\texttt{neg}(f)}^{v}\sigma$ are
true. This means that
$f$ is in $E(a,u)$ and $\codetext{neg}(f)$ is added to $S'$ by the closure operation.
This implies that  $S'$ is inconsistent, which represents a contradiction.

\item $\codetext{Stat}_f^{v}\sigma$ and $\codetext{Dyn}_{\texttt{neg}(f)}^u\sigma$ are
true. This leads
 a contradiction as in the previous case.

\item $\codetext{Stat}_f^{v}\sigma$ and $\codetext{Stat}_{\texttt{neg}(f)}^{v}\sigma$ are true.
This means that $f$ and $\codetext{neg}(f)$ are added to  $S'$
by the closure operation. Thus, $S'$ is inconsistent, which is a contradiction.
\end{enumerate}
It remains to prove that constraint $(\ref{cieffe6})$ is satisfied by $\sigma$.
Let us assume that $f \in S'$. Thus, $F_f^{v}\sigma_{S'}$ is true.
Three cases must be considered.

\begin{enumerate}
\item $f \in E(a,u)$. This means that there is a dynamic causal law
$\codetext{causes}(a,f,\alpha_{f,i})$
where $S \models \alpha_{f,i}$. {F}rom the definition, this leads
to $\alpha_{f,i}^u\sigma$ being true and $\sigma_a(A_a^u)=1$.
Thus, constraints $(\ref{cieffe2})$ and $(\ref{cieffe4})$ set
$\codetext{Dyn}_f^u\sigma$ and
$\codetext{Fired}_f^{u,v}\sigma$ both true.
As a consequence, constraint $(\ref{cieffe6})$ is satisfied.

\item $f \notin E(a,u)$ and $f \in S$.
This means that $f \in S \cap S'$.
In this case $\codetext{Fired}_{\texttt{neg}(f)}^{u,v}\sigma$ must be false,
 otherwise $S'$ would be inconsistent (by closure).
Thus,  $F_f^u\sigma_S$ should be true,
$F_f^{v}\sigma_{S'}$ is true
and $\codetext{Fired}_{\texttt{neg}(f)}^{u,v}\sigma$ is false, which
satisfy
constraint $(\ref{cieffe6})$ (regardless of the
value of $\codetext{Fired}_f^{u,v}\sigma$).

\item $f \notin E(a,u)$ and $f \notin S$.
This means that $f$ is inserted in $S'$ by closure.
Thus, there is a static causal law of the form
$\codetext{caused}(\gamma_{{f,j}},f)$ such that $S' \models \gamma_{{f,j}}$.
In this case, by $(\ref{cieffe3})$, $\codetext{Stat}_f^{v}\sigma$ is true
 and, by $(\ref{cieffe4})$, so is $\codetext{Fired}_f^{u,v}\sigma$.
Thus, constraint $(\ref{cieffe6})$ is satisfied.
\end{enumerate}

If $f \notin S'$, then
$\codetext{neg}(f) \in S'$ and the proof is similar with
positive and negative roles interchanged.\qed
\end{proof}

Let us observe that the converse of the above theorem does not
necessarily hold. The problem arises from the fact that
the implicit minimality in the closure operation is not
reflected in the computation of solutions to the constraint.
Consider the domain description where $\mathcal{F}=\{ f, g , h\}$ and
$\mathcal{A} =\{ a \}$, with the following laws:
\begin{center}
\begin{tabular}{lcl}
\codetext{executable($a$,[]).} & \phantom{aaaa} &  \codetext{caused([$g$],$h$).}\\
\codetext{causes($a$,$f$,[]).}   & \phantom{aaaa} &  \codetext{caused([$h$],$g$).}\\
\end{tabular}
\end{center}
\noindent
Let us consider
$S=\{\codetext{neg($f$)},\codetext{neg($g$)},\codetext{neg($h$)}\}$.
Then,
$S'= \{\codetext{$f$},
\codetext{$g$},\codetext{$h$}\}$
determines a solution of the constraint $C_{\mathcal{F}}^{u,v}$ with the
execution of action $a$, but
$\Clo_{\mathcal{SL}}(E \cup (S \cap S'))=\{f\} \subset S'$.
However, the following holds:

\begin{theorem}[Weak Soundness]\label{partialcorr}
Let $\mathcal{D} = \langle \mathcal{DL},\mathcal{EL},\mathcal{SL}\rangle$.
Let $\sigma_{S}\circ\sigma_{S'}\circ\sigma_{a}$
identify a solution of the constraint $C_{\mathcal{F}}^{u,v}$.
Then $\Clo_{\mathcal{SL}}(E(a,u) \cup (S \cap S'))\subseteq S'$.
\end{theorem}
\begin{proof}
It is immediate to see that
$\sigma_{S}$ and $\sigma_{S'}$ uniquely determines two
consistent and complete sets of fluent literals $u$ and $v$.
Let $f$ be a positive fluent in $\Clo_{\mathcal{SL}}(E(a,u) \cup (S \cap S'))$.
We show now that $f \in S'$.

\begin{enumerate}
\item If $f$ is in $S \cap S'$ we are done.

\item  If $f \in E(a,u)$, there is a law
$\codetext{causes}(a,f,\alpha_{f,i})$ such that
$S \models \alpha_{f,i}$. Since $S$ is determined by $\sigma_S$,
by $(\ref{cieffe2})$, we have that $\sigma_S\circ \sigma_a$ is
a solution of $\alpha_{f,i}^u \wedge A^u_{a}$, which implies
that  $\codetext{Dyn}_f^u$ is true, and
$\sigma_{S'}(F_f^{v})$ is true in $\sigma_{S'}$. Therefore,
$f\in S'$.
Observe also that
$\sigma_a$ making true $A_a^u$ will imply that
$\delta_{{a},h}^u$ is true (for some $h\in\{1,\ldots,p_{a}\}$), which
will imply satisfiability of the executability preconditions for~$a$.

\item
We are left with the case of $f \notin E(a,u)$ and $f \notin S \cap S'$.
Since $S'$ is determined by $\sigma_{S'}$, and
$f \in \Clo_{\mathcal{SL}}(E(a,u) \cup (S \cap S'))$,
there is a law $\codetext{caused}(\gamma_{{f,j}},f)$
such that $S' \models \gamma_{{f,j}}$, and by
construction $\sigma_{S'}$ makes
$\gamma_{{f,j}}^{v}$ true. Thus, $\codetext{Stat}_f^{v}$
is true and therefore $F_f^{v}$ is true. Hence, $f \in S'$.
\end{enumerate}
The proof proceeds similarly in the case of a negative fluent
$\codetext{neg}(f)$ in
$\Clo_{\mathcal{SL}}(E(a,u) \cup (S \cap S'))$.\qed
\end{proof}

\medskip

Let us consider the set of static causal laws
$\mathcal{SL}$. We can introduce a notion of
\emph{positive dependence graph}, following the traditional principle
of dependence analysis used in logic programming (e.g.,~\cite{lin}).
The graph ${\cal G}(\mathcal{SL})$ is defined as follows:
\begin{itemize}
\item the set of the nodes in ${\cal G}(\mathcal{SL})$ corresponds
        to the set of fluent literals, i.e.,
\[ \mbox{\it Nodes}({\cal G}(\mathcal{SL})) = \{f \:|\: f \in \mathcal{F}\} \cup
                                \{ \codetext{neg}(f) \:|\: f \in \mathcal{F}\}\]

\item edges are created to denote the dependence of a fluent literal
        on other literals due to a static causal law, i.e.,
\[ \mbox{\it Edges}({\cal G}(\mathcal{SL})) = \{(\ell_1,\ell_2)\:|\: \codetext{caused}(L,\ell_1)\in\mathcal{SL}, L = \codetext{[\dots},\ell_2,\codetext{\dots]}\}\]
\end{itemize}
A set of fluent literals $L$ is a \emph{loop} if, for any $\ell_1,\ell_2\in L$,
we have that there is a path from $\ell_1$ to $\ell_2$ in ${\cal G}(\mathcal{SL})$
such that all nodes encountered in such path are in $L$.
We say that a domain specification $\mathcal{D}=\langle \mathcal{DL}, \mathcal{EL},
        \mathcal{SL}\rangle$ is \emph{acyclic} if the graph  ${\cal G}(\mathcal{SL})$
        does not contain any loops.

\begin{figure}
\begin{center}
\makebox{
\begin{pspicture}(0,0)(6.9,4.2)
\put(0,0){\psscalebox{1.1}{\figurazero}}
\end{pspicture}
}
\end{center}
\caption{\label{tratteggi}Sets of fluents involved in a state
transition and a literal $\ell$ introduced by closure.}
\end{figure}
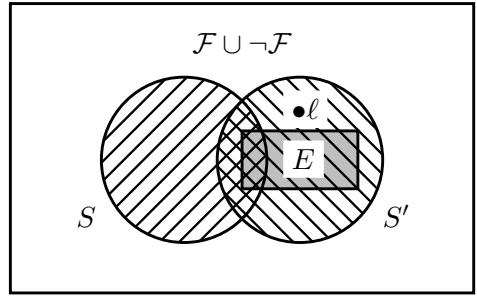

\begin{theorem}[Acyclic Soundness]\label{correctB}
Let $\mathcal{D} = \langle \mathcal{DL},\mathcal{EL},\mathcal{SL}\rangle$.
Let $\sigma_{S}\circ\sigma_{S'}\circ\sigma_a$
be a solution of the constraint $C_{\mathcal{F}}^{u,v}$. If the dependency graph of
$P$ is acyclic, then $\Clo_{\mathcal{SL}}(E(a,u) \cup (S \cap S')) = S'$.
\end{theorem}
\begin{proof}
Theorem~\ref{partialcorr} proves that $\Clo_{\mathcal{SL}}(E(a,u) \cup (S \cap S')) \subseteq S'$.
It remains to prove that for any (positive or negative) fluent $\ell$,
if $\ell \in S'$, then $\ell \in \Clo_{\mathcal{SL}}(E(a,u) \cup (S \cap S'))$.

If $\ell \in E(a,u)$ or $\ell \in S$, then
trivially $\ell \in \Clo_{\mathcal{SL}}(E(a,u) \cup (S \cap S'))$.\\
Let us prove that (cf., \Fig{tratteggi}):
$$(\ell \in S' \wedge \ell \notin E(a,u) \cup (S \cap S'))
\rightarrow \ell \in \Clo_{\mathcal{SL}}(E(a,u) \cup (S \cap S'))$$
To this aim, consider the dependence graph ${\cal G}(\mathcal{SL})$.
Because of the acyclicity of  $\mathcal{G}(\mathcal{SL})$,
there are nodes in $\mathcal{G}(\mathcal{SL})$ without incoming
edges---we will refer to them as \emph{leaves}.
For any node $\ell$ of $\mathcal{G}(\mathcal{SL})$,
let $d(\ell)$ denote the length of the longest
path from a leaf of $\mathcal{G}(\mathcal{SL})$ to~$\ell$.
We prove the property for a positive fluent literal $\ell=f$,
by induction on $d(\ell)$.

\noindent
\emph{Base case.}
If  $f \notin E(a,u) \cup (S \cap S')$
is a positive fluent which is a leaf (the proof is similar for the case of
negative literals), then
two cases could be possible.
\begin{list}{$\bullet$}{\topsep=2pt \itemsep=2pt \parsep=0pt \leftmargin=12pt}
\item
There is no law of the form
$\codetext{caused}(\_,f)$ in $\mathcal{SL}$.
In this case, it cannot be that $f \in S'$ due to
constraint (\ref{cieffe3}).

\item
There is a law $\codetext{caused}([\,],f)$.
In this case $f \in S'$ by closure.
\end{list}

\noindent
\emph{Inductive step.}
Let $f \notin E(a,u) \cup (S \cap S')$ be a positive fluent
such that there are laws
$\codetext{caused}(\gamma_{{f,1}}, f), \ldots, \codetext{caused}(\gamma_{{f,h}}, f)$
in $\mathcal{SL}$.
By the inductive hypothesis, let us assume that the thesis holds for each fluent literal
$\ell$ such that $d(\ell)<d(f)$.
Since $f \notin E(a,u)$ and $f \notin S \cap S'$,
we have that $F_f^u$ is false, $F_f^{v}$ is true,
and $\codetext{Dyn}_f^u$ is false under $\sigma_S\circ\sigma_{S'}\circ\sigma_a$.
{F}rom the fact that constraint~$(\ref{cieffe6})$ is satisfied,
it follows that $\codetext{Stat}_f^{v}$ is true.
Moreover, $\codetext{Dyn}_f^u$ is false
because  $f \notin E(a,u)$.
On the other hand, because of~$(\ref{cieffe5})$, we have that
 $\codetext{Dyn}_{\texttt{neg}(f)}^u$, $\codetext{Stat}_{\texttt{neg}(f)}^{v}$,
and $\codetext{Fired}_{\texttt{neg}(f)}^{u,v}$ are all false.
Consequently, constraint~$(\ref{cieffe6})$ can be rewritten as
 $F_f^{v} \leftrightarrow \bigvee_{j=1}^h \gamma_{{f,j}}^{v}$.
Since $f \in S'$ (i.e., $F_f^{v}$ is true), there must exists a $j\in\{1,\ldots,h\}$
such that $\gamma_{{f,j}}^{v}$ is verified by $\sigma_{S'}$.
This implies that, for each fluent $g$ required to be true  (resp., false)
in $\gamma_{{f,j}}$, $F_{g}^{v}$ is set true (resp., false) by $\sigma_{S'}$.
By inductive hypothesis, such fluent literals (either $g$ or $\codetext{neg}(g)$)
belong to $\Clo_{\mathcal{SL}}(E(a,u) \cup (S \cap S'))$.
Since $\Clo_{\mathcal{SL}}(E(a,u) \cup (S \cap S'))$  is closed w.r.t.
the static laws, it follows that
$f \in \Clo_{\mathcal{SL}}(E(a,u) \cup (S \cap S'))$.

The proof in case of a negative fluent $\codetext{neg}(f)$ is similar.
\qed
\end{proof}

\medskip

In order to achieve soundness in cases where the graph
$\mathcal{G}(\mathcal{SL})$ contains loops, it is necessary to introduce
additional constraints in conjunction with $C_{\mathcal{F}}^{u,v}$.
Intuitively, in the semantics of $\mathcal{B}$, cyclic dependencies created by the
static causal laws are resolved by the closure operation $\Clo_{\mathcal{SL}}(\cdot)$ by
minimizing the number of fluent literals that are made true---this derives
by the implicit minimality of the closure. Additional constraints can
be added to enforce this behavior; these constraints can be derived by following
a principle similar to that of  \emph{loop formulae} commonly used in the
context of logic programming~\cite{lin}.

The notion of loop formulae can be developed in our context as follows.
Let $L = \{\ell_1,\dots,\ell_k\}$ be a loop in $\mathcal{G}(\mathcal{SL})$ and
let us consider the transition from $u$ to $v$ as studied earlier.
Let us define a \emph{counter-support}  for $\ell_i$ w.r.t.\ the loop $L$
as a set of constraints $cs$ with the following properties:
\begin{itemize}
\item for each $\codetext{causes}(a_j,\ell_i,\alpha)$ in
        $\mathcal{DL}$, \,$cs$ contains either $A^u_{a_j}=0$ or
        $F_{\bar{\ell}}^u=1$ for some $\ell$ in $\alpha$;
\item for each $\codetext{caused}(\gamma,\ell_i)$ in $\mathcal{SL}$
        such that none of $\ell_1,\dots,\ell_k$ is in $\gamma$,
        for some $\ell$ in $\gamma$
        \,$cs$ contains $F_{\bar{\ell}}^v=1$;
\item $cs$ contains either $F_{\bar{\ell_i}}^u=1$ or
        $F_{\bar{\ell_i}}^v=1$.
\end{itemize}
(As usual, we might identify a set $cs$ of constraint with their conjunction, depending on the need.)
Let us denote with $\mbox{\it Counters}(\ell_i,L)^{u,v}$ the set of all
such counter-supports.
The loop formulae for $L$ w.r.t. $u,v$
is the set of constraints
\[\mathit{Form}(L)^{u,v} =  \{ c_1\wedge \dots \wedge c_k \rightarrow F_{\ell_1}^v=0 \wedge \dots \wedge
F_{\ell_k}^v=0 \:|\: c_i \in \mbox{\it Counters}(\ell_i,L)^{u,v}\}.\]
To take into account all different loops in $\mathcal{G}(\mathcal{SL})$,
let $\mathit{Form}(\mathcal{D})^{u,v}$ be the constraint
$$
\mathit{Form}(\mathcal{D})^{u,v} = \bigwedge_{L \mbox{ \scriptsize is a loop in }
\mathcal{G}(\mathcal{SL})}\mathit{Form}(L)^{u,v}.
$$

Following the analogous proofs relating answer sets and models of
a program completion that satisfies loop formulae (e.g.,~\cite{lin}) one
can show:
\begin{theorem}[Soundness]\label{correctB1}
Let $\mathcal{D} = \langle \mathcal{DL},\mathcal{EL},\mathcal{SL}\rangle$ and
let $\sigma_{S}\circ\sigma_{S'}\circ\sigma_a$
be a solution of the constraint $C_{\mathcal{F}}^{u,v}\wedge \mathit{Form}(\mathcal{D})^{u,v}$.
Thus, $\Clo_{\mathcal{SL}}(E(a,u) \cup (S \cap S')) = S'$.
\end{theorem}

Let the action description $\mathcal{D}$ meet the conditions of Theorem~\ref{correctB1}  and let
$\langle \mathcal{S}, \nu, R \rangle$ be its underlying transition system.
The following can be proved.
\begin{theorem}
There is a trajectory $\langle s_0,a_1,s_1,a_2,\ldots,a_{\n},s_{\n} \rangle$
in the transition system $\langle \mathcal{S}, \nu, R \rangle$ if and only if
$s_0$ is closed w.r.t. $\mathcal{SL}$ and
there is a solution for the constraint
$$\bigwedge_{j=0}^{\n-1}\big(C_{\mathcal{F}}^{s_{j},s_{j+1}}
\wedge \mathit{Form}(\mathcal{D})^{s_{j},s_{j+1}}\big)$$
\end{theorem}
\begin{proof}
The result follows directly by application of Theorems~\ref{Completeness} and~\ref{correctB1}
and by observing that for each transition $\langle s_j,a_{j+1},s_{j+1}\rangle$,
the satisfaction of constraint $C_{\mathcal{F}}^{s_{j},s_{j+1}}$ implies that
the state $s_{j+1}$
is closed w.r.t.~$\mathcal{SL}$.\qed
\end{proof}

Let $\langle\mathcal{D}, \mathcal{O}\rangle$ be an instance of a planning problem
where
$\mathcal{D}$ is an action description and $\mathcal{O}$ contains any
number of axioms of the form
$\codetext{initially}(C)$ and
$\codetext{goal}(C)$.
We can state the following.
\begin{corollary}
There is a trajectory
$\langle s_0,a_1,s_1,a_2,\ldots,a_{\n},s_{\n} \rangle$
for the planning problem
$\langle\mathcal{D}, \mathcal{O}\rangle$
if and only if
$s_0$ is closed w.r.t. the static causal laws of $\mathcal{D}$ and
there is a solution for the constraint
$$
\bigwedge_{\codetext{{\scriptsize initially}}(C)\in\mathcal{O}}C^{s_0}
\:\wedge\:
\bigwedge_{j=0}^{\n-1}\big(C_{\mathcal{F}}^{s_{j},s_{j+1}}
\wedge \mathit{Form}(\mathcal{D})^{s_{j},s_{j+1}}\big)
\:\wedge\:
\bigwedge_{\codetext{{\scriptsize goal}}(C)\in\mathcal{O}}C^{s_{\n}}
$$
\end{corollary}

\subsection{Mapping the Model to CLP(FD)}\label{CLPMappingB}

The modeling described in \Sect{modelingB} has been translated into a concrete implementation using
SICStus Prolog. In this translation, constrained CLP variables directly
reflect the Boolean variables modeling fluents and action's occurrences.
Consequently, causal laws and executability conditions
are directly translated into CLP constraints (and inherit the corresponding completeness and soundness
results).
In this section we highlight the main aspects of the implementation---while
the complete code can be found at~\sito.

A plan with exactly $\n+1$ states, $p$  fluents, and $m$  actions is represented by:
\begin{itemize}
\item
A list, called \codetext{States}, containing $\n+1$ lists, each
composed  of $p$ terms of the form
\codetext{fluent(fluent\_name, Bool\_var)}.
The variable of the $i^{th}$ term in the $j^{th}$ list is
assigned $1$ if and only if the
$i^{th}$ fluent is true  in the $j^{th}$ state of the trajectory.
For example, if we have $\n=2$ and the fluents \codetext{f}, \codetext{g},
and \codetext{h}, we have:

{\codesize\begin{verbatim}
States = [[fluent(f,X_f_0),fluent(g,X_g_0),fluent(h,X_h_0)],
          [fluent(f,X_f_1),fluent(g,X_g_1),fluent(h,X_h_1)],
          [fluent(f,X_f_2),fluent(g,X_g_2),fluent(h,X_h_2)]]
\end{verbatim}}

\item
A list \codetext{ActionsOcc}, containing  $\n$ lists, each
composed  of $m$ terms of the form
\codetext{action(action\_name,Bool\_var)}.
The variable of the $i^{th}$ term of the $j^{th}$ list is
assigned $1$
if and only if the $i^{th}$ action
occurs during the transition from state $j$ to state $j+1$.
For example, if we have $\n=2$ and the actions are \codetext{a} and \codetext{b}, then:

{\codesize\begin{verbatim}
ActionsOcc = [[action(a,X_a_1),action(b,X_b_1)],
              [action(a,X_a_2),action(b,X_b_2)]]
\end{verbatim}}
\end{itemize}
The planner makes use of these structures in the construction of the
plan; appropriate constraints are set between the various Boolean
variables to capture their relationships.
For each list in \codetext{ActionsOcc},
exactly one \codetext{action(a$_i$,VA$_i$)} contains a variable
that is assigned the value $1$ (cf.,~constraint~(\ref{unica_azione})).

\begin{figure}
\begin{center}
\begin{codice}{0}
clpplan(N, ActionsOcc, States) :- \\
\tab  setof(F, fluent(F), Lf),\\
\tab  setof(A, action(A), La),\\
\tab   make\_states(N, Lf, States),\\
\tab   make\_action\_occurrences(N, La, ActionsOcc),\\
\tab   setof(F, initially(F), Init),\\
\tab   setof(F, goal(F), Goal),\\
\tab   set\_initial(Init, States),\\
\tab   set\_goal(Goal, States),\\
\tab   set\_transitions(ActionsOcc, States),\\
\tab   set\_executability(ActionsOcc, States),\\
\tab   get\_all\_actions(ActionsOcc, AllActions),\\
\tab   labeling(AllActions).
\end{codice}
\end{center}
\caption{\label{mainplan}Main predicate of the CLP(FD) planner.}
\end{figure}

The CLP implementation of the
$\mathcal{B}$ language assumes that
the action description is encoded as Prolog facts---observe
that the syntax of  $\mathcal{B}$
is compliant with Prolog's syntax,
allowing us to directly store the domain description as rules and
facts in the Prolog database.
The entry point of the planner is shown in \Fig{mainplan}.

The main predicate is \codetext{clpplan(N, ActionsOcc, States)}
(line (1)) that computes a plan of length $\n$ for the action description present
in the Prolog database. If such a plan exists,
the variables in \codetext{ActionsOcc} and  \codetext{States} will be instantiated
so as to describe the found trajectory.

Lines (2) and (3)
collect the lists of all fluents (\codetext{Lf}) and all actions
(\codetext{La}). Lines (4) and (5) are used for the
creation of the  lists \codetext{States}
and \codetext{ActionsOcc}.
In particular, all the
variables for fluents and actions are declared
as Boolean variables. Furthermore, a constraint is added to enforce
that in every state transition, exactly
one action can be fired.

Lines (6) and (7) collect the description of the initial state (\codetext{Init}) and the
required content of the final state (\codetext{Goal}). These information
are then added to the Boolean variables related to the first and last state,
respectively, by the predicates in lines (8) and~(9).

Lines (10) and (11) impose the constraints on state transitions and action executability,
as described in \Sect{modelingB}. We will give more details on this part below.

Line (12) gathers all variables denoting action occurrences, in preparation
for the labeling phase (line (13)). Note that the labeling is focused on the
selection of the action to be executed at each time step.
Some details on the labeling strategy are discussed in \Sect{sec:experimental}.
Please observe that in the code of \Fig{mainplan} we  omit the parts
concerning delivering the results to the user.

The main constraints are added by the predicate \codetext{set\_transitions}.
The process is based on a recursion across
fluents and consecutive states. The predicate
\codetext{set\_one\_fluent}
is called (see \Fig{setonefluent}) at the core of the recursion.
Its parameters are the fluent \codetext{F}, the starting state
\codetext{FromState}, the next state \codetext{ToState}, the list
\codetext{Occ} of action variables, and
finally the variables \codetext{IV} and \codetext{EV},
related to the value of the fluent \codetext{F}
in \codetext{FromState} and \codetext{ToState}, respectively
(see also \Fig{act_fig}).

\begin{figure}
\begin{codice}{13}
set\_one\_fluent(F, IV, EV, Occ, FromState, ToState) :- \\
\tab    findall([X,L], causes(X,F,L), DynPos),\\
\tab    findall([Y,M], causes(Y,neg(F),M), DynNeg),\\
\tab    dynamic(DynPos, Occ, FromState, DynP, EV),\\
\tab    dynamic(DynNeg, Occ, FromState, DynN, EV),\\
\tab    findall(P, caused(P,F), StatPos),\\
\tab    findall(N, caused(N,neg(F)), StatNeg),\\
\tab    static(StatPos, ToState, StatP, EV),\\
\tab    static(StatNeg, ToState, StatN, EV),\\
\tab    bool\_disj(DynP, StatP, PosFired),\\
\tab    bool\_disj(DynN, StatN, NegFired),\\
\tab    PosFired*NegFired  \verb"#=" 0,\\
\tab    EV \verb"#<=>"  PosFired \verb"#\/" (\verb"#\" NegFired  \verb"#/\" IV).
\medskip\\
dynamic([], \_, \_, [], \_).\\
dynamic([[Action,Precondition]|R], Occ, FromState, [Flag|Flags], EV)~:-\\
\tab         member(action(Action,VA), Occ),\\
\tab         get\_precondition\_vars(Precondition, FromState, ListPV),\\
\tab         length(ListPV, NPrec),\\
\tab         sum(ListPV, SumPrec),\\
\tab         (VA  \verb"#/\" (SumPrec \verb"#=" NPrec)) \verb"#<=>" Flag,\\
\tab          dynamic(R, Occ, FromState, Flags, EV).
\medskip\\
static([], \_, [], \_).\\
static([Cond|Others], ToState, [Flag|Flags], EV) :- \\
\tab         get\_precondition\_vars(Cond, ToState, ListPV),\\
\tab         length(ListPV, NPrec),\\
\tab         sum(ListPV, SumPV),\\
\tab         (SumPV \verb"#=" NPrec) \verb"#<=>" Flag,\\
\tab     static(Others, ToState, Flags, EV).
\end{codice}
\caption{\label{setonefluent}Transition from state to state.}
\end{figure}

For a given fluent \codetext{F},
the predicate \codetext{set\_one\_fluent} collects the list \codetext{DynPos}
(respectively \codetext{DynNeg})
of all the pairs \codetext{[Action,Preconditions]}
such that the dynamic action
\codetext{Action} makes \codetext{F} true (respectively false) in the state transition
(lines (15) and (16)).
The variables involved are then constrained by the procedure \codetext{dynamic}
(lines (17) and (18)).

Similarly, the static causal laws are handled
by collecting the lists of conditions that affect the truth value of a fluent \codetext{F}
(i.e., the variables  \codetext{StatPos} and \codetext{StatNeg}, in lines (19)--(20))
and constraining them through the procedure
\codetext{static} (lines (21) and (22)).
The disjunctions of all the positive and negative conditions are collected
in lines (23) and (24) and stored in \codetext{PosFired}
and \codetext{NegFired}, respectively.

Finally, lines (25) and (26) take care of the relationships between all these variables.
Line (25) implements the constraint (\ref{per}) for
the state \codetext{ToState} of \Fig{cieffe}, stating
that we do not want inconsistent action theories.
If \codetext{PosFired} and \codetext{NegFired} are both false, then
$\codetext{EV}=\codetext{IV}$ (inertia).
Precisely, a fluent is true in the next state (\codetext{EV})
if and only if there is an action or a static causal law making it true (\codetext{PosFired})
or it was true in the previous state (\codetext{IV}) and no causal law makes it false.

Let us consider the predicate \codetext{dynamic} (see line (27) in \Fig{setonefluent}).
It recursively processes a list of pairs \codetext{[Action,Preconditions]}.
The variable \codetext{VA} associated to the execution of action \codetext{Action}
is retrieved in line (29).
The variables associated to its preconditions are retrieved from
state \codetext{FromState} and collected in \codetext{ListPV} in line~(30).
A precondition holds if and only if all the variables in the list
\codetext{ListPV} are assigned value~1, i.e.,
when their sum is equal to the length, \codetext{NPrec}, of the
list \codetext{ListPV}.
If (and only if) the action variable \codetext{VA} is true and
the preconditions holds, then there is an action effect (line~(33)).

Similarly, the predicate
\codetext{static} (line (35) in \Fig{setonefluent}) recursively processes a list of
preconditions.
The variables involved in each of such precondition \codetext{Cond} are retrieved from the
state \codetext{ToState} and collected in \codetext{ListPV} (line (37)).
A precondition holds if and only if all the variables in the list \codetext{ListPV} have
value~1, i.e.,
when their sum is equal to the length, \codetext{NPrec}, of \codetext{ListPV}.
This happens if and only if there is a static action effect
(see  line (40)).

\begin{figure}
\begin{codice}{41}
set\_executability\_sub([], \_, \_).\\
set\_executability\_sub([[Act,C]|CA], ActionsOcc, State) :-  \\
\tab    member(action(Act,VA), ActionsOcc),\\
\tab    preconditions\_flags(C, State, Flags),\\
\tab    bool\_disj(Flags, F),\\
\tab    VA \verb"#==>" F,\\
\tab    set\_executability\_sub(CA, ActionsOcc, State).\\
preconditions\_flags([], \_, []).\\
preconditions\_flags([C|R], State, [Flag|Flags]) :- \\
\tab      get\_precondition\_vars(C, State, Cs),\\
\tab      length(Cs, NCs),\\
\tab      sum(Cs, SumCs),\\
\tab      (NCs \verb"#=" SumCs)  \verb"#<=>" Flag,\\
\tab      preconditions\_flags(R, State, Flags).
\end{codice}
\caption{\label{execfig}Executability conditions.}
\end{figure}

Executability conditions are handled as follows. For each state transition
and for each action \codetext{Act}, the predicate
\codetext{set\_executability\_sub} is called (see \Fig{execfig}).
The variable \codetext{VA},
encoding the application of an action \codetext{Act}
is collected in line (44).
A precondition hold if and only if the sum of the (Boolean) values of its fluent literals
equals their number (lines (52)-(54)).
The variable \codetext{Flags} stores the list of these conditions
and the variable \codetext{F} their disjunction.
If the action is executed (\codetext{VA = 1}, see line (47)),
then at least one of the executability conditions must hold.

\section{The Action Language with Constraints on Multi-valued Fluents}\label{motivat}

As a matter of fact, constraints represent a very declarative notation to express relationships
between
un\-knowns.
As such, the ability to use them directly in an action language
greatly enhances the declarative and expressive power of the language,
facilitating
the encoding of complex action domains, such as those involving multi-valued
fluents.
Furthermore, the encoding of an action theory using multi-valued fluents
leads to more concise and more efficient representations
and
better exposing non-determinism (that could be exploited, for example,
by a parallel planner). Let us consider some representative examples.

\begin{example}[Maintenance Goals]\label{ExaMaintenanceGoals}
It is not uncommon to encounter planning problems where along with the
type of goals described earlier (known as \emph{achievement} goals), there
are also \emph{maintenance} goals, representing properties that must persist
throughout the trajectory.
Constraints are a natural way of encoding maintenance properties, and can be
introduced along with simple temporal operators. E.g., if the fluent
\codetext{$fuel$}
represents the amount of fuel available, then the maintenance goal which
guarantees
that we will not be left stranded could be encoded as:
\codetext{always($fuel>0$)}.
\eoe
\end{example}

\begin{example}[Control Knowledge]\label{ExaControlKnowledge}
Domain-specific control knowledge can be formalized as constraints that we
expect to be satisfied by all the trajectories. For example, we may know that
if a certain action occurs at a given time step  (e.g., \codetext{$ingest\_poison$})
then at the next time step we will always perform the same action
(e.g., \codetext{$call\_doctor$}). This could be encoded as
\begin{center}
\codetext{caused([occ($ingest\_poison$)], occ($call\_doctor$)$^{1}$)}
\end{center}

\noindent
where $\codetext{occ}(a)$ is a fluent describing the occurrence of the action $a$
and $f^1$ indicates that the fluent $f$ should hold at the
next time step.
\eoe
\end{example}

\begin{example}[Delayed Effect]\label{ExaDelayedEffect}
Let us assume that the action
\codetext{$request\_reimbursement$}
 has
a delayed effect (e.g., the increase by \$50 of
\codetext{$bank\_account$} after 30
time units). This could be
expressed as a dynamic causal law:
\begin{center}
\codetext{causes($request\_reimbursement$,incr($bank\_account$,$50$)$^{30}$,[])}
\end{center}
where \codetext{incr} is a constraint introduced to deal with additive
computations---in a way closer to $\mathcal{B}$'s syntax we should write:
\begin{center}
\codetext{causes($request\_reimbursement$,$bank\_account^{30}$ =
$bank\_account+50$,[])}.
\end{center}
This is a particular case of additive fluents~\cite{additive}.
\eoe
\end{example}

In what follows we introduce the action description language \BMV\
in which multi-valued fluents are admitted and constraints are first-class components
in the description of planning problems.
The availability of multi-valued constraints enables a number of immediate language
extensions and improves the expressive power of the overall framework.

Action description languages such as $\mathcal{B}$
rely on the common assumption, traditionally referred to as
\emph{Markovian property} in the context of systems and control theory: the
executability of an action and its effects depend exclusively on the
current state of the world~\cite{gabaldon,elephant}.
Nevertheless, it is not uncommon to encounter real world
situations where such property is not satisfied, i.e., situations
where the executability and/or the effects of an action depend not only on
what holds in the current situation, but also on whether some conditions
were satisfied at a previous point in time. For example, an agent controlling
access to a database should forbid access if in the recent past three failed
password submission attempts have been performed by the user.

Although non-Markovian preconditions and effects can be expressed in
a Markovian theory through the introduction of additional fluents (and
a correct handling of inertia), the resulting theory can become  significantly
larger and less intuitive.
An alternative solution consists of admitting past references
in modeling such kind of situations.
In this frame of mind, \BMV\ allows timed references to past points in time
within constraints, i.e., non-Markovian expressions that might
involve fluents' values.
Effects of dynamic laws that involves future references might
also be specified.
As a further feature the \BMV\ language admits the specification
of global constraints (involving absolutely specified points in time)
and costs for actions and plans.

The resulting description language
supports all the kind of modeling and reasoning outlined in the above
Examples~\ref{ExaMaintenanceGoals}--\ref{ExaDelayedEffect}.

\medskip

In the next sections,
we first introduce the syntax of the full-blown action description
language \BMV (\Sect{BMVsyntax}). In \Sect{BMVsemantics}
we will develop the semantics and the constraint-based
abstract implementation of this new language. In doing this,
for the sake of readability, we proceed incrementally in order
to focus on the main points and features of the framework.
We first consider the sub-language \BMVi\ obtained from \BMV\
by disallowing  timed references (\Sect{MarkovianBMV});
in \Sect{sema_2}, we treat the general case dealing with
past and future references.  The abstract implementation is provided in \Sect{abstractImplBMV}.
Finally, we give the semantics to the complete language involving cost and
global constraints (\Sect{sema_4}).

\section{The Language \BMV}\label{BMVsyntax}

As for $\mathcal{B}$, the action signature consists of a set $\mathcal{F}$  of fluent names,
a set $\mathcal{A}$ of action names,
and a set $\mathcal{V}$ of values for fluents in $\mathcal{F}$. In the following
we assume that $\mathcal{V}\subseteq\z$.

In an action domain description,
an assertion (\emph{domain declaration}) of the type
$$\codetext{fluent}(f, \{d_1,\ldots,d_k\})$$
declares that $f$ is a fluent
and that its set of values is $\{d_1,\ldots,d_k\}$; we refer
to the set $\{d_1,\dots,d_k\}$ as the \emph{domain} of $f$.
We also admit the simplified notation
$\codetext{fluent}(f, d_1, d_2)$
to specify all the integer values in the interval $[d_1,d_2]$ as admissible
(with $d_1 \leq d_2$).

An \emph{annotated fluent} \,(\codetext{AF})\, is an expression
$f^t$, where $f$ is a fluent and $t\in \interi$.
We will often denote $f^0$ simply by~$f$.
Intuitively speaking, if $t<0$ then
$f^t$ denotes the value that the fluent $f$ had $t$ steps
ago in the past; similarly, if $t > 0$, then $f^t$ denotes
the value $f$ will have $t$ steps in the future. We refer to
annotated fluents with $t>0$ as positively annotated fluents.

Annotated fluents can be used in \emph{Fluent Expressions} (\codetext{FE}),
which are defined inductively as follows:
$$ \codetext{FE} ::=
d \:\: |\:\:
\codetext{AF} \:\: |\:\:
\codetext{FE}_1 \oplus \,\codetext{FE}_2
\:\: |\:\:
- (\codetext{FE})
\:\: |\:\:
\codetext{abs}(\codetext{FE})
\:\: |\:\:
\codetext{rei}(\codetext{FC})
$$
where $d \in\mathcal{V}$ and $\oplus \in \{+,-,*,/,\codetext{mod}\}$.
\codetext{FC} is a fluent constraint (see below).
We refer to
the fluent expressions \codetext{rei(FC)} as the \emph{reification} of the
fluent constraint \codetext{FC}---its formal semantics is given in \Sect{reificaSem}.

Fluent expressions can be used to build
\emph{primitive fluent constraints} (\codetext{PC}), i.e., formulae of the
form  $\codetext{FE}_1 \:\op\: \codetext{FE}_2$,
where $\codetext{FE}_1$ and $\codetext{FE}_2$ are fluent expressions,
and $\op$ is a relational operator, i.e.,
$\op \in \{=, \neq, \geq, \leq, >, <\}$.
\emph{Fluent constraints} are propositional combinations of
primitive fluent constraints:
$$\begin{array}{rcl}
\codetext{PC} & ::= &
\codetext{FE}_1 \,\op \,\codetext{FE}_2\\
\codetext{C} & ::= &
\codetext{PC}
\:\: |\:\:
\neg\codetext{C}
\:\: |\:\:
\codetext{C}_1 \wedge \codetext{C}_2
\:\: |\:\:
\codetext{C}_1 \vee \codetext{C}_2
\end{array}
$$

The constant symbols \codetext{true} and \codetext{false} can be used
as a shorthand for true constraints (e.g., $d=d$, for some $d\in\mathcal{V}$)
and unsatisfiable constraints (e.g., $d \neq d$).

The language \BMV\ allows one to specify an
action domain description, which relates
actions, states, and fluents using  axioms of the following forms
($PC$ denotes a primitive fluent constraint, while
$C$ is a fluent constraint).
\begin{itemize}
\item Axioms of the form
$\codetext{executable}(a, C)$,
 stating  that
the fluent constraint
$C$
has to be satisfied by the current state
for the action $a$ to be executable.

\item Axioms of the form
$\codetext{causes}(a, PC, C)$
encode  dynamic causal laws. When the action $a$ is
executed, if the constraint $C$ is
satisfied by the current state, then state produced
by the execution of the action is required to satisfy
the primitive fluent constraint~$PC$.

\item Axioms of the form
$\codetext{caused}(C_1, C_2)$
describe  static causal laws.
If the fluent constraint
$C_1$ is satisfied in a  state,
then the  constraint $C_2$
must also hold in such  state.
\end{itemize}
An \emph{action domain description} of \BMV\ is a tuple
$\langle \mathcal{DL}, \mathcal{EL}, \mathcal{SL}
\rangle$, where
$\mathcal{EL}$ is a set of executability conditions,
$\mathcal{SL}$ is a set of static causal laws, and
$\mathcal{DL}$ is a set of  dynamic causal laws.
In the following, we assume that positively annotated
fluents can occur only in the effect part of dynamic
causal laws.

A specific instance of a planning problem
is a pair $\langle\mathcal{D}, \mathcal{O}\rangle$, where
$\mathcal{D}$ is an action domain description and $\mathcal{O}$ contains any
number of axioms of the form
$\codetext{initially}(C)$ and
$\codetext{goal}(C)$, where $C$ is a
fluent constraint.

\begin{example}
\label{rico2}
A sample action theory in \BMV\ is:
\begin{center}
\begin{tabular}{lcl}
\codetext{fluent($f$, $\{1,2,3,4,5\}$).} &&\\
\codetext{fluent($g$, $\{1,2,3,4,5\}$).} && \\
\codetext{fluent($h$, $\{1,2,3,4,5\}$).} &&\\
\codetext{causes($a$, $f=g+2$, $g<3$).}&&\\
\codetext{executable($a$, true).}      && \\
\codetext{initially($f=1$).}&&\\
\codetext{initially($g=1$).}&&\\
\codetext{initially($h=1$).}&&\\
\codetext{goal($f=5$).} &&
\end{tabular}
\end{center}
\eoe
\end{example}

Notice that, for any given dynamic law
$\codetext{causes}(a, PC, C)$, such that~$a$ is executed in a state~$u$ satisfying~$C$,
the constraint $PC$ has to be evaluates/satisfied in the target state~$v$. Hence, the (relative)
timed references occurring in $PC$ (respectively, in $C$) are resolved
with respect to~$v$ (resp.,~$u$).
On the other hand, for a static law $\codetext{caused}(C_1, C_2)$, relative timed references
of both $C_1$ and $C_2$ have to be resolved with respect to the current state.

\subsection{Absolute Temporal References}\label{absolutetemporalconstraints}

The language \BMV\ allows the definition of \emph{absolute temporal
constraints}, i.e., constraints that refer to specific moments in time
in the trajectory (by associating the time point $0$ to the initial state).
differently from the case of annotated fluents, where
points in time are \emph{relative} to the current state.
A \emph{timed fluent} is defined as an expression of the form
$$\codetext{FLUENT} \,\verb"@"\, \codetext{TIME}.$$
Timed fluents can be used to build \emph{timed fluent expressions} ($TE$) and
\emph{timed primitive constraints} ($TC$), similarly to what done for normal fluents.
For instance, the constraint
$$f{\verb"@"}2 < g{\verb"@"}4$$
states that the value the fluent $f$ has at time $2$ in the plan
is less than the value that the fluent $g$ has at time~4.
Similarly, $h{\verb"@"}2 = 3$ imposes that the fluent $h$ must assume value $3$
at time~$2$.

Timed constraints can be used in the following kind of assertion:
$$\codetext{time\_constraint}(TC)$$
The assertion requires the timed constraint $TC$ to hold.

Some other accepted constraints are:
\begin{itemize}
\item
$\codetext{holds}(FC, n)$: this constraint is
a particular case of the previous one.
It is satisfied if the
primitive fluent constraint $FC$ holds in the
$n^{th}$ state.
It is therefore a
generalization of the \codetext{initially} axiom.
Observe that assertions of this kind can be used to guide
the search of a plan by adding some point-wise information
about the states occurring along the computed trajectory (e.g.,
this is useful to implement the landmarks model as used
in the FF planner~\cite{ff}).
\item
$\codetext{always}(FC)$: this constraint
imposes the condition that the fluent constraint $FC$
holds in all the states.
Observe  that $FC$ has to be evaluated in all states,
and its evaluation is strict---i.e., any reference to
fluents outside the time limits leads to
the satisfaction of the constraint; hence, annotated fluents
should be  avoided in $FC$.
\end{itemize}
In specifying  a planning problem
 $\langle\mathcal{D}, \mathcal{O}\rangle$, we can consider such kinds of assertions
as part of the observations $\mathcal{O}$.

\begin{example}
Let us consider the case of an agent that has a certain amount of money
(e.g., $\$5,000$) to invest; she is interested in purchasing as many
stocks as possible. The stocks can be purchased from three trading agencies
($1$, $2$, and $3$); each agency has $1,000$ stocks available at $\$2$ each. The
stocks have to be purchased in separate transactions, but each trading agency
require the agent to have a balance of at least $\$2,000$ at the start of
the day before agreeing in the transaction. A purchase can be of at most $3,000$
shares at a time.

We can model this problem with the following fluents:
{\small$$\begin{array}{lcl}
\codetext{fluent}(money,0, 5000). & & \codetext{fluent}(have(stock1),0,1000).\\
\codetext{fluent}(have(stock2),0,1000). &               & \codetext{fluent}(have(stock3),0,1000).\\
\codetext{fluent}(available(stock1),0, 1000). & & \codetext{fluent}(available(stock2),0,1000).\\
\codetext{fluent}(available(stock3),0,1000). \\
\codetext{fluent}(price(stock1),2,2).         &                 & \codetext{fluent}(price(stock2),2,2).\\
\codetext{fluent}(price(stock3),2,2)).        &                &\\
\end{array}$$
}

\noindent
The only action is
\[ \codetext{action$(buy(StockType, N)) :-  N > 0, N < 3000.$} \]
The executability condition for the action captures one property: the agent is
accepted by the trading agency.
\[
\codetext{executable$(buy(Type,N), money @0 > 2000 \wedge money > N *price(Type))$}.
\]
The dynamic causal law for this action is:
\[
\begin{array}{l}
\codetext{causes$(buy(Type,N), money = money-N*price(Type), \codetext{true})$}.\\
\codetext{causes$(buy(Type,N), have(Type) = have(Type)+N, \codetext{true})$}.
\end{array}
\]

The initial state can be described as
{\small$$\begin{array}{lcl}
\codetext{initially}(price(stock1) = 2). &  & \codetext{initially}(price(stock2) = 2).\\
\codetext{initially}(price(stock3) = 2). & & \codetext{initially}(have(stock1)=0).\\
\codetext{initially}(have(stock2) = 0). & & \codetext{initially}(have(stock3) = 0).\\
\codetext{initially}(money = 5000). && \codetext{initially}(available(stock1) = 1000).\\
\codetext{initially}(available(stock2)=1000). && \codetext{initially}(available(stock3) = 1000).\\
\end{array}$$}
\eoe
\end{example}

\subsection{Cost Constraints}\label{costconstraints}

In \BMV\ it is possible to specify information about the \emph{cost}
of each action and about the \emph{global cost} of a plan (that is
defined as the sum of the costs of all its actions). This type of information
are useful to explore the use of constraints in determining
\emph{optimal} plans.

The cost of actions is expressed using  assertions of the
following forms
(where $FE$ is a fluent expression built using the fluents
present in the state):
\begin{itemize}
\item
$\codetext{action\_cost}(a,FE)$ specifies the cost of the
execution of the action $a$ as result of the expression $FE$.
\item
$\codetext{state\_cost}(FE)$
specifies the cost of a state as the result of the
evaluation of $FE$.
\end{itemize}
Whenever, for an action or a state, no cost declaration is provided,
a default cost of $1$ is assumed.
Once we have provided the costs for actions and states, we can impose constraints
on the cumulative costs of specific states or complete trajectories.
This can be done in  \BMV\ using assertions of the following types
(where $k$ is a number and \op\ a relational operator):
\begin{itemize}
\item
$\codetext{cost\_constraint}(\codetext{plan}\:\op\:k)$;
the assertion adds a constraint on the global cost of the plan.
\item
$\codetext{cost\_constraint}(\codetext{goal}\:\op\:k)$;
the assertion imposes a constraint on the global cost of the final state.
\item
$\codetext{cost\_constraint}(\codetext{state(i)}\:\op\:k)$;
the assertion imposes a constraint on the global cost of the \codetext{i}$^{th}$ state
of the trajectory.
\end{itemize}
As an immediate generalization of the above constraints, we admit
assertions of the form
~$\codetext{cost\_constraint}(C)$,~
where $C$ is a constraint, possibly involving fluents,
where the atoms \codetext{plan}, \codetext{goal}, and \codetext{state(i)}
might occur in any place where a fluent might---intuitively representing
the cost of a plan, of the goal state, and of the \codetext{i}$^{th}$
state, respectively.

Some directives can be added to an action theory to select optimal
solutions with respect to the specified costs:
$$\codetext{minimize\_cost}(FE),$$
where $FE$ is an expression involving the atoms \codetext{plan},
\codetext{goal}, and \codetext{state(i)},
and possibly other fluents.
This assertion constrains the search to determine  a plan
that minimizes the value of the expression $FE$.
For instance, the two assertions
~$\codetext{minimize\_cost}(\codetext{plan})$~
and
~$\codetext{minimize\_cost}(\codetext{goal})$~
 constrain the search of a plan
with minimal global cost and with minimal cost of the goal state, respectively.

We provide a more precise semantics for all these assertions in \Sect{sema_4}.
In specifying  a planning problem
 $\langle\mathcal{D}, \mathcal{O}\rangle$, we consider cost constraints as
part of the observations~$\mathcal{O}$.

\section{Semantics and Abstract Implementation of \BMV}\label{BMVsemantics}
We will build the semantics of the language \BMV\ incrementally. We will
start by building the semantics
for the sub-language of \BMV\ devoid of any form of time reference
and cost constraints (\Sect{MarkovianBMV}).
This core language is called \BMVi.
The subsequent \Sects{sema_2}--\ref{sema_4} treat the full \BMV.

\subsection{Semantics for Timeless Constraints}\label{BMVisemantics}\label{MarkovianBMV}

Each fluent $f$ is uniquely assigned to a domain $\dom(f)$
in the following way:
\begin{itemize}
\item if $\codetext{fluent}(f,Set) \in \mathcal{D}$,
then $\dom(f) = Set$.
\end{itemize}
A function $v: \mathcal{F} \rightarrow \mathcal{V}\cup\{\bot\}$
is a \emph{state} if $v(f) \in \dom(f) \cup\{\bot\}$ for all $f \in \mathcal{F}$. The
special symbol $\bot$ denotes that the value of the fluent is undefined.
A state $v$ is \emph{complete} if for all $f\in \mathcal{F}$, $v(f) \neq \bot$.
For a number $\n\geqslant 1$, we define a \emph{state sequence} $\vec{v}$
as a tuple $\langle v_0, \dots, v_{\n}\rangle$ where each $v_i$
is a state.

Given a state ${v}$, and an expression $\varphi$, we define
the \emph{value} of $\varphi$ in ${v}$
(with abuse of notation, denoted by ${v}(\varphi)$)
as follows:\label{reificaSem}\footnote{The expression $|n|$ denotes
the (algebraic) absolute value of~$n$.}
\begin{equation}\label{eq:semOfBMV}
\begin{array}{lcl}
\bullet & \hspace{.2cm} & {v}(x) = x ~~ \mbox{ if $x \in \mathcal{V}$}
\\
\bullet & \hspace{.2cm} &{v}(f) = v(f) ~~ \mbox{ if $f \in \mathcal{F}$ (abuse of notation here)}
\\
\bullet & \hspace{.2cm} &{v}(-(\varphi)) = -({v}(\varphi))
\\
\bullet & \hspace{.2cm} &{v}(\codetext{abs}(\varphi)) = |({v}(\varphi))|
\\
\bullet & \hspace{.2cm} &{v}(\varphi_1 \oplus \varphi_2)=  {v}(\varphi_1) \oplus
{v}(\varphi_2)
\\
\bullet & \hspace{.2cm} &{v}(\codetext{rei}(C))= 1 ~~ \mbox{if ${v}\models C$}
\\
\bullet & \hspace{.2cm} &{v}(\codetext{rei}(C))= 0 ~~ \mbox{if ${v}\not\models C$}
\end{array}
\end{equation}
We treat the interpretation of the various $\oplus$
operations and relations as strict with respect to~$\bot$
(i.e., $\bot \oplus x = x \oplus \bot = \bot$, $\codetext{abs}(\bot)=\bot$, etc.).

The last two cases in~(\ref{eq:semOfBMV}) specify the semantics of reification.
Reified constraints are useful to enable reasoning about the satisfaction
state of other formulae.
The intuitive semantics is that a fluent expression $\codetext{rei}(C)$,
where $C$ is a fluent constraint,
assumes a Boolean value ($0$ or $1$) depending on the truth of~$C$.
Note that the semantics of reified constrains relies on the
notion of \emph{satisfaction}, which in turn is defined by structural induction on
constrains, as follows.
Given a primitive fluent constraint $\varphi_1 \:\codetext{op}\:\varphi_2$,
a state ${v}$  \emph{satisfies} $\varphi_1 \:\codetext{op}\:\varphi_2$, written
$ {v}\models  \varphi_1  \:\codetext{op}\:\varphi_2$, if and only if
it holds that
$ {v}(\varphi_1) \:\codetext{op}\:  {v}(\varphi_2) $
where the semantics of the arithmetic relators/operators is the usual one on $\z$.
If either $ {v}(\varphi_1)$ or $ {v}(\varphi_2)$ is $\bot$, we assume
that $ {v}\not\models \varphi_1\:\codetext{op}\:\varphi_2$
(and $ {v}\not\models \varphi_1\:\codetext{nop}\:\varphi_2$
where \codetext{nop} is the negation of the operator
\codetext{op}). Basically undefined formulas are neither proved nor
disproved.
The satisfaction relation $\models $ can be generalized to the case of
propositional combinations of fluent constraints in the usual manner.

\smallskip

Given a constraint $C$, let $\FV(C)$ be the set of fluents
occurring in it. A function $\sigma:\FV(C) \longrightarrow \mathcal{V}$
is a \emph{solution} of $C$ if $\sigma \models C$.
We denote the domain $\FV(C)$ of the function $\sigma$ as $\dom(\sigma)$.
In other words, a solution $\sigma$ of $C$ can be seen as a partial state satisfying $C$.
Observe that we require the solution to manipulate exclusively the fluents
that appear in the constraint.

\begin{example}
\label{rico1}
Let us consider an action theory over the fluents $f, g, h$, where each
fluent has domain $\{1, \dots, 5\}$. If $C$ is the constraint
$f > g+2$, then a solution of $C$ is $\sigma = \{f/5,g/2\}$. Note that
the substitution $\theta=\{f/5, g/2, h/1\}$ is not a solution of $C$, since
$\dom(\theta) \neq \FV(f > g+2)$.
\eoe
\end{example}

Let $\sigma$ be a solution of a constraint $C$ and $v$ a state,
with $\ine(\sigma,v)$ we denote the state obtained completing $\sigma$ in
$v$ \emph{by inertia}, as follows:
$$\ine(\sigma,v)(f) = \left\{ \begin{array}{ll}
          \sigma(f) & \mbox{if $f \in \dom(\sigma)$}\\
          v(f) & \mbox{otherwise}
          \end{array}
          \right.$$

\begin{example}
Let us continue with Example \ref{rico1}. If $\sigma = \{f/5,g/2\}$ and $v = \{f/1,g/1,h/1\}$, then
$\ine(\sigma,v) = \{f/5,g/2,h/1\}$.
\eoe
\end{example}

An action $a$ is
\emph{executable} in a state  ${v}$
if there is  an  axiom
$\codetext{executable}(a,C)$ such that
$ {v}\models  C$.

\begin{remark}
As for the case of the language $\mathcal{B}$,
also in \BMV\ the executability laws express
necessary but not sufficient preconditions
for action execution (cf., Remark~\ref{execVSnonexec}).
Moreover,  thanks to the generality of the
constraint language---i.e., any propositional combination of primitive
constraints can be used in \BMV---the \codetext{executable} laws also allow direct formulation of
non-executability conditions and the roles of the \codetext{executable}
and \codetext{nonexecutable} axioms coincide.
\end{remark}

Let us denote with $Dyn(a)$ the set of dynamic causal law axioms for
action $a$.
The \emph{effect} of executing $a$ in state $v$,
denoted by $\Eff(a, v)$, is a constraint defined as follows:
\[ \Eff(a, v) =\bigwedge \left\{ C \:|\:
        \codetext{causes}(a,C,C_1) \in Dyn(a),
         {v}\models  C_1\right\}. \]

\subsubsection{\BMVi\ without static causal laws}\label{sema_0}

Let us start by considering the simplified situation in which
$\mathcal{SL}=\emptyset$, i.e., no static causal laws are specified
in the domain description.

During the execution of an action \codetext{a},
a fluent has to be considered as inertial, provided that
it does not appear among the effects of the dynamic laws for \codetext{a}.
In other words, since these effects are expressed through constraints,
a fluent is inertial if it does not occur in any
of the constraints specified in the dynamic laws for \codetext{a}.

The description of the state transition system
corresponding to a given action description theory
$\langle \mathcal{DL}, \mathcal{EL}, \emptyset \rangle$
can be completed by defining the notion of transition.

A triplet $\langle v,a,v'\rangle$, where $v,v'$
are complete states and $a$ is an action, is a \emph{valid state
transition} if:
\begin{itemize}
\item the action $a$ is executable in $ {v}$, ~and
\item
$v'
= \ine(\sigma,v)$,
where $\sigma$ is a solution of the constraint
$\Eff(a,{v})$.
\end{itemize}

Let  $\langle \mathcal{D}, \mathcal{O}\rangle$ be an instance of a planning problem,
 $\vec{v} = \langle v_0, \dots, v_{\n}\rangle$ be a sequence
of complete states and $a_1,\ldots,a_{\n}$ be actions.
We say that $\langle v_0,a_1,v_1, \dots, a_{\n},v_{\n}\rangle$
is a \emph{valid trajectory} if:
\begin{itemize}

\item for each axiom of the form $\codetext{initially}(C)$ in $\mathcal{O}$,
we have that ${v}_0 \models~C$,
\item for each axiom of the form $\codetext{goal}(C)$ in $\mathcal{O}$, we have that
    ${v}_{\n} \models C$, and
\item for all $i \in \{0,\dots,\n-1\}$,
$\langle{v}_{i},a_{i+1},v_{i+1}\rangle$ is a valid state transition.
\end{itemize}

\begin{example}
Let us consider the Example \ref{rico2}. Observe that
$\langle \{f/1,g/1,h/1\}, a, \{f/5, g/3, h/1\}\rangle$ is a valid
trajectory.
\eoe
\end{example}

\begin{remark}
Given a planning problem $\langle\mathcal{D},\mathcal{O}\rangle$ in \BMVi,
differently from what happens in the case of $\mathcal{B}$,
a solution to a planning problem is described by a valid trajectory, not just by
a sequence of actions. This is the case because actions might have
non-deterministic effects. For instance, let us consider Example \ref{rico2}. If the action \codetext{a}
is executed and the precondition \codetext{g<3} holds, then the dynamic causal law
imposes the constraint \codetext{f=g+2} in the reached state. There are
many different ways to satisfy this requirement.
Hence, in general, a sequence of actions might not characterize a unique state sequence.

The same argument also applies to the action description language \BMV,
so in what follows we will consider the valid trajectories as the solutions of a
planning problem.
\end{remark}

\subsubsection{Abstract implementation in absence of static laws}\label{absConcBMVi}\label{nonMarkovianBMV}

In this section we propose
a constraint-based characterization of the
state transition system defined in \Sect{sema_0}.
Similarly to what we have done in the case of $\mathcal{B}$,
for any specific state, each fluent $f$ will be represented by
an integer-valued constraint variable.
Boolean variables will instead model the occurrences of actions.

Let $u$ be a state;
given a fluent $f$, we indicate with $F_f^u$ the
variable representing $f$ in $u$.
We generalize such a notation to any constraint $C$, i.e.,
we denote with $C^u$ the constraint obtained from $C$ by
replacing each fluent $f \in \FV(C)$
by $F_f^u$.
For each action $a \in \mathcal{A}$, a Boolean
variable $A_a^u$ is introduced,
representing whether the action
is executed or not in the transition from
$u$ to the next state.

Given a specific fluent $f$,
we develop a system of constraints to constrain the values of~$F_f^u$.
Let us consider the dynamic
causal laws that have $f$ within their consequences:
$$
\mathcal{DL}_f ~ = ~ \big\{\codetext{causes}(a_{i_{f,1}}, C_{{f,1}}, \alpha_{{f,1}}),~
\cdots,~ \codetext{causes}(a_{i_{f,m_f}}, C_{{f,m_f}}, \alpha_{{f,m_f}})
\big\}
$$

For each action $a$ we will have its executability
conditions:
$$
\mathcal{EL}_a ~ = ~ \big\{\codetext{executable}(a, \delta_{a,1}),~
\cdots,~ \codetext{executable}(a, \delta_{a,p_{a}})
\big\}
$$

\begin{figure}
\begin{center}
{\small\fbox{\begin{minipage}[c]{.87\textwidth}
\begin{eqnarray}
 F_f^v,F_f^u     &   \in         & \dom(f) \label{cieffebis9}\\
 A_a^u         & \rightarrow   & \displaystyle{\bigvee_{j=1}^{p_{a}} \delta^u_{a,j}}  \label{cieffebis10}\\
 A_{a_{i_{f,j}}}^u \wedge \alpha_{{f,j}}^u & \leftrightarrow & Dyn^u_{f,j}~~~~~~~~~\forall j \in \{1,\ldots,m_f\}  \label{cieffebis11}\\
Dyn^u_{f,j}  & \rightarrow   & C_{{f,j}}^v~~~~~~~~~~~~~~\forall j \in \{1,\ldots,m_f\}  \label{cieffebis12}\\
\displaystyle{\neg \bigvee_{j=1}^{m_f} Dyn^u_{f,j}} & \rightarrow & F_f^u = F_f^v \label{cieffebis13}
\end{eqnarray}
\end{minipage}}
}\end{center}
\caption{\label{cieffebis}The constraints $C^{u,v}_{f,a}$ for a state transition from $u$ to $v$,
for a fluent $f$.}
\end{figure}

\noindent
\Fig{cieffebis} describes the constraints
$C^{u,v}_{f,a}$
that can be used in encoding the relations that determine
the value of the fluent $f$ in the state $v$ (i.e., constrain the variable $F_f^v$)
w.r.t.\ the application of the
action $a$ in the state~$u$.
After the settings of the domains (by~(\ref{cieffebis9})),
we impose through~(\ref{cieffebis10}) that if  action $a$ is
executed, then at least one of the preconditions for its executability must hold in $u$.
For each $j\in\{1,\ldots,m_f\}$ the constraint~(\ref{cieffebis11}) defines a Boolean flag $Dyn_{f,j}^u$
that holds if and only if action $a_{i_{f,j}}$
is applicable in $u$ and the preconditions of the $j^{th}$ dynamic causal law
for $f$ holds in $u$.
The constraint~(\ref{cieffebis12}) requires that if $Dyn_{f,j}^u$ is true, then
the corresponding effects must
hold in the new state $v$. Finally, inertia constraints are set by means of~(\ref{cieffebis13}).

\smallskip

We will denote
with $C_{f}^{u,v}$ the conjunction of these constraints for all actions $a\in\mathcal{A}$.
Given an action domain specification over the signature $\langle \mathcal{V},
\mathcal{F}, \mathcal{A}\rangle$ and two states $u,v$,
 the system of constraints $C_{\mathcal{F}}^{u,v}$ includes:
\begin{itemize}
\item the constraint $C_{f}^{u,v}$ for each fluent literal
        $f$ in the language of $\mathcal{F}$
\item the constraint
$\sum_{a \in \mathcal{A}} A_a^u= 1$
(unique action execution in the state transition).
\end{itemize}

The next theorem states completeness and soundness of the encoding described so far.
We need a further piece of notation. Given two states $u,v$ and  an action~$a$,
let $C_{\mathcal{F}}^{u,a}$ be the constraint obtained from $C_{\mathcal{F}}^{u,v}$ by
setting $A_a=1$, $A_b=0$ for all $b\not=a$, and $F_f^u=u(f)$
for each fluent literal $f$.

\begin{theorem}\label{soundcompleteBMV0}
Let $\mathcal{D} = \langle \mathcal{DL},\mathcal{EL},\emptyset\rangle$
and let $u,v$ two states and $a$ an action.
Then
$\langle u,a,v\rangle$ is a valid
transition in the semantics of the language \BMVi\
if and only if $v$ represents a solution of the constraint
$C_{\mathcal{F}}^{u,a}$.
\end{theorem}
\begin{proof}
\begin{description}
\item[($\Rightarrow$)]
Let $\langle u,a,v\rangle$ be a valid
transition. Then, $a$ is executable in $u$. Hence $u\models \delta_{a,j}$ for some
$j\in\{1,\ldots,p_{a}\}$ and~(\ref{cieffebis10}) is satisfied.
By the definition of state we have that~(\ref{cieffebis9}) is also satisfied.
Let  $v=\ine(\sigma,u)$ with $\sigma$ solution of $\Eff(a,u)$.

If $f$ is a fluent not belonging to $\dom(\sigma)$ then $f$ does not occur in $\Eff(a,u)$
and it is not affected by any dynamic causal law involved in the state transition.
By definition of $\ine(\cdot)$ we have that $v(f)=u(f)$ and this satisfies constraint~(\ref{cieffebis13}).
Satisfaction of constraints~(\ref{cieffebis11}) and~(\ref{cieffebis12}) is immediately verified
by observing that for all dynamic causal laws
$\codetext{causes}(a_{i_{f,h}}, C_{{f,h}}, \alpha_{{f,h}})$ having $f$ in $C_{{f,h}}$,
the constraint $\alpha_{{f,h}}$ is false in $u$. Then, the corresponding flag
$Dyn_{f,h}^u$ is set false by~(\ref{cieffebis11}). Consequently,~(\ref{cieffebis12})
is satisfied.

Assume now that $f$ is a fluent in $\dom(\sigma)$. This means that there are
dynamic causal laws $\codetext{causes}(a_{i_{f,h}}, C_{{f,h}}, \alpha_{{f,h}})$
such that $\alpha_{{f,h}}$ is true in $u$, for $h\in X=\{j_1,\ldots,j_r\}\subseteq\{1,\ldots,m_f\}$.
Consequently, the flag $Dyn_{f,h}^u$ is set true for $h\in X$ and false otherwise.
Since $\sigma$ is a solution of $\Eff(a,u)$, $v$ satisfies the constraint
$C_{{f,j}}^v$ for all $j\in X$.
This implies that~(\ref{cieffebis12}) is satisfied for each $j\in\{1,\ldots,m_f\}$.
Since some flags $Dyn_{f,i}^u$ are true constraint~(\ref{cieffebis13}) is satisfied too.

\item[($\Leftarrow$)]
Assume that $v$ satisfies the constraint $C_{\mathcal{F}}^{u,a}$.
By~(\ref{cieffebis10}), because $A_a=1$, some of the constraints $\delta_{a,h}^u$
is satisfied. Hence, action $a$ is executable in $u$.
By the satisfaction of~(\ref{cieffebis11})
and~(\ref{cieffebis12}), $v$ satisfies all constraints $C_{{f,j}}^v$ for which the
corresponding $\alpha_{{f,j}}^u$ is satisfied. Then, $v$ is a solution for $\Eff(a,u)$.
Consequently, since $v=\ine(v,u)$ (by definition, since $v$ is complete),
$\langle u,a,v\rangle$ is a valid transition.
\end{description}
\qed\end{proof}

Let $\langle\mathcal{D}, \mathcal{O}\rangle$ be an instance of a planning problem
where
$\mathcal{D}$ is an action description and~$\mathcal{O}$ contains any
number of axioms of the form
$\codetext{initially}(C)$ and
$\codetext{goal}(C)$.
We can state the following.
\begin{theorem}\label{completeBMV0soundBMV0}
There is a valid trajectory $\langle v_0,a_1,v_1,a_2,\ldots,a_{\n},v_{\n} \rangle$
if and only if
there is a solution for the constraint
$$
\bigwedge_{\codetext{{\scriptsize initially}}(C)\in\mathcal{O}}C^{v_0}
\:\wedge\:
\bigwedge_{j=0}^{\n-1}\big(C_{\mathcal{F}}^{v_j,v_{j+1}}\big)
\:\wedge\:
\bigwedge_{\codetext{{\scriptsize goal}}(C)\in\mathcal{O}}C^{v_{\n}}
$$
\end{theorem}
\begin{proof}
The result follows from (repeated) applications of Theorem~\ref{soundcompleteBMV0}.
\qed\end{proof}

\subsubsection{Adding static causal laws}\label{sema_1}

In this section we consider the case of action theories
$\langle \mathcal{DL},\mathcal{EL},\mathcal{SL}\rangle$ of \BMVi,
involving static causal laws (i.e., such that $\mathcal{SL}\not=\emptyset$).

The presence of static laws requires refining the semantics
of the language, in order
to ensure proper treatment of inertia in the construction of a
valid trajectory.

We start by defining three operations $\cap,\cup$, and $\triangle$
on states, as follows:
$$
\begin{array}{rcl}
v_1\cup v_2 (f) & = & \left\{\begin{array}{ll}
                v_1(f) &  \textit{if }\:v_1(f)=v_2(f)\\
                v_1(f) &  \textit{if }\: v_2(f)=\bot\\
                v_2(f) &  \textit{if }\: v_1(f) = \bot\\
                \bot &  \textit{otherwise}
                 \end{array}
            \right. \\
v_1\cap v_2 (f) & = &
        \left\{\begin{array}{ll}
                v_1(f) &  \textit{if } \:v_1(f)=v_2(f)\\
                    \bot & \textit{otherwise}
                \end{array}
        \right.\\
\triangle(v_1,v_2,S)(f) & = &\left\{\begin{array}{cl}
v_1(f) & \mbox{\textit{if} $f \in S$}\\
v_2(f) & \mbox{otherwise}
\end{array}\right.
\end{array}$$
where the set $S$ used in $\triangle$ is a set of fluents.
Observe that $\ine(\sigma,v) = \triangle(\sigma,v,\dom(\sigma))$.

\medskip

A state $v$ is \emph{closed} w.r.t.\ a set of static causal laws
$$
\mathcal{SL}
=   \{\codetext{caused}(C_1,D_1),\dots, \codetext{caused}(C_k,D_k)\}
$$
if $v \models (C_1 \rightarrow D_1) \wedge \cdots \wedge
(C_k \rightarrow D_k)$. We denote this property as $v \models \mathcal{SL}$.

Given two states $v,v'$, a set of fluents $D$,
and a set $\mathcal{SL}$ of static causal  laws,
we say that $v'$ is \emph{minimally closed} w.r.t.\ $v,D,$ and $\mathcal{SL}$
if
\begin{itemize}
\item $v'\models\mathcal{SL}$ (i.e., $v'$ is closed) and
\item for all $S \subseteq D$, if $\triangle(v,v',S) \neq v'$
then $\triangle(v,v',S) \not\models \mathcal{SL}$.
\end{itemize}
The notion of minimally closed state is intended to capture the law of inertia,
w.r.t.\ a given set $D$ of fluents.
Notice, in fact, that $\triangle(v,v',\emptyset) = v'$.
Intuitively speaking,
$v'$ is minimally closed when it is obtainable from $v$ by applying
a minimal set of (necessary) changes in the values of the `inertial' fluents (those in $D$).
In other words, it is not possible to obtain from $v$ a state different from $v'$
and closed w.r.t.\ $\mathcal{SL}$, by applying ``fewer changes'' than those involved in
obtaining~$v'$.
A pictorial representation of $\triangle(v,v',X)$ is shown in \Fig{chiusuraminima}.

Observe that if $\mathcal{SL}=\emptyset$ then $v'$ is minimally
closed w.r.t.\ $v,D,$ and $\mathcal{SL}$
if and only if $v =v'$.

\begin{figure}
\begin{center}
\fbox{\begin{minipage}[c]{.8\textwidth}
\begin{tabular}{c}
\begin{pspicture}(-1.0,1)(8.5,4.4)
\put(0,0){\psscalebox{1.1}{\figurauno}}
\end{pspicture}
\\
~\\
\begin{pspicture}(-1.0,1)(8.5,4.4)
\put(0,0){\psscalebox{1.1}{\figuradue}}
\end{pspicture}
\end{tabular}
\end{minipage}}
\end{center}
\caption{\label{chiusuraminima}The set $\triangle(v,v',X)$ is obtained by combining a portion of $v$
and a portion of $v'$, depending on the third argument $X$, which acts as a regulator in ``mixing''
portions of $v$ and~$v'$.
The figure visualizes, in gray, the two sets $\triangle(v,v',D)$ (above)
and $\triangle(v,v',S)$ (below)
for $S\subseteq D\subseteq \mathcal{F}$ and illustrates the definition of minimal closure.
A state $v'$ is minimally closed if and only if $v'\models \mathcal{SL}$
and for all $S\subseteq D$, if $\triangle(v,v',D)\not=v'$ then
$\triangle(v,v',S)\not\models\mathcal{SL}$.
In both cases, the surrounding frame represents the set $\mathcal{F}$ of all fluents.}
\end{figure}
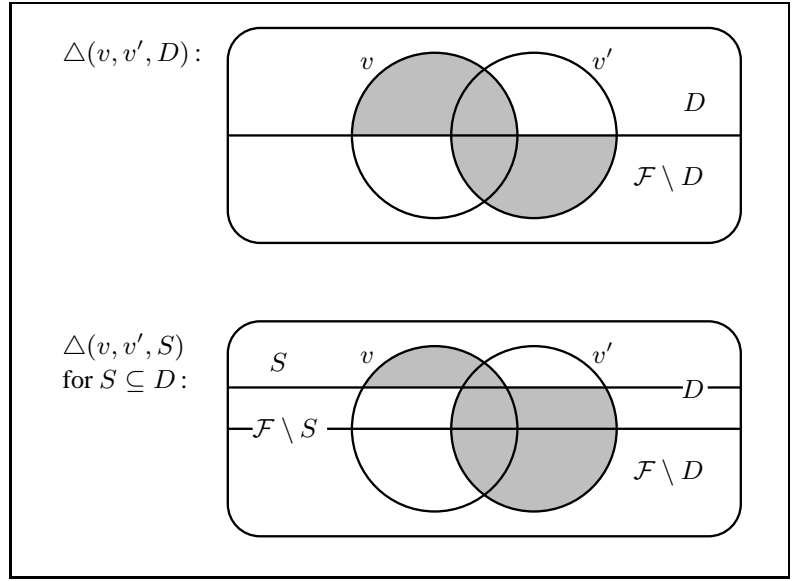

\begin{example}\label{esa:miniclo}
Let $f,g,h$ be fluents with $\dom(f)=\dom(g)=\dom(h)=\{0,1\}$ and
\\\centerline{$\mathcal{SL}
=   \{\codetext{caused}(f=1,g=1),\codetext{caused}(f=0,g=0)\}$.}
Consider the states $v=\{f/0, g/0, h/0\}$, $v' = \{f/1, g/1, h/1\}$,
$v'' = \{f/0, g/0, h/1\}$ and let $D = \{f,g\}$.
Then, $v'$ and $v''$ are both closed w.r.t.\ $\mathcal{SL}$.

However,
$v''$ is minimally closed w.r.t.\ $v$, $D$, and $\mathcal{SL}$,
while
$v'$ is not minimally closed since
$\triangle(v,v',D) = \{h/1,f/0,g/0\}$ is different
from $v'$ and closed.
\eoe
\end{example}

A triplet
$\langle v, a, v'\rangle$,
where $v$ and $v'$ are complete states and $a$ is an action,
is a \emph{valid transition} if:

\begin{enumerate}
\item
the action $a$ is executable in $v$ and
\item
we have that $v'= \ine(\sigma,v')$
where
\begin{itemize}
\item $\sigma$ is a solution of the constraint
$\Eff(a,{v})$,
and
\item $v'$ is minimally closed w.r.t.\
$v$, $\mathcal{F} \setminus \dom(\sigma)$, and
$\mathcal{SL}$.
\end{itemize}
\end{enumerate}
Intuitively, the conditions that define a transition are
designed to guarantee that:
\begin{itemize}
\item a solution $\sigma$ for the constraints
        describing the effects of the action is determined;
\item such solution is part of the new state $v'$ constructed
        (thanks to $v' = \ine(\sigma,v')$); and
\item the new state is minimally closed with respect to
        all the fluents not affected by the execution of the
        action.
\end{itemize}
Let us observe that,
since all fluents in the domain of any solution $\sigma$ of
$\Eff(a,v)$ maintain  the same value in~$v'$,
it holds that
$v' \models\Eff(a,v)$.

Notice that the notion of a valid transition given in presence of static laws
properly extends the one given in \Sect{sema_0}.
In fact, the following property holds:
\begin{lemma}
If ${\cal SL} = \emptyset$ then $\ine(\sigma,v)=\ine(\sigma,v')$.
\end{lemma}
\begin{proof}
It is sufficient to note that,
if ${\cal SL} = \emptyset$ then
$v'$ is minimally closed w.r.t. ${\cal F}\setminus dom(\sigma)$
if and only if $\ine(\sigma,v)=v'$.
\end{proof}

\begin{example}
Let us extend the action description of Example \ref{esa:miniclo}. We consider
the following domain description:
\[\begin{array}{lcl}
\codetext{fluent}(f,\{0,1\}). & \hspace{.5cm} & \codetext{fluent}(g,\{0,1\}).\\
\codetext{fluent}(h,\{0,1\}).\\
\codetext{action}(a). & & \codetext{executable}(a,h=0).\\
\codetext{causes}(a,h=1) &&\\
\codetext{caused}(f=1,g=1). & & \codetext{caused}(f=0,g=0).
\end{array}
\]
Let us consider the three states $v=\{f/0, g/0, h/0\}$, $v' = \{f/1, g/1, h/1\}$,
and $v'' = \{f/0, g/0, h/1\}$.
Then $\langle v,a,v''\rangle$ is a valid transition, while
$\langle v,a,v'\rangle$ is not.
\eoe
\end{example}

Let  $\langle \mathcal{D}, \mathcal{O}\rangle$ be a planning problem instance.
Let $\vec{v} = \langle v_0, \dots, v_{\n}\rangle$ be a sequence
of complete states and let $a_1,\ldots,a_{\n}$ be actions.
Then $\langle v_0,a_1,v_1, \dots, a_{\n},v_{\n}\rangle$
is a \emph{valid trajectory}
if the following conditions hold:
\begin{itemize}

\item $v_{0} \models \mathcal{SL}$, and
    for each axiom $\codetext{initially}(C)$ in $\mathcal{O}$,
 we have that ${v}_0 \models C$;
\item
for each axiom of the form $\codetext{goal}(C)$ in $\mathcal{O}$,
we have that ${v}_{\n} \models C$;
\item
$\langle v_{i},a_{i+1},v_{i+1}\rangle$ is a valid transition,
for each $i \in \{0,\dots,\n-1\}$.
\end{itemize}

\subsubsection{Abstract implementation in presence of static laws}\label{absConcBMVii}
Let us consider a fluent $f$ and a transition from state $u$ to state $w$,
due to an action $a$,
and let us adopt the same notation ($F_f^u$, $C^u$, $A_a^u$, etc.)
introduced in \Sect{absConcBMVi}.
The state transition from $u$ to $w$ can be seen as the composition of two steps
involving an intermediate state $v$.
The first of these steps reflects the effects of the dynamic laws, whereas
the second step realizes the closure w.r.t.\ the static causal laws.
Hence we proceed by introducing a set of variables corresponding to
the intermediate state $v = \ine(\sigma,u)$,
where $\sigma$ is a solution of $\Eff(a,u)$.
The constraint-based description of the first step is essentially the same we
described in \Sect{nonMarkovianBMV}---thus, we only need to
extend the constraint system defined in \Fig{cieffebis} to reflect the second part of the transition.

\medskip

Given a set $L\subseteq\mathcal{F}$ of fluents, let $\mathcal{SL}_L\subseteq\mathcal{SL}$ be
the collection of all static causal laws in which at least one fluent of $L$ occurs.
Moreover, for simplicity, let $\mathcal{SL}_f$ denote $\mathcal{SL}_{\{f\}}$, i.e.,
the set of all static causal laws that involve the fluent $f$.

Let us define a relation $R\subseteq \mathcal{F}\times\mathcal{F}$ so that
$f_1{R}f_2$ if and only if $\mathcal{SL}_{f_1}\cap\mathcal{SL}_{f_2}\not=\emptyset$.
$R$ is an equivalence relation and it partitions $\mathcal{F}$.
Each element (i.e., equivalence class) of the quotient $\mathcal{F}/R$ is said to
be a \emph{cluster} (w.r.t.\  $\mathcal{SL}$).
Notice that a cluster can be a singleton $\{f\}$.
Let $f$ be a fluent, we denote with $L_f$ its cluster w.r.t.\  $\mathcal{SL}$.

\begin{example}
Assume that $\mathcal{SL}$ consists of the rules
\[
\begin{array}{lclcl}
\codetext{caused}(\codetext{true},f=1). & \hspace{.2cm} &
\codetext{caused}(g=2,h=3). & \hspace{.2cm} &
\codetext{caused}(h<5,r=2).
\end{array}
\]
Then the two clusters are $\{f\}$ and $\{g,h,r\}$.
\eoe
\end{example}

Given a fluent $f$,
let us consider the sets of dynamic and executability
laws $\mathcal{DL}_f$ and $\mathcal{EL}_a$,
as defined in \Sect{absConcBMVi}.
Moreover, let us consider the cluster containing $f$, let it be $L_f=\{f_1,\ldots,f_k\}$,
and the corresponding set of static causal laws  $\mathcal{SL}_{L_f}$:
$$
\mathcal{SL}_{L_f} ~ = ~ \big\{\codetext{caused}(G_{f,1},D_{f,1}),~
\cdots, ~
\codetext{caused}(G_{f,h_f},D_{f,h_f})
\big\}.
$$

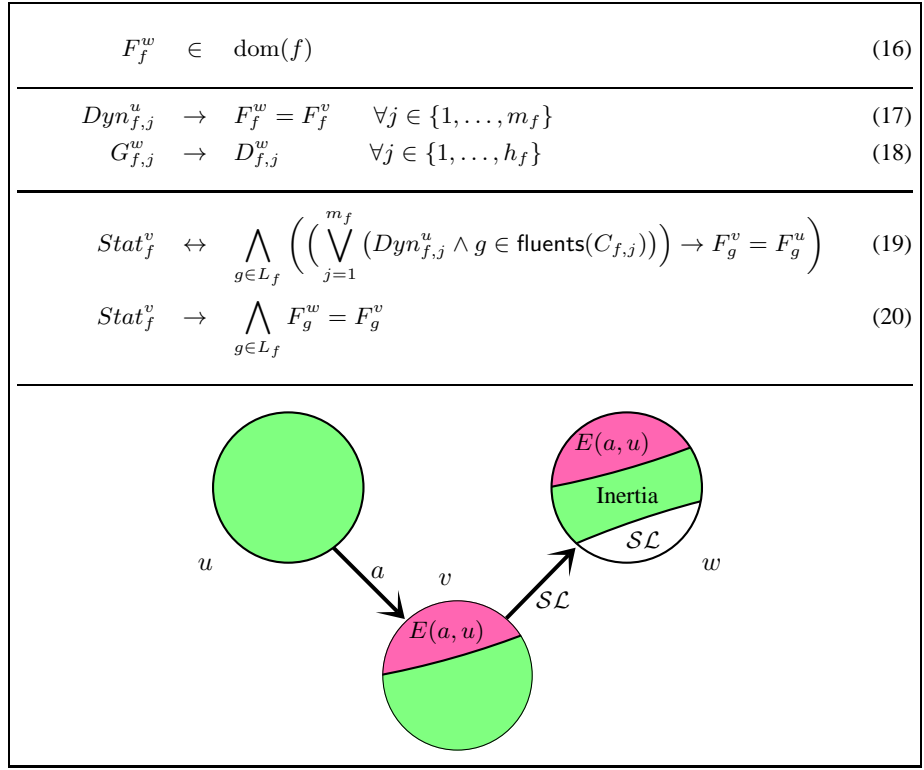
\begin{figure}
\begin{center}
{\small\fbox{\begin{minipage}[c]{.94\textwidth}
\begin{eqnarray}
 F_f^w & \in & \dom(f) \label{cieffequatris15}\\
\hline
Dyn^u_{f,j} & \rightarrow & F_f^w=F_f^v ~~~~~~~\forall j \in \{1,\ldots,m_f\}\label{cieffequatris20}~~~\\
G^w_{f,j} & \rightarrow & D^w_{f,j} ~~~~~~~~~~~~~~~\forall j \in \{1,\ldots,h_f\}\label{cieffequatris21}~~~\\
\hline
Stat_f^v & \leftrightarrow &
\bigwedge_{g\in L_f}\bigg(\Big(\bigvee_{j=1}^{m_f}\big(Dyn^u_{f,j}  \wedge g \in \FV(C_{{f,j}})\big)\Big) \rightarrow F_{g}^v=F_{g}^u\bigg)
\label{cieffequatris22}~~~~\\
Stat_f^v & \rightarrow &
\bigwedge_{g\in L_f}F_{g}^w=F_{g}^v
\label{cieffequatris23}\\
\hline\nonumber
\end{eqnarray}
\centerline{
\begin{pspicture}(0,1)(7.7,5.5)
\put(0,0.1){\psscalebox{1.0}{\figuratre}}
\end{pspicture}
}
\end{minipage}}}
\end{center}
\caption{\label{cieffequatris}The constraints for a state transition from $u$ to $w$  (with intermediate state~$v$), for a fluent $f$.}
\end{figure}

\noindent
\Fig{cieffequatris} describes the constraints (to be added to those in \Fig{cieffebis})
that are used in encoding the relations that determine
the value of the fluent $f$ in state $w$ (represented through the variable $F_f^w$)
after the execution of
action $a$ in the state~$u$ (we recall that $v$ is to be considered as an
intermediate state $v=\ine(\sigma,u)$).

The constraint (\ref{cieffequatris15}) sets the domains for the
variables $F_f^w$. The constraint~(\ref{cieffequatris20}) propagates to $w$ the effects of the dynamic laws.
Constraint~(\ref{cieffequatris21}) imposes closure w.r.t.\ the static causal laws.
Finally, constraints~(\ref{cieffequatris22})--(\ref{cieffequatris23}) require that if all
the fluents in $\dom(\sigma)$ that belong to  the cluster $L_f$
are left unchanged in the transition, then all the fluents of $L_f$
should not change their values.
More precisely,
as far as~(\ref{cieffequatris22}) is concerned, $Stat_f^v$ is set
to true if, for all fluents
$g$ in $L_f$, either $g$ is not affected by the dynamic laws (i.e., $F_{g}^v=F_{g}^u$), or
for each activated dynamic law $\codetext{causes}(a_{i_{f,j}}, C_{{f,j}}, \alpha_{{f,j}})$
(i.e., such that its precondition $\alpha_{{f,j}}^u$ is true), $g$
does not occur in its effects (i.e., in $C_{{f,j}}$).
Notice that, with respect to a specific state transition,
we are not considering subject to inertia all those fluents that
occur in the effects of (at least) one activated dynamic law.

The enforcement of the constraint~(\ref{cieffequatris23})
constitutes a necessary, but not sufficient,
condition for the target state to be minimally closed.
We will discuss later on this point.

Let us denote
with $C_{f}^{u,w}$ the conjunction of the constraints
(\ref{cieffebis9})--(\ref{cieffequatris21}) for all actions
$a\in\mathcal{A}$.
Given an action domain specification over the signature $\langle \mathcal{V},
\mathcal{F}, \mathcal{A}\rangle$,
 the system of constraints $C_{\mathcal{F}}^{u,w}$ includes:
\begin{itemize}
\item the constraint $C_{f}^{u,w}$ for each fluent literal
        $f$ in the language of $\mathcal{F}$;
\item the constraint
$\sum_{a \in \mathcal{A}} A_a^u= 1$.
\end{itemize}
Similarly, let $\mathit{Stat}_{\mathcal{F}}^{u,w}$  denote the conjunction of all
the constraints of the forms~(\ref{cieffequatris22}) and (\ref{cieffequatris23}).

The next theorem states completeness of the encoding described so far.
Again, given two states $u,w$ and  an action~$a$,
let $C_{\mathcal{F}}^{u,a}$ and $\mathit{Stat}_{\mathcal{F}}^{u,a}$
denote the constraints obtained from
$C_{\mathcal{F}}^{u,w}$  and $\mathit{Stat}_{\mathcal{F}}^{u,w}$,
respectively,
by setting $A_a=1$, $A_b=0$ for all $b\not=a$, and $F_f^u=u(f)$
for each fluent literal $f$.

\begin{theorem}\label{completeBMV1}
Let $\mathcal{D} = \langle \mathcal{DL},\mathcal{EL},\mathcal{SL}\rangle$
and let $u,w$ two states and $a$ an action.
Then, if
$\langle u,a,w\rangle$ is a valid
transition in the semantics of the language \BMVi,
then $w$ represents a solution of the constraint
$C_{\mathcal{F}}^{u,a} \wedge \mathit{Stat}_{\mathcal{F}}^{u,a}$.
\end{theorem}
\begin{proof}
For the constraints~(\ref{cieffebis9})--(\ref{cieffequatris15}),
considering the transition from $u$ to $v$, the proof proceeds
analogously to  the first part of the proof of Theorem~\ref{soundcompleteBMV0}.

Let us sketch the part of the proof regarding the effect of the static causal laws.
Since $\langle u,a,w\rangle$ is a valid
transition, $w=\ine(\sigma,w)$, \,$w$ agrees with $v=\ine(\sigma,u)$
on all fluents in $\dom(\sigma)$, hence~(\ref{cieffequatris20}) hold.
Moreover, $w$ is closed w.r.t.\ $\mathcal{SL}$, hence it satisfies~(\ref{cieffequatris21}).
{F}rom the fact that $w$ is minimally closed w.r.t.\
$\ine(\sigma,u),\mathcal{F}\setminus\dom(\sigma)$, and $\mathcal{SL}$,
it follows that $w$ satisfies~(\ref{cieffequatris22})--(\ref{cieffequatris23}).
\qed\end{proof}

\medskip

The above encoding does not guarantee soundness.
This is because the constraints~(\ref{cieffequatris20})--(\ref{cieffequatris21})
in \Figs{cieffebis} and~\ref{cieffequatris} might admit solutions not corresponding to
minimally closed states.

We introduced the notion of cluster to partially recover the soundness
of the encoding. Intuitively speaking,
a cluster generalizes, to the multi-valued case, the notion of loop seen
in \Sect{SoundCompleOfB}: a cluster is a set of fluents whose values have
been declared to be mutually dependent through a set of static causal laws.
In a state transition, similarly to the case of loops, changes to
the fluents of a cluster might occur because of their mutual influence, not
being (indirectly) caused by dynamic laws.

Constraints~(\ref{cieffequatris22}) and (\ref{cieffequatris23}) impose inertia
on all the fluents of a cluster whenever none of them is influenced by dynamic laws.
However, note that imposing~(\ref{cieffequatris22})--(\ref{cieffequatris23}) does not
completely circumvent the problem because state transitions violating the inertia are
still admitted.
In fact, (\ref{cieffequatris22})--(\ref{cieffequatris23}) do not impose
inertia on the fluents of a cluster when at least one of them is changed
by the dynamic laws. This might lead to invalid transitions, in which
a change in the value of a fluent of a cluster happens even if
this is not necessary in order to satisfy all the static causal laws.

Nevertheless, we introduced the constraints~(\ref{cieffequatris22})
and (\ref{cieffequatris23})  because they constitute
a good compromise w.r.t.\ the efficiency of a
concrete implementation (as discussed later).

To completely enforce soundness, we
need to apply a filter on the solutions that are admitted by the
encoding described so far.
To this aim, let us introduce a condition on the values of the fluent, which is
intended to mimic, in the multi-valued setting, the effect of loop formulae.

Let us assume that the action $a$ is executed in the state $u$, and that
$\sigma$, $v$, and $w$ have been determined so that to satisfy the
constraint $C_{\mathcal{F}}^{u,w}$.
In this situation the following constraint characterizes
an hypothetical state $x$, different from $w$:
\begin{eqnarray}
\mathit{Form}(\mathcal{D})^{u,a} & = &
\bigg(
~
C_{\mathcal{F}}^{u,x}  ~\wedge \label{formulazza2}\\
& & ~~~~
\bigwedge_{f\in\mathcal{F}}\Big(\bigvee_{j=1}^{m_f}Dyn^u_{f,j}\rightarrow F_{f}^x=F_{f}^w\Big) ~\wedge \label{formulazza1}\\
& & ~~~~
\bigvee_{f\in\mathcal{F}}F^x_{f}\not=F^w_{f} ~\wedge~
 \bigwedge_{f\in\mathcal{F}}\Big(F^x_{f}\not=F^w_{f} \rightarrow F^x_{f}=F^u_{f}\Big)~
\bigg)\label{formulazza3}
\end{eqnarray}
Intuitively, the satisfaction of such a formula witnesses the existence of a counterexample
for the minimal closure of $w$.
Notice that, being $\sigma$, $v$, and $w$ already determined, the only fluents/variables to
be determined are those describing the state~$x$, if any.
The conjunct in line~(\ref{formulazza2}) states that $x$ is a target state
alternative to $w$; in particular, it enforces the closure of $x$ w.r.t.\ $\mathcal{SL}$.
The conjunction~(\ref{formulazza1}) states that $x$ and $w$ agree on the fluents in $\dom(\sigma)$.
Finally,~(\ref{formulazza3}) states that $x$ must differ from $w$ and it
must  agree with $u$ in at least one fluent---that,
because of~(\ref{formulazza1}), it is in $\mathcal{F}\setminus\dom(\sigma)$.

We can prove the following result, that generalizes Theorem~\ref{soundcompleteBMV0}
to the case of $\mathcal{SL}\not=\emptyset$.
\begin{theorem}\label{soundBMV1}
Let $\mathcal{D} = \langle \mathcal{DL},\mathcal{EL},\mathcal{SL}\rangle$
and $u,w$ two states, with $u$ closed w.r.t. $\mathcal{SL}$.
Let $a$ an action such that
$w$ represents a solution of the constraint
$C_{\mathcal{F}}^{u,a}$.
Then
$\langle u,a,w\rangle$ is a valid
transition in the semantics of the language \BMVi,
if $\mathit{Form}(\mathcal{D})^{u,a}$ is unsatisfiable.
\end{theorem}
\begin{proof}
By proceeding as in the proof of Theorem~\ref{soundcompleteBMV0},
we can show that all needed conditions for $\langle u,a,w\rangle$
to be a valid transition are satisfied, except for the minimal closure of $w$.

Let us assume, by contradiction,
that $w$ is not minimally closed w.r.t.\ $u$, $\mathcal{F}\setminus\dom(\sigma)$,
and $\mathcal{SL}$.
Then, there exists $S\subseteq\mathcal{F}\setminus\dom(\sigma)$ such that
$x=\triangle(u,w,S)\not=w$ and $x\models\mathcal{SL}$
For each fluent $f\not\in S$ it holds that $F_f^x=F_f^w$. Moreover, $F_f^v=F_f^w$
holds too, because $w$ satisfies $C_{\mathcal{F}}^{u,a}$.
Hence, $Dyn^u_{f,j}\rightarrow F_{f}^x=F_{f}^v$ holds for all~$j$.

For each fluent $f$, since $x$ is closed w.r.t.\ $\mathcal{SL}$, we have that
$G^x_{f,j}\rightarrow D^x_{f,j}$ (for all $j \in \{1,\ldots,h_f\}$).
Observe that the conditions of the forms~(\ref{cieffebis9})--(\ref{cieffequatris15})
in the conjunct at line~(\ref{formulazza2}) (i.e., in $C_{\mathcal{F}}^{u,x}$)
do not depend on the specific $x$. Then, the conjunct~(\ref{formulazza2})
is satisfied.

Let us also observe that condition~(\ref{formulazza1}) holds too. This is so
because, for all $f\in\dom(\sigma)$ we have that $F_{f}^x=F_{f}^w=F_{f}^v$.
{F}rom the fact that $x\not=w$ it follows that
$\bigvee_{f\in\mathcal{F}}F^x_{f}\not=F^w_{f}$ holds.
Finally, the condition~(\ref{formulazza3}) is satisfied because,
whenever $F^x_{f}\not=F^w_{f}$ holds, by the definition of $\triangle$,
it must be the case that $F^x_{f}=F^u_{f}$.
It follows that $\mathit{Form}(\mathcal{D})^{u,a}$ is satisfiable (by $x$).

This is a contradiction and proves that $w$ is minimally closed w.r.t.\ $u$,
$\mathcal{F}\setminus\dom(\sigma)$, and $\mathcal{SL}$,
and that $\langle u,a,w\rangle$ is a valid transition.
\qed\end{proof}

Let $\langle\mathcal{D}, \mathcal{O}\rangle$ be an instance of a planning problem,
where
$\mathcal{D}$ is a domain description and $\mathcal{O}$ contains any
number of axioms of the form
$\codetext{initially}(C)$ and
$\codetext{goal}(C)$.
We conclude this section by stating a generalization of Theorem~\ref{completeBMV0soundBMV0}
to the case of $\mathcal{SL}\not=\emptyset$.

\begin{theorem}\label{completeBMV1soundBMV1}
There is a valid trajectory $\langle v_0,a_1,v_1,a_2,\ldots,a_{\n},v_{\n} \rangle$
if and only if
\begin{itemize}
\item
$v_0\models\mathcal{SL}$
\item
There is a solution for the constraint
$$
\bigwedge_{\codetext{{\scriptsize initially}}(C)\in\mathcal{O}}C^{v_0}
\:\wedge\:
\bigwedge_{j=0}^{\n-1}C_{\mathcal{F}}^{v_j,v_{j+1}}
\:\wedge\:
\bigwedge_{\codetext{{\scriptsize goal}}(C)\in\mathcal{O}}C^{v_{\n}}
$$
\item
 For each $j\in\{0,\ldots,\n-1\}$ the formula
$\mathit{Form}(\mathcal{D})^{v_{j},a_{j+1}}$ is unsatisfiable.
\end{itemize}
\end{theorem}
\begin{proof}
The result follows from Theorems~\ref{completeBMV1} and~\ref{soundBMV1}.
\qed\end{proof}

\begin{remark}[Embedding of $\mathcal{B}$ into \BMVi]
We conclude this section by showing that \BMVi\ is at least as expressive as
$\mathcal{B}$.
To this aim it suffices to describe how to translate a
 domain description $\mathcal{D}$ of $\mathcal{B}$
to a \BMVi domain description $\mathcal{D}'$,
in such a way that the semantics of the domain is preserved.
Let us outline the main points of such a translation.

Each Boolean fluent $f$ in $\mathcal{D}$
can be modeled in  \BMVi\ by a multi-valued fluent $f'$ whose
domain is $\mathcal{V} = \{0,1\}\subseteq\z$.

\noindent
Each action in $\mathcal{D}$ uniquely corresponds to an action in $\mathcal{D}'$.

\noindent
Let us consider a
dynamic causal law of $\mathcal{D}$, e.g.,
$${\codetext{causes}(a, f, [f_1,\ldots,f_k, \codetext{neg}(g_1),\ldots,\codetext{neg}(g_h)])}.$$
This law is translated in $\mathcal{D}'$ as
$${\codetext{causes}(a, f'=1,[f'_1=1,\ldots,f'_k=1, g'_1=0, \ldots, g'_h=0])}.$$
In a similar manner, static laws and executability conditions of $\mathcal{D}$
are mapped into \BMVi.
Consequently, the two domain descriptions $\mathcal{D}$ and $\mathcal{D}'$
describe two isomorphic transition systems.
\end{remark}

\subsection{Adding annotated fluents and non-Markovian references}\label{sema_2}

In this section, we generalize the treatment described in \Sect{BMVisemantics}
in order to provide a state-transition semantics for \BMV\ suitable to cope with temporal
references.
The first form of temporal references involves annotated fluents and concerns relative access
to their past values, w.r.t.\ the current state.
There is no restriction on the occurrences of this kind of annotated fluents: they
might be used in all laws of a domain description.
In this case, the extension of the semantics described in \Sect{BMVisemantics}
comes rather naturally.
Since references may relate different points in time along the plan,
the approach consists of considering
sequences of states instead of pairs of states, to define the transition constraints.

Regarding  references to future points in time (i.e., positively annotated
fluents), we recall that they
are admitted in the consequences of dynamic causal laws only.
This restriction allows the treatment of future and past references by
exploiting the very same mechanisms.
The semantics is further enriched in \Sect{sema_4} to encompass state constraints specified
by using absolute time references, as well as costs.

\medskip

Let $\vec{v} = \langle v_0,\dots,v_\n \rangle$ be a  state sequence.
Given $\vec{v}$, and $i \in \{0,\dots,\n\}$,
we define the concept of \emph{value} of $\varphi$ in  $\vec{v}$ at time $i$
(with abuse of notation, denoted by $\vec{v}(i,\varphi)$) as follows:\footnote{A
slightly simplified treatment could be described if only past references are admitted.
In this case, we consider~$i$ to be the current point in time and $j$ to be negative.
The notation could then be simplified by considering just a
prefix $\vec{v} = \langle v_0,\dots,v_i \rangle$ of the state sequence.}
$$
\begin{array}{l}
\vec{v}(i,x) = x  ~~ \mbox{ if $x \in \mathcal{V}$}
\\
\vec{v}(i,f^j) = v_{i+j}(f) ~~ \mbox{ if $f \in \mathcal{F}$, and $0\leqslant i+j\leqslant \n$}
\\
\vec{v}(i,f^j) = v_0(f) ~~ \mbox{ if $f \in \mathcal{F}$ and $i+j < 0$}
\\
\vec{v}(i,f^j) = v_{\n}(f) ~~ \mbox{ if $f \in \mathcal{F}$ and $i+j>\n$}
\\
\vec{v}(i,abs(\varphi)) = |\vec{v}(i,\varphi)|
\\
\vec{v}(i,-(\varphi)) = -(\vec{v}(i,\varphi))
\\
\vec{v}(i,\varphi_1 \oplus \varphi_2)= \vec{v}(i,\varphi_1) \oplus \vec{v}(i,\varphi_2)
\\
\vec{v}(i, \codetext{rei}(C))= 1 ~~  \mbox{if $\vec{v}\models_i C$}
\\
\vec{v}(i, \codetext{rei}(C))= 0 ~~  \mbox{if $\vec{v}\not\models_i C$}
\end{array}
$$
where $n \in\mathcal{V}$,\: $\oplus \in \{+,-,*,/,\codetext{mod}\}$.

As for (\ref{eq:semOfBMV}) of \Sect{BMVisemantics},
the semantics of reified constraints relies on
the notion of satisfaction, which in turn has to be contextualized to a
specific point in time~$i$.  More formally,
given a fluent constraint $\varphi_1 \:\codetext{op}\:\varphi_2$ and a state sequence
$\vec{v}$,
the notion of satisfaction at time $i$ is defined as
{$  \vec{v} \models_i  \varphi_1  \:\codetext{op}\:\varphi_2 \:\Leftrightarrow
\:
    \vec{v}(i,\varphi_1) \:\codetext{op}\:  \vec{v}(i,\varphi_2)
$.}
The notion $\models_i $ is generalized to the case of
propositional combinations of  fluent constraints in the usual manner.

\smallskip

Given a constraint $C$, let $\GFV(C)$ be the set of annotated fluents
$f^i$, for $i \geqslant 0$, occurring in $C$.
Given a state sequence $\vec{v} = \langle v_0,\dots,v_i\rangle$,
with $0\leqslant i<\n$,
a function $\sigma:\GFV(C) \longrightarrow \mathcal{V}$
is an \emph{$i$-solution} of $C$ w.r.t. $\vec{v}$,
if it holds that
$$\langle v_0,\dots,v_i,\ine(\sigma|_0,v_i),
\overline{(\sigma|_1)}, \dots,\overline{(\sigma|_{\n-(i+1)})}\rangle \models_{i+1} C,$$
where each $\sigma|_k$ (for $k\geqslant 0$)
is the restriction of the assignment $\sigma$ to the fluent annotated with $k$,
and $\overline \mu$ denotes the substitution obtained by completing $\mu$,
with assignment to $\bot$ for all fluents not in $\dom(\mu)$.
Note that we treat the interpretation of the various
operations as strict w.r.t. $\bot$ and we assume satisfied all constraints that
refer to undefined expressions.
Hence, for instance, if $C$ is constraint and there is a
sub-expression $\psi$ of $C$ evaluated as $\bot$, then we assume $\vec{v}\models_i C$.

\begin{example}
\label{rico3}
Let $\n=3$ and $i=1$.
Consider the constraint $C\,\equiv\,(g^0 = f^{-1}+f^{-2})$ and let
$\vec{v} = \langle v_0, v_1\rangle = \langle \{f/2,g/1\}, \{f/1,g/2\}\rangle$.

Then $\sigma = \{g/3\} = \sigma|_0$ is a 1-solution of the constraint~$C$,
since
\begin{itemize}
\item
 $\ine(\sigma|_0, \{f/1,g/2\}) =
 \ine(\{g/3\}, \{f/1,g/2\}) =
 \{f/1,g/3\}$, ~and
\item
$\langle \{f/2,g/1\}, \{f/1,g/2\}, \{f/1,g/3\},
\{f/\bot,g/\bot\},\rangle\models_2 g^0=f^{-1}+f^{-2}$,
in fact, we have that $\vec{v}(2,C)$ is $\vec{v}(2,g^0)=\vec{v}(2,f^{-1}+f^{-2})$,
which is equivalent to $v_2(g^0)=v_1(f)+v_0(f)$.
\end{itemize}
\eoe
\end{example}

A state sequence $\vec{v} = \langle v_0,\dots,v_h \rangle$ is
\emph{closed} w.r.t.\ a set of static causal laws $$\mathcal{SL} =
\{ \codetext{caused}(C_1,D_1),\dots,\codetext{caused}(C_k,D_k)\}$$ if
for all $i \in \{0,\dots,h\}$ it holds that $\vec{v}\models_i
(C_1 \rightarrow D_1) \wedge \cdots \wedge (C_k \rightarrow D_k)$.

We also generalize the notion of minimal closure as follows: given a
state sequence $\vec{v} = \langle v_0, \dots,v_{i}\rangle$
and a  state $v'$ we say that $v'$ is minimally closed w.r.t.
$\vec{v}$, $D$, and $\mathcal{SL}$ if
\begin{itemize}
\item $\langle v_0, \dots, v_i,v' \rangle$ is closed w.r.t.  $\mathcal{SL}$
\item for all sets of fluents $S \subseteq D$, if the state $\Delta(v_{i},v',S)$
    is different from $v'$, then
    $\langle v_0,\dots,v_i,\Delta(v_{i},v',S)\rangle$ is
    not closed w.r.t. $\mathcal{SL}$.
\end{itemize}

The action $a$ is
\emph{executable} in $\vec{v}$ at time $i$
if there is  an  axiom
$\codetext{executable}(a,C)$ such that
$ \vec{v}\models_i  C$.

Let us denote with $Dyn(a)$ the set of dynamic causal laws for an action $a$.
The \emph{effects} of executing $a$ in $\vec{v}$ at time $i$,
denoted by $\Eff(a, \vec{v}, i)$, is
\[ \Eff(a, \vec{v},i) =\bigwedge \left\{ PC \:|\:
        \codetext{causes}(a,PC,C) \in Dyn(a),
         \vec{v}\models_i  C\right\} \]

Given a constraint $C$, we denote by $\mathit{shift}^t(C)$ the constraint obtained
from $C$ by replacing each fluent $f^x$ with $f^{x-t}$.

Let us assume that $\vec{v} = \langle v_0, \dots, v_{i}\rangle$ is
a sequence of complete states and that $\vec{a}$ is a sequence of actions $\langle a_1,\dots,a_{i+1}\rangle$.
The effects of the sequence of actions in
$\vec{v}$ is represented by the formula
$$E(i,\vec{a}, \vec{v}) =
  \bigwedge_{j=0}^i \mathit{shift}^{j-i}\big(\Eff(a_{j+1},\vec{v}, j)\big)
  \wedge
  \bigwedge_{j=0}^i \bigwedge_{f \in \mathcal{F}}
      f^{j-i} = v_{j}(f)$$

Let us observe that this constraint might involve all fluents of
the states $v_0,\dots,v_i$, as well as fluents of future states.
The values of fluents in states $v_0,\dots,v_i$ are fixed by $\vec{v}$.

Let  $\langle \mathcal{D}, \mathcal{O}\rangle$ be a planning problem instance,
$\vec{v} = \langle v_0, \dots, v_{\n}\rangle$ be a sequence
of complete states and $a_1,\ldots,a_{\n}$ be actions.
Then, $\langle v_0,a_1,v_1, \ldots, a_{\n},v_{\n}\rangle$
is a \emph{valid trajectory} if the following conditions
hold:
\begin{itemize}
\item
$\langle v_0, \dots, v_{\n} \rangle$
is closed w.r.t. $\mathcal{SL}$

\item
for each axiom of the form $\codetext{initial}(C)$ in $\mathcal{O}$,
    we have that $\vec{v} \models_0 C$
\item
for each axiom of the form $\codetext{goal}(C)$ in $\mathcal{O}$, we have that
    $\vec{v} \models_{\n} C$
\item
for each $i \in \{0,\dots,\n-1\}$ the following conditions hold
    \begin{itemize}
\item
action $a_{i+1}$ is executable in $\vec{v}$ at time $i$ and
\item
we have that
$v_{i+1}  = \ine(\sigma|_0,v_{i+1})$
where
\begin{itemize}
\item
$\sigma$ is a $i$-solution of the
constraint
$E(i,\langle a_1,\dots,a_{\n}\rangle, \langle v_0,\dots,v_{\n-1}\rangle)$ w.r.t.\
 $\langle v_0, \dots,v_i\rangle$,
\item
$v_{i+1}$ is minimally closed
w.r.t.\ $\langle v_0,\dots,v_i\rangle$,
$\mathcal{F} \setminus \dom(\sigma)$, and $\mathcal{SL}$.
\end{itemize}
\end{itemize}
\end{itemize}

\begin{example}
Let us consider the following domain specification and planning
 problem instance (for $\n=2$):
\[\begin{array}{lcl}
\codetext{fluent}(f,1,5). &~~~~~~~~~~~~ &\\
\codetext{fluent}(g,1,5).&&\\
\codetext{fluent}(h,1,5). &    & \\
\codetext{action}(a).  &    & \\
\codetext{action}(b).  &    & \\
\codetext{executable}(a, \codetext{true}). && \\
\codetext{executable}(b, \codetext{true}).&&\\
\codetext{causes}(a, g^0=g^{-1}+2, \codetext{true}). && \\
\codetext{causes}(b,f^0=g^{-1}+h^{-2}, \codetext{true}). &&\\
\codetext{initially}(f=1).   && \\
\codetext{initially}(g=1). &&\\
\codetext{initially}(h < 3). && \\
\codetext{goal}(f > 4).&&
\end{array}
\]
 Observe that the only valid trajectory is
$$\langle \{f/1,g/1,h/2\}, a, \{f/1,g/3,h/2\}, b, \{f/5,g/3,h/2\}\rangle.$$
The validity can be verified by observing that:
\begin{itemize}
\item
$\{f/1,g/1,h/2\}$ satisfies all the constraints provided in the
        \codetext{initial} declarations;
\item
$\{f/5,g/3,h/2\}$ satisfies the goal constraint $f>4$;
\item
the action $a$ is executable in $\langle v_0\rangle= \langle \{f/1,g/1,h/2\}\rangle$ and
        action $b$ is executable in
$$\langle v_0, v_1\rangle= \langle \{f/1,g/1,h/2\}, \{f/1,g/3,h/2\}\rangle$$
(since both their executability laws and the action conditions are trivially true).

\item
Consider the first state transition and $i=0$ and note that $\GFV(g^0=g^{-1}+2)=\{g\}$.
Then, $\sigma'=\{g/3\}$ is a $0$-solution of $g^0=g^{-1}+2$
w.r.t.\ $\langle \{f/1,g/1,h/2\}\rangle$.
In fact, $\sigma'|_0=\sigma'$,\, $\sigma'|_1=\{\}$, and
        \begin{itemize}
        \item
 $v_1 = \ine(\sigma'|_0,v_0)= \ine(\{g/3\},\{f/1,g/1,h/2\}) = \{f/1,g/3,h/2\}$
        \item
$\langle  v_0, v_1, \overline{\sigma'|_1} \rangle = \langle v_0, v_1, \{f/\bot,g/\bot,h/\bot\} \rangle
\models_1 g^0=g^{-1}+2$.
        \item
$ v_1$ is minimally closed
                w.r.t.
           $\langle \{f/1,g/1,h/2\} \rangle$, $\{f,h\}$ and $\emptyset$.
        \end{itemize}

\item
Consider the second state transition and $i=1$ and note that $\GFV(f^0=g^{-1}+h^{-2})=\{f\}$.
Then, $\sigma''=\{f/5\}$ is a $1$-solution of $f^0=g^{-1}+h^{-2}$
w.r.t.\ $\langle v_0, v_1\rangle$.
In fact, $\sigma''|_0=\sigma''$, and
        \begin{itemize}
        \item
 $v_2 = \ine(\sigma''|_0,v_1)= \ine(\{f/5\},\{f/1,g/3,h/2\}) = \{f/5,g/3,h/2\}$
        \item
$\langle  v_0, v_1, v_2\rangle \models_2 f^0=g^{-1}+h^{-2}$.
        \item
$ v_2$ is minimally closed
                w.r.t.
           $\langle v_0, v_1 \rangle$, $\{g,h\}$ and $\emptyset$.
        \end{itemize}
\end{itemize}
\eoe
\end{example}

\subsection{Abstract implementation of \BMV}\label{abstractImplBMV}

The constraint encoding for \BMV\ is similar to the one developed earlier
for the case
of \BMVi\ (cf., \Figs{cieffebis} and~\ref{cieffequatris}).
In the encoding of a trajectory
$\langle v_0,a_1,v_1, \dots, a_{\n},v_{\n}\rangle$
in \BMVi, we introduced a variable $F_{f}^{v_i}$
to represent the value of the fluent $f$ in the $i^{th}$ state~$v_i$.
In each state transition, say from $v_i$ to $v_{i+1}$,
the implementation of \BMVi\
imposes only constraints involving
the variables/fluents of the current state.
In the language encompassing timed references, each constraint occurring
in the action description
can address the values that fluents assume in any
of the states of the sequence $\vec{v} = \langle v_0, \dots, v_{\n}\rangle$.
Since all the variables representing these values are
present in the encoding,
only the following change is needed to adapt to \BMV\ the
implementation designed for \BMVi:
to obtain from a constraint $C$ (involving fluents), a
constraint $C^{{\vec{v}},i}$ (involving the corresponding variables),
at time $i$,
we replace each $f^j$ with the variable $F^{v_{i+j}}_f$.

By adopting this refined construction for $C^{{\vec{v}},i}$,
we can inherit all the results of \Sect{absConcBMVii}.
In particular, for an action description $\mathcal{D}$,
similarly to what done in \Sect{absConcBMVii}, we
denote by $C_{\mathcal{F}}^{{\vec{v}},a_i}$ and by $\mathit{Form}(\mathcal{D})^{{\vec{v}},a_i}$
the constraints homologous to $C_{\mathcal{F}}^{v_{i-1},v_{i}}$ and $\mathit{Form}(\mathcal{D})^{v_{i-1},a_i}$,
respectively.

The completeness result for \BMV\ directly generalizes that obtained for
\BMVi. With regards to  soundness, the observation made w.r.t.\ \BMVi\
in \Sect{absConcBMVii} still applies.
In fact, let $\langle\mathcal{D}, \mathcal{O}\rangle$ be an instance of a planning problem
where $\mathcal{D}$ is a domain
description and $\mathcal{O}$ contains  axioms of the form $\codetext{initially}(C)$ and $\codetext{goal}(C)$.
We state the following:

\begin{theorem}\label{completeBMVsoundBMV}
There is a valid trajectory ${\vec{v}} = \langle v_0,a_1,v_1,\ldots,v_{\n},v_{\n} \rangle$
if and only if
\begin{itemize}
\item
${\vec{v}}$ is closed w.r.t.\ $\mathcal{SL}$
\item
There is a solution for the constraint
$$
\bigwedge_{\codetext{{\scriptsize initially}}(C)\in\mathcal{O}}C^{{\vec{v}},0}
\:\wedge\:
\bigwedge_{j=0}^{\n-1}C_{\mathcal{F}}^{{\vec{v}},a_{j+1}}
\:\wedge\:
\bigwedge_{\codetext{{\scriptsize goal}}(C)\in\mathcal{O}}C^{{\vec{v}},\n}
$$
\item
For each $j\in\{0,\ldots,\n-1\}$ the formula
$\mathit{Form}(\mathcal{D})^{{\vec{v}},a_{j+1}}$ is unsatisfiable.
\end{itemize}
\end{theorem}

\subsection{Adding costs and global constraints}\label{sema_4}

Cost and time constraints can be introduced by filtering the
solutions characterized by Theorem~\ref{completeBMVsoundBMV}, in
order to rule out
the unsatisfactory solutions.
More precisely, given a trajectory
$\langle v_0,a_1,v_1,\dots,a_\n,v_\n \rangle$
satisfying the requirements of Theorem~\ref{completeBMVsoundBMV},
we say that the trajectory satisfies
a set of global constraints as described in \Sects{absolutetemporalconstraints}
and~\ref{costconstraints}
 if all the constraints described next hold.

Let us start by investigating the cost constraints.
Let
$$\codetext{action\_cost}(a_1,FE_1), \dots,\codetext{action\_cost}(a_{\n},FE_{\n})$$ and
$\codetext{state\_cost}(FE')$
be specified in the action description.\footnote{As mentioned, if some of these
assertion is missing a default cost 1 is assumed.}

Let us recall that the general form of cost constraints is $\codetext{cost\_constraint}(C)$,
 where $C$ is a constraint defined as in \Sect{BMVsyntax}, with
the added ability to refer to
the atoms \codetext{plan}, \codetext{goal}, and \codetext{state($i$)} wherever
fluents can be used.
Consequently, we extend our definition of value of an expression $\varphi$
in $\vec{v}=\langle v_0,\dots,v_\n \rangle$ at time $i$ (for all $j$):
$$
\begin{array}{l}
\vec{v}(j,\codetext{plan}) = v_{0}(FE_1)+\cdots+v_{\n-1}(FE_{\n})
\\
\vec{v}(j,\codetext{goal}) = v_{\n}(FE')
\\
\vec{v}(j,\codetext{state($i$)}) = v_{i}(FE') ~~ \mbox{ if $0\leqslant i\leqslant \n$}
\end{array}
$$
(assigning cost constraints to to states outside the plan is senseless.
However, for completeness, for $i < 0$ or $i>\n$ we set
$\vec{v}(j,\codetext{state($i$)}) =  0$ but any other choice ---
e.g., $\bot$,  or the values on states $0$ or $\n$ --- is reasonable).
This modification allows us to derive the notion of satisfaction of
a cost constraint $C$
from the notion of satisfaction
defined in \Sect{abstractImplBMV}.
As particular cases, we obtain that:
\begin{itemize}
\item
for each assertion  $\codetext{cost\_constraint}(\codetext{plan}\,\op\,k)$
the plan cost
$(v_{0}(FE_1)+\cdots+v_{\n-1}(FE_{\n}))$ has to satisfy the
stated constraint, i.e.,
it must hold that
$(v_{0}(FE_1)+\cdots+v_{\n-1}(FE_{\n}))\,\op\, k$;

\item
for each assertion $\codetext{cost\_constraint}(\codetext{goal} \,\op\, k)$,
the cost $v_\n(FE')$ of the goal state must satisfy the constraint:
$v_\n(FE')  \,\op\, k$;
\item
for each assertion
$\codetext{cost\_constraint}(\codetext{state($i$)} \,\op\, k)$,
the cost $v_i(FE')$ assigned to the  $i^{th}$  state
has to satisfy the constraint  $v_i(FE')  \,\op\, k$.
\end{itemize}

The handling of time constraints requires the following modifications:
\begin{itemize}
\item
for each assertion
$\codetext{time\_constraint}(C)$,
it holds that $\langle v_0,\dots,v_\n \rangle \models_0 C$, where each timed fluent
$f{\verb"@"}i$ is evaluated as $v_i(f)$;

\item
for each assertion of the form
$\codetext{holds}(C,i)$ it holds that $\langle v_0,\dots,v_\n \rangle  \models_i C$;

\item
for each assertion of the form
$\codetext{always}(C)$, it holds that
$\langle v_0,\dots,v_\n \rangle  \models_i C$
for all $i\in\{0,\dots,\n\}$.
\end{itemize}

Moreover, if $\codetext{minimize\_cost}(FE'')$ is specified,
then there exists no other trajectory $\vec{v}'$ such that
$\vec{v}'(\n,FE'') < \vec{v}(\n,FE'')$.
As particular cases, we have that
\begin{itemize}
\item
if \codetext{minimize\_cost}(\codetext{plan}) is specified,
then there exists no other trajectory having a smaller plan cost;
\item
if \codetext{minimize\_cost}(\codetext{goal}) is specified in the action description,
then there is no trajectory $\langle v'_0,a'_1,v'_1,\dots,a'_\n,v'_\n \rangle$,
fulfilling all constraints, and such that  $v'_\n(FE') < v_\n(FE')$.
\end{itemize}

In this manner, we characterize the solutions of a given planning problem
to be exactly those solutions described by Theorem~\ref{completeBMVsoundBMV}
that additionally satisfy all the global constraints,
the requirements on costs, and the time constraints
expressed in the action description.
Soundness and completeness properties directly
carry over.

\begin{figure}
\begin{minipage}[t]{1.0\textwidth}
\begin{codice}{55}
\label{bmapcode056} set\_one\_fluent(fluent(FluentName,IV),~ActionOccs,~Now,~States) :- \\
\label{bmapcode057} \tab    findall([Act,OP,FE1,FE2,L],\\
\label{bmapcode058} \tab\tab\tab~(causes(Act,FC,L), zero\_subterm(FluentName,FC),\\
\label{bmapcode059} \tab\tab\tab~FC =..~[OP,FE1,FE2]), Dyn),\\
\label{bmapcode060} \tab    state\_select(Now, States, FromState),\\
\label{bmapcode061} \tab    Next is Now+1,\\
\label{bmapcode062} \tab    state\_select(Next, States, ToState),\\
\label{bmapcode063} \tab    member(fluent(FluentName,EV), ToState),\\
\label{bmapcode064} \tab    dynamic(Dyn, ActionOccs, FromState, DynFormula, Next, States),\\
\label{bmapcode065} \tab    cluster\_rules(FluentName, Stat), \%\%\% These 2 lines can be dropped in\\
\label{bmapcode066} \tab    static(Stat, States, Next, StatFormula),\%\%\% absence of static laws\\
\label{bmapcode067} \tab    bool\_disj(DynFormula, StatFormula, Formula),\\
\label{bmapcode068} \tab    \verb"#\" Formula \verb"#=>" EV \verb"#=" IV.
\smallskip\\
\label{bmapcode069} dynamic([], \_, \_, [], \_, \_).\\
\label{bmapcode070} dynamic([[Act,OP,FE1,FE2,Prec]|Rest],AOccs,State,[Flag|PF1],Now,States)~:-\\
\label{bmapcode071} \tab    member(action(Act,VA), AOccs),\\
\label{bmapcode072} \tab    Last is Now-1, \%\%\% Looks for preconditions in FromState and before\\
\label{bmapcode073} \tab    get\_precondition\_vars(Last, Prec, States, ListPV),\\
\label{bmapcode074} \tab    length(Prec, NPrec),\\
\label{bmapcode075} \tab    sum(ListPV, SumPrec),\\
\label{bmapcode076} \tab    \%\%\% The effect is in the next state (Now=Last+1)\\
\label{bmapcode077} \tab    rel\_parsing(FE1, Val1, Now, States),\\
\label{bmapcode078} \tab    rel\_parsing(FE2, Val2, Now, States),\\
\label{bmapcode079} \tab    exp\_constraint(Val1, OP, Val2, C),\\
\label{bmapcode080} \tab    (VA  \verb"#/\" (SumPrec \verb"#=" NPrec)) \verb"#<=>" Flag, \\
\label{bmapcode081} \tab    Flag \verb"#=>" C,\\
\label{bmapcode082} \tab    dynamic(Rest, ActionOccs, State, PF1, Now, States).
\smallskip\\
\label{bmapcode083} rel\_parsing(Num, Num, \_, \_) :- \\
\label{bmapcode084} \tab   integer(Num), !.
\\
\label{bmapcode085} rel\_parsing(rei(RC), Val, Time, States) :- \\
\label{bmapcode086} \tab   RC =..~[OP,E1,E2],\\
\label{bmapcode087} \tab   rel\_parsing(E1, Val1, Time, States),\\
\label{bmapcode088} \tab   rel\_parsing(E2, Val2, Time, States),\\
\label{bmapcode089} \tab   exp\_constraint(Val1, OP, Val2, Val), !.
\\
\label{bmapcode090} rel\_parsing(abs(FE), Val, Time, States) :- \%\%\% similar for -(FE) \\
\label{bmapcode091} \tab   rel\_parsing(FE, Val1, Time, States),\\
\label{bmapcode092} \tab   Val \verb"#=" abs(Val1), !.
\\
\label{bmapcode093} rel\_parsing(FE, Val, Time, States) :- \\
\label{bmapcode094} \tab   FE =..~[OP,FE1,FE2], \\
\label{bmapcode095} \tab   member(OP, [+,-,mod,/,*]),\\
\label{bmapcode096} \tab   rel\_parsing(FE1, Val1, Time, States),\\
\label{bmapcode097} \tab   rel\_parsing(FE2, Val2, Time, States),\\
\label{bmapcode098} \tab   ( OP = + -> Val \verb"#=" Val1 + Val2;\\
\label{bmapcode099} \tab     OP = - -> Val \verb"#=" Val1 - Val2;\\
\label{bmapcode100} \tab     OP = * -> Val \verb"#=" Val1 * Val2;\\
\label{bmapcode101} \tab     OP = / -> Val \verb"#=" Val1 / Val2;\\
\label{bmapcode102} \tab     OP = mod -> Val \verb"#=" Val1 mod Val2 ), !.
\\
\label{bmapcode103} rel\_parsing(Fluent\verb"^"Delta, Val, Time, States) :- \\
\label{bmapcode104} \tab   H is Time+Delta,\\
\label{bmapcode105} \tab   length(States,N),\\
\label{bmapcode106} \tab    in\_interval(H,N,E),\\
\label{bmapcode107} \tab    state\_select(E, States, State),\\
\label{bmapcode108} \tab    member(fluent(Fluent,Val),State),!.
\\
\label{bmapcode109} rel\_parsing(Fluent \verb"@" Time, Val, \_, States) :- \\
\label{bmapcode110} \tab   state\_select(Time,States,State),\\
\label{bmapcode111} \tab   member(fluent(Fluent,Val),State), !.
\\
\label{bmapcode112} rel\_parsing(Fluent, Val, Time, States) :- \\
\label{bmapcode113} \tab   state\_select(Time, States, State),\\
\label{bmapcode114} \tab   member(fluent(Fluent,Val), State).
\medskip\\
\label{bmapcode115} parsing(Fluent, Val, State) :- \\
\label{bmapcode116} \tab   rel\_parsing(Fluent, Val, 0, [State]).
\medskip\\
\label{bmapcode117} exp\_constraint(L, OP, R, C) :- \\
\label{bmapcode118} \tab       (OP == eq  -> C \verb"#<=>" L \verb"#="  R;\\
\label{bmapcode119} \tab        OP == neq -> C \verb"#<=>" L \verb"#\=" R;\\
\label{bmapcode120} \tab        OP == geq -> C \verb"#<=>" L \verb"#>=" R;\\
\label{bmapcode121} \tab        OP == leq -> C \verb"#<=>" L \verb"#=<" R;\\
\label{bmapcode122} \tab        OP == gt  -> C \verb"#<=>" L \verb"#>"  R;\\
\label{bmapcode123} \tab        OP == lt  -> C \verb"#<=>" L \verb"#<"  R).
\end{codice}
\end{minipage}
\caption{\label{Fig:codice_BMVi}Relevant parts of the \BMV\ implementation.}
\end{figure}

\section{Concrete Implementation of \BMV}\label{concreteImpleBMV}

The overall structure of the concrete implementation of the
language \BMV\ follows that used for implementing
the $\mathcal{B}$ language. We focus here on the main differences.

To start, let us briefly describe the code depicted in \Fig{Fig:codice_BMVi} and
show that this concrete implementation reflects the abstract one
defined in \Fig{cieffebis}.\footnote{Observe that the concrete implementation
uses the functors {\tt eq}, {\tt neq}, etc. to denote the primitive
constraints $=$, $\neq$, etc.}
Hence, we preliminarily ignore lines (\ref{bmapcode065})--(\ref{bmapcode066})
of \Fig{Fig:codice_BMVi}.

The first difference w.r.t.\ the implementation of $\mathcal{B}$ (cf., \Sect{implemB})
is that each fluent variable is
assigned to a finite set domain, drawn from the fluent declaration---instead
of being treated as a Boolean variable.

The predicate  \codetext{set\_one\_fluent}
(lines (\ref{bmapcode056})--(\ref{bmapcode068})) has a similar role
as in the implementation of $\mathcal{B}$.
Given the fluent \codetext{FluentName},
the relevant parts of the dynamic causal laws
are collected in  lines (\ref{bmapcode057})--(\ref{bmapcode059}).
The predicate \codetext{zero\_subterm} is an auxiliary predicate
that detects if a constraint involves a fluent---i.e., it
looks for an occurrence of \codetext{FluentName}
in the constraint
imposed by the dynamic causal laws.
All the fluents explicitly involved
in the consequence of a dynamic law are collected.
In line (\ref{bmapcode063}), the variable \codetext{EV}
identifying the fluent \codetext{FluentName} in the following state
\codetext{ToState} is retrieved.

The  predicate \codetext{dynamic} (line (\ref{bmapcode064}))
collects the list of Boolean flags \codetext{DynFormula}.
If one of the variables in \codetext{Dyn}
is true then the variable \codetext{EV} is involved in
a constraint imposed by a dynamic causal law.
In line (\ref{bmapcode067}) the disjunction of these
flag variables is computed in \codetext{Formula}
(let us ignore, for the time being, the variable \codetext{StatFormula}).
In line (\ref{bmapcode068}) the inertia constraint is added:
if \codetext{Formula} is false then the value of the fluent
is left unchanged by the transition (i.e., $\codetext{IV}=\codetext{EV}$).
This corresponds to the $\ine(\cdot)$ operator.

For each action \codetext{Act} affecting the value \codetext{EV},
the predicate \codetext{dynamic} (lines (\ref{bmapcode069})--(\ref{bmapcode082}))
retrieves its preconditions
and builds the constraint \codetext{C} involving \codetext{EV}
that must be imposed if the preconditions
are satisfied.
The flag variable \codetext{Flag} in line (\ref{bmapcode080})
is introduced to keep track of the fact that the action
has occurred (i.e., \codetext{VA} is true) and
the corresponding precondition holds.
If \codetext{Flag} is true then  the constrain \codetext{C}  is
asserted (line (\ref{bmapcode081})).
All flags are stored in a list (cf., the variable \codetext{DynFormula}
in line (\ref{bmapcode064})).

Lines (\ref{bmapcode083})--(\ref{bmapcode114}) provide an excerpt of the definition
of the predicate \codetext{rel\_parsing}.
This predicate is used to
transform fluent expressions to internal
expressions involving fluent variables.
\codetext{States} is a list of states (each of them, in turn is a
list of all the fluent variables).
The first argument is the fluent expression and the second one
is the output internal expression.
The argument \codetext{Time} represents the specific point in time
in which a fluent is referred to (cf., the variable \codetext{Now}
used in lines (\ref{bmapcode069})--(\ref{bmapcode082}) and (\ref{bmapcode124})--(\ref{bmapcode134})
to specify the precise point in time in which a fluent expression/constraint has to be evaluated).
The predicate \codetext{in\_interval} called in line (\ref{bmapcode106})
sets $\codetext{E}=\codetext{H}$ if $0 \leq \codetext{H} \leq \n$,
$\codetext{E}=0$ (resp., $\codetext{E}=\n$) if $\codetext{H}<0$ (resp., $\codetext{H} > \n$).
Similarly, predicate \codetext{exp\_constraint} (lines (\ref{bmapcode117})--(\ref{bmapcode123}))
transforms fluent constraints into the corresponding constraints
on the fluent variables.

The above described fragment of implementation is completed with the code needed
to handle initial and goal state specifications. Namely,
for a specific instance of a planning problem $\langle\mathcal{D}, \mathcal{O}\rangle$,
as done for $\mathcal{B}$, all constraint on the initial state (resp., those
on the goal state) are reflected by constraining the variables $F_f$ in the
representation of the initial (resp., final) state.

We  proceed by splitting the
correctness proof into steps.
We can now state the following
result.\footnote{When establishing completeness an soundness results for the
concrete implementation, we assume the same properties hold for
the real implementation of the CLP(FD) solver at hand (in our case, SICStus Prolog).}

\begin{theorem}\label{concretesoundcompleteNostatic}
The concrete implementation (partially depicted in \Fig{Fig:codice_BMVi}) is
correct and complete w.r.t.\
the system of constraints of \Fig{cieffebis}.
\end{theorem}
\begin{proof}
This result immediately follows from the above argument.
In fact, the constraint~(\ref{cieffebis9}) of \Fig{cieffebis}
is implicitly rendered by domain assignment for CLP variables.
Constraints~(\ref{cieffebis11}) and (\ref{cieffebis12}) are dealt with in
lines~(\ref{bmapcode057})--(\ref{bmapcode064}).
Line~(\ref{bmapcode068}) imposes constraint~(\ref{cieffebis13}).
Concerning the sequentiality of the plan and the executability conditions
(i.e., constraint~(\ref{cieffebis10})), we can observe that
the implementation does not differ from that of $\mathcal{B}$
(in \Fig{Fig:codice_BMVi} we omitted the corresponding code, see \Fig{execfig}).
\qed\end{proof}

\begin{figure}[thb]
\begin{codice}{123}
\label{bmapcode124} static([], \_, \_, []).\\
\label{bmapcode125} static([[OP,FE1,FE2,Cond]|Others], States, Now, [Flag|Flags]) :- \\
\label{bmapcode126} \tab    get\_precondition\_vars(Now, Cond, States, List),\\
\label{bmapcode127} \tab    length(List, NL),\\
\label{bmapcode128} \tab    sum(List, Result),\\
\label{bmapcode129} \tab    rel\_parsing(FE1, Val1, Now, States),\\
\label{bmapcode130} \tab    rel\_parsing(FE2, Val2, Now, States),\\
\label{bmapcode131} \tab    exp\_constraint(Val1, OP, Val2, C),\\
\label{bmapcode132} \tab    (Result \verb"#=" NL) \verb"#<=>" Flag,\\
\label{bmapcode133} \tab    Flag \verb"#=>" C,\\
\label{bmapcode134} \tab    static(Others, States, Now, Flags).
\end{codice}
\caption{\label{Fig:BMVii}Static causal laws treatment}
\end{figure}

Let us now consider the presence of static causal laws.
In \Fig{Fig:BMVii}, we list the predicate used to add constraints for the static causal laws.
Notice that the concrete implementation of
\Fig{Fig:BMVii} contains a discrepancy with respect to the
abstract one of \Fig{cieffequatris}.
In particular, the concrete implementation does not deal with an intermediate state (named $v$
in the abstract implementation).
The fluents of the target state are computed
by exploiting direct relationships with the starting state of the transition.
This allows us to introduce fewer CLP variables.

In line~(\ref{bmapcode065}) of \Fig{Fig:codice_BMVi} the predicate \codetext{cluster\_rules} collects
all the (static) conditions imposed on the fluents of the cluster of \codetext{FluentName}.
The call to the predicates \codetext{static} (line (\ref{bmapcode066}))
collects the list of Boolean flags \codetext{StatFormula}
which are used to model the constraints~(\ref{cieffequatris22})
and~(\ref{cieffequatris23}) of \Fig{cieffequatris}.
In line (\ref{bmapcode067}), the disjunction of these
flag variables, together with those originating from the dynamic causal laws
(i.e., \codetext{DynFormula}), is computed in \codetext{Formula}, as explained above.

For each condition implied by a static causal law,
the predicate \codetext{static} (lines (\ref{bmapcode124})--(\ref{bmapcode134}))
builds the constraint \codetext{C}
that must be imposed to ensure closure.
The flag variable \codetext{Flag} in line (\ref{bmapcode132})
is introduced to reflect the  satisfaction of the constraint.
If \codetext{Flag} is true then  the constrain \codetext{C}  is
asserted (line (\ref{bmapcode133})).
All such flags are stored in the list \codetext{Flags}
(cf., the variable \codetext{StatFormula}).

We have the following result:

\begin{theorem}\label{completOfConcreteBMVii}
The concrete implementation (partially depicted in \Figs{Fig:codice_BMVi} and \ref{Fig:BMVii}) is
complete w.r.t.\ the system of constraints of \Figs{cieffebis} and \ref{cieffequatris}.
\end{theorem}
\begin{proof}
The result directly follows from the above argument.
Constraint (\ref{cieffequatris15}) of \Fig{cieffequatris}
is im\-plic\-it\-ly ren\-der\-ed by the do\-main as\-sign\-ment for the CLP variables (let us
remember that the
intermediate state $v$
is not explicit in the concrete implementation).
Constraints~(\ref{cieffebis9})--(\ref{cieffebis13}) are dealt with as
done in Theorem~\ref{concretesoundcompleteNostatic}.
The conditions originating from the static causal laws are dealt with through
the predicates  \codetext{cluster\_rules} and \codetext{static}.
\qed\end{proof}

Let us observe  that there is a second difference between the concrete implementation of
\Figs{Fig:codice_BMVi} and \ref{Fig:BMVii} and the abstract one of \Fig{cieffequatris}:
no requirements for the unsatisfiability of
$\mathit{Form}(\mathcal{D})^{\vec{v},a_i}$ are imposed in correspondence of the
state transition from $v_{i-1}$ to $v_{i}$ (for any $i$).
This allows the generation of state transitions where the target state is potentially not
minimally closed. This means that the concrete implementation
may produce solutions (i.e., plans) that the abstract semantics
would forbid because of the non-minimal effects of (clusters of) static causal laws.
On the other hand, we reflect constraints~(\ref{cieffequatris22})
and~(\ref{cieffequatris23}) as described earlier, through the predicates
\codetext{static} (listed in \Fig{Fig:BMVii}) and
\codetext{cluster\_rules} (whose obvious code is omitted).

The final step in the design of the concrete implementation is the
introduction of  suitable restrictions on the labeling phase
of the CLP solver.
Notice that, if at step $i$ in a trajectory, a consequence of a dynamic law
involves a fluent $f^j$, for $j>i$, then such a constraint has to be evaluated
considering as already assessed all the states $v_h$ preceding~$v_i$.
Hence, the labeling has to proceed ``left-to-right'' w.r.t.\ the CLP variables that
model the states $v_1,\ldots,v_i$. In other words, when searching for a solution,
the variables representing the state $v_h$ have to be labeled before
those representing the state $v_{h+1}$, for each $v_h$ in the trajectory.
The implementation of this labeling strategy is depicted
in \Fig{mylabeling}.
Moreover, observe that we impose further restrictions (through the predicate
\codetext{no\_loop} in lines (\ref{bmapcode147})--(\ref{bmapcode155}))
to avoid loops in plans, i.e., to forbid those trajectories where the
same state appears twice.

\begin{figure}[ht]
\begin{codice}{134}
\label{bmapcode135} lm\_labeling(Actionsocc, States) :- \\
\label{bmapcode136} \tab     lm\_labeling(Actionsocc, States, 1).\\
\label{bmapcode137} lm\_labeling([], \_, \_) :- !.\\
\label{bmapcode138} lm\_labeling([CurrAct|Actions], States, I) :- \\
\label{bmapcode139} \tab    lm\_labeling\_aux(CurrAct),\\
\label{bmapcode140} \tab    no\_loop(States, I),\\
\label{bmapcode141} \tab    I1 is I+1,\\
\label{bmapcode142} \tab    lm\_labeling(Actions, States, I1).\\
\label{bmapcode143} lm\_labeling\_aux([]).\\
\label{bmapcode144} lm\_labeling\_aux([action(\_,A)|R]) :- \\
\label{bmapcode145} \tab    indomain(A),\\
\label{bmapcode146} \tab    lm\_labeling\_aux(R).
\medskip\\
\label{bmapcode147} no\_loop(States, A) :- \\
\label{bmapcode148} \tab    state\_select(A, States, StateA),\\
\label{bmapcode149} \tab    no\_loop(A, States, StateA).\\
\label{bmapcode150} no\_loop(0, \_, \_) :- !.\\
\label{bmapcode151} no\_loop(B, States, StateA) :- \\
\label{bmapcode152} \tab    B1 is B-1,\\
\label{bmapcode153} \tab    state\_select(B1, States, StateB),\\
\label{bmapcode154} \tab    StateA \verb"\==" StateB, \\
\label{bmapcode155} \tab    no\_loop(B1, States, StateA).
\end{codice}
\caption{\label{mylabeling}Implementation of a leftmost labeling strategy.}
\end{figure}

To complete the implementation of \BMV\ we need to take care of
the cost-based constraints, whose behavior relies on
the optimization features offered by SICStus' labeling predicate:
the labeling phase is guided by an objective function to be optimized.

Constraints on costs, as well as absolute temporal constraints,
are handled by asserting suitable CLP constraints on the variables
that model fluent values.
This is realized through the predicates listed in \Fig{codiceGlobalConstraints}.
In particular, \codetext{set\_cost\_constraints} deals with constraints on
actions/plans and states.
For instance, \codetext{set\_statecosts} (line~\ref{statecost})
retrieves all the assertions
of the form \codetext{cost\_constraint(state(I) \,OP\, Num)}
and imposes the corresponding constraints.
A similar predicate \codetext{set\_goal} (not reported in the figure)
accomplishes the
same  for the final state only.
The predicate \codetext{set\_plancost} acts similarly, using the predicate
\codetext{make\_one\_action\_occurrences} (lines~(\ref{bmapcode184})--(\ref{bmapcode185}))
where the cost for each single action is considered.

All the absolute temporal constraints defined in the action description are handled by
the predicate \codetext{set\_time\_constraint} (cf., lines~(\ref{bmapcode186})--(\ref{bmapcode194})).
Also in this case, direct references to CLP variables implement
the references to fluent expressions in any absolute point in time.

As mentioned, all these constraints can be seen as filters used to
validate each trajectory found by the labeling phase.
The planner  described in \Figs{Fig:codice_BMVi}--\ref{mylabeling} is
completed by adding the code in \Fig{codiceGlobalConstraints}.
Completeness of the implementation of the full \BMV\ immediately follows
from the above discussion.

\begin{figure}[ht]
\begin{minipage}[t]{1.0\textwidth}
\begin{codice}{155}
set\_cost\_constraints(States, PlanCost, GOALCOST) :-\\
\tab     set\_goalcost(States, GOALCOST),\\
\tab     set\_plancost(PlanCost),\\
\tab     set\_statecosts(States).
\medskip\\
set\_plancost(PC) :-\\
\tab    findall([OP,Num],(cost\_constraint(C), C\,=..\,[OP,plan,Num]), PlanCosts),\\
\tab    set\_plancost\_aux(PlanCosts,PC).\\
set\_plancost\_aux([],\_).\\
set\_plancost\_aux([[OP,Num]|PlanCosts],PC) :-\\
\tab     add\_constraint(PC,OP,Num),\\
\tab     set\_plancost\_aux(PlanCosts,PC).
\medskip\\
\label{statecost}set\_statecosts(States) :-\\
\tab    findall([I,OP,N],(cost\_constraint(C), C\,=..\,[OP,state(I),N]), Costs),\\
\tab    set\_statecost\_aux(Costs,States).
\medskip\\
set\_statecost\_aux([],\_).\\
set\_statecost\_aux([[I,OP,Num]|StateCosts],States) :-\\
\tab    (state\_cost(FE),!; FE = 1), \\
\tab    rel\_parsing(FE,Val,I,States), \\
\tab    add\_constraint(Val,OP,Num),\\
\tab    set\_statecost\_aux(StateCosts,States).
\medskip\\
\label{bmapcode168} make\_action\_occs(N, ActionsOcc, PlanCost, Na) :- \\
\label{bmapcode169} \tab    setof(A, action(A), La),\\
\label{bmapcode170} \tab    length(La, Na),\\
\label{bmapcode171} \tab    make\_action\_occurrences(N, La, ActionsOcc, PlanCost).\\
\label{bmapcode172} make\_action\_occurrences(1, \_, [], 0).\\
\label{bmapcode173} make\_action\_occurrences(N, List, [Act|ActionsOcc], Cost) :- \\
\label{bmapcode174} \tab    N1 is N-1,\\
\label{bmapcode175} \tab    make\_action\_occurrences(N1, List, ActionsOcc, Cost1),\\
\label{bmapcode176} \tab    make\_one\_action\_occurrences(List, Act, Cost2),\\
\label{bmapcode177} \tab    get\_action\_list(Act, AList),\\
\label{bmapcode178} \tab    fd\_only\_one(AList),\\
\label{bmapcode179} \tab    Cost \verb"#=" Cost1+Cost2.\\
\label{bmapcode180} make\_one\_action\_occurrences([], [], 0).\\
\label{bmapcode181} make\_one\_action\_occurrences([A|Actions], [action(A,OccA)|OccActs], Cost) :- \\
\label{bmapcode182} \tab    make\_one\_action\_occurrences(Actions, OccActs, Cost1),\\
\label{bmapcode183} \tab    fd\_domain\_bool(OccA),\\
\label{bmapcode184} \tab    (action\_cost(A,CA),!; CA = 1), \%\%\%Default action cost = 1\\
\label{bmapcode185} \tab     Cost \verb"#=" OccA*CA+Cost1.
\medskip\\
\label{bmapcode186} set\_time\_constraints(States) :- \\
\label{bmapcode187} \tab    findall([FE1,OP,FE2], (time\_constraint(C),C\,=..\,[OP,FE1,FE2]), TimeCs),\\
\label{bmapcode188} \tab    set\_time\_constraints(TimeCs, States).
\medskip\\
\label{bmapcode189} set\_time\_constraints([], \_).\\
\label{bmapcode190} set\_time\_constraints([[FE1,OP,FE2]|Rest], States) :- \\
\label{bmapcode191} \tab     rel\_parsing(FE1, Val1, \_, States),\\
\label{bmapcode192} \tab     rel\_parsing(FE2, Val2, \_, States),\\
\label{bmapcode193} \tab     add\_constraint(Val1, OP, Val2),\\
\label{bmapcode194} \tab     set\_time\_constraints(Rest, States).
\medskip\\
\label{bmapcode195} add\_constraint(L, OP, R) :- \\
\label{bmapcode196} \tab   exp\_constraint(L, OP, R, 1).
\end{codice}
\end{minipage}
\caption{\label{codiceGlobalConstraints}Handling of global constraints and costs.}
\end{figure}

\section{Experimental Analysis}\label{sec:experimental}

We implemented CLP-based prototypes of $\mathcal{B}$
and \BMV. These have been realized in SICStus Prolog 4, and they have
been developed on an AMD Opteron 2.2GHz Linux machine.
Extensive testing has been performed to validate our CLP-based approach.
Here we concentrate on a few representative examples.
The source code of the implementations and the examples
can be found at \sito.
No particular built-in predicates of SICStus have been used and therefore porting
to other CLP-based Prolog systems is straightforward. A porting to B-Prolog has
been realized and used to participate in the
2009 ASP Competition.\footnote{See the web site \url{http://www.cs.kuleuven.be/~dtai/events/ASP-competition/Teams/Bpsolver-CLPFD.shtml}}

In the rest of this section, we analyze the performance of the
implementation on a diverse set of
benchmarks. For each benchmark, we compare a natural encoding using the
traditional $\mathcal{B}$ language with an encoding using \BMV.

The problems encoded in $\mathcal{B}$ have been solved using both
the CLP(FD) implementation
and implementations obtained by mapping the problem to ASP
and using different ASP solvers (Smodels, Clasp, and Cmodels
with different SAT-solvers).

\bigskip

In order to solve a $\mathcal{B}$-planning problem $\langle \mathcal{D},
\mathcal{O}\rangle$ using an  ASP solver, we have developed a Prolog
translator that takes as input  $\langle \mathcal{D},
\mathcal{O}\rangle$ and the plan length \codetext{n}, and it
generates an ASP program,  whose stable models are in one-to-one correspondence
with the plans of length \codetext{n} for $\langle \mathcal{D},
\mathcal{O}\rangle$.
This encoding follows the general ideas
outlined in  \cite{Lif99}.
In particular, the definitions of \codetext{fluent},
\codetext{action}, and
\codetext{initially} are already in ASP syntax.
The length of the plan \codetext{n} is used to define the predicate
\codetext{time(0..n)}. The ASP-based planner makes use of a
choice rule to ensure that exactly one action is applied
at each time step:\\
\centerline{\codetext{1\{occ(Act,Ti):action(Act)\}1 :- time(Ti), Ti < n}.}

\noindent
The predicate \codetext{hold(Fluent,Time)} defines the truth value of a fluent
\codetext{Fluent} at a given time step (\codetext{Time}).
The truth value of the fluents  at time $0$ are given as facts describing
the initial state; we require the initial state to be complete.
The executability rules, the dynamic causal laws and the
 static causal laws are instantiated for each admissible time step.
Finally, the goal conditions are added to define
the predicate \codetext{goal}; the requirement that the goal has
to be satisfied at the end of the plan is imposed using an
ASP constraint of the form\\
\centerline{\codetext{:- not goal}.}

\bigskip

As far as the CLP-based implementations are concerned, we
use a leftmost variable selection strategy. Moreover, we included
a loop control feature to avoid the repetition of the same state
in a trajectory (cf., the predicate \codetext{no\_loop} in \Fig{mylabeling}).

Tables~\ref{tabellabottiNOLOOP}--\ref{tabellaWGC}, discussed in detail
in the next subsections, illustrate an
excerpt of the experimental results.
In order to simplify the comparison among the solvers,
in each table we introduce  an extra column, denoted by ``Best ASP,''
which indicates the performance of an hypothetical ASP-solver that
always acts as the best between all the  ASP-solvers considered.

The specific meaning of the various columns is as follows:
\begin{itemize}
\item \emph{Instance}: the name of the specific instance of the problem
\item \emph{Length}: the plan length used in searching for a solution
\item \emph{Answer}: indication of whether an answer exists or not for the
        given plan length
\item \emph{lparse}: the time required to ground the ASP encoding of the problem
      (using lparse 1.1.1)
\item \emph{Smodels}: the execution time using the {Smodels} system
      (using Smodels 2.32)
\item \emph{Cmodels}: the execution time using the {Cmodels} system (using Cmodels 3.70
    with different SAT solvers)
\item \emph{Clasp}: the execution time using the {Clasp} system
      (using Clasp 1.0.2)

\item \emph{Best ASP}: a summary of the best execution time across all the different
        ASP solvers
\item \emph{CLP(FD)}: the execution time using the CLP(FD)-based implementation of
        $\cal B$. Execution times have the form $t_1+t_2$, where $t_1$ is the time
        needed for posting constraints and $t_2$ the time for solving the constraints (i.e.,
        finding a plan)

\item \emph{\BMV}: the execution time using the \BMV\ encoding of the problem.
      The first column is related to computations where no constraints for the plan
      cost are imposed.
      Instead, the computations of the second column have a constraint
      that limits the plan cost to the number in parenthesis.
      The format is $t_1+t_2$ as explained in the previous point.
\end{itemize}
In the remaining subsections we briefly describe the benchmarks
tested and the obtained results.
The actual encoding in  $\mathcal{B}$ and \BMV\ have been
placed in the Appendix for the sake of readability.
A summary and a discussion of all the experiments is presented in
\Sect{palloSection}.

\begin{sidewaystable}
\begin{tabular}{|c|c|c|r|r|r|r|r|r|r||r|r|rr|}
Barrels'  & & &
\multicolumn{8}{c|}{$\mathcal{B}$} &
\multicolumn{3}{c|}{\BMV} \\
capacities &
\boxstandup{Length}  &
\boxstandup{Answer}  &
\multicolumn{1}{c|}{\em lparse}   &
\multicolumn{1}{c|}{Smodels}   &
\multicolumn{3}{c|}{Cmodels} &
\multicolumn{1}{c|}{Clasp} &
\multicolumn{1}{c||}{Best}
&  \multicolumn{1}{c|}{CLP(FD)}
&  \multicolumn{1}{c|}{unconstrained}
&  \multicolumn{2}{c|}{constrained plan cost}
\\
& & & & & zchaff & relsat & minisat &  & \multicolumn{1}{c||}{ASP} &
&  \multicolumn{1}{c|}{plan cost}
&  \multicolumn{2}{c|}{(in parentheses)}
\\
\cline{1-14}
8-5-3   &  6 & N &  8.74&  0.10 &    0.34 &    0.63 &    0.30 &     0.27 & 0.10 & 0.14+0.29  &  0.03+0.03 &  (70) & 0.02+0.03
\\
8-5-3   &  7 & Y &  8.92&  0.20 &    1.87 &    2.39 &    0.55 &     0.23 & 0.20 & 0.22+0.28  &  0.03+0.02 &  (70) & 0.02+0.02
\\
8-5-3   &  8 & Y &  8.87&  0.20 &    7.34 &    3.63 &    0.62 &     0.53 & 0.20 & 0.26+1.04  &  0.05+0.07 &  (70) & 0.01+0.06
\\
8-5-3   &  9 & Y &  9.03&  0.17 &   17.60 &    5.02 &    0.60 &     2.34 & 0.17 & 0.24+1.03  &  0.02+0.05 &  (70) & 0.02+0.06
\\
12-7-5  & 10 & N & 34.47&  1.98 &  153.36 &   14.56 &   41.34 &    29.13 & 1.98 & 0.58+4.85  &  0.04+0.13 & (120) & 0.04+0.13
\\
12-7-5  & 11 & Y & 34.54&  2.28 &   98.72 &   15.78 &   11.71 &    52.15 & 2.28 & 0.64+2.61  &  0.02+0.07 & (120) & 0.03+0.07
\\
12-7-5  & 12 & Y & 35.42&  1.60 &  125.84 &   20.45 &   83.06 &    35.81 & 1.60 & 0.73+8.11  &  0.07+0.18 & (120) & 0.05+0.19
\\
12-7-5  & 13 & Y & 35.69&  0.68 &  342.40 &   42.36 &   97.99 &   111.36 & 0.68 & 0.79+6.23  &  0.07+0.14 & (120) & 0.07+0.14
\\
16-9-7  & 14 & N &115.47& 11.15 & 1508.43 &  613.42 &   75.67 &  1838.39 &11.15 & 1.30+27.16 &  0.03+0.31 & (200) & 0.07+0.31
\\
16-9-7  & 15 & Y &114.03& 12.30 &  586.43 &   58.45 &   65.19 &  1133.21 &12.30 & 1.53+13.35 &  0.06+0.13 & (200) & 0.07+0.14
\\
16-9-7  & 16 & Y &115.60&  6.06 &  793.00 &  151.56 &  157.38 &   744.60 & 6.06 & 1.62+37.69 &  0.07+0.37 & (200) & 0.07+0.36
\\
16-9-7  & 17 & Y &114.60&  1.75 & 2963.37 &  128.91 &  145.11 & 14106.98 & 1.75 & 1.67+26.98 &  0.07+0.27 & (200) & 0.07+0.27
\\
20-11-9 & 18 & N &185.38& 43.71 & 2949.10 & 2312.09 &  493.98 & --       &43.71 & 2.76+102.14&  0.09+0.58 & (300) & 0.08+0.57
\\
20-11-9 & 19 & Y &186.76& 40.08 & 3053.53 & 1187.10 & 1152.27 & 11292.40 &40.08 & 2.94+45.43 &  0.09+0.24 & (300) & 0.10+0.24
\\
20-11-9 & 20 & Y &186.31& 21.67 & 1866.28 & 2265.05 & 1378.93 & 12286.98 &21.67 & 3.05+120.90&  0.09+0.68 & (300) & 0.09+0.65
\\
20-11-9 & 21 & Y &189.28&  4.39 & 5482.78 &  586.18 & 1746.81 & --       & 4.39 & 3.17+80.54 &  0.10+0.46 & (300) & 0.10+0.43
\\
\cline{1-14}
\end{tabular}
\caption{\label{tabellabottiNOLOOP}Experimental
results with various instances of the three-barrel problem (timeout 24000sec).}
\end{sidewaystable}

\subsection{Three-barrel Problem}\label{sectionThreeBarrelProblem}
We experimented with different encodings of the three-barrel problem.
Our formulation is as described in Example~\ref{exempiobarrelsB}.
\Fig{Bool_Barrels} and \Sect{MV_Barrels}
show the encoding of the problem (for $N=12$) in
$\mathcal B$ and in \BMV, respectively.
Notice that, in order to represent each multi-valued fluent
$f$ of the \BMV\ formulation,
a number of Boolean fluents have to be introduced
in the $\mathcal{B}$ encoding, one for each admissible
value of~$f$.

Table~\ref{tabellabottiNOLOOP} provides the execution times (in seconds)
for different values of $N$ and different plan lengths.
The results show that  the constraint-based encoding
of $\mathcal B$ outperforms the ASP encodings (if we consider both grounding
and execution). In turn, the \BMV\ encoding outperforms all other encodings.
This can be explained by considering that the CLP encoding of this problem
benefits from numerical fluents (in reduced number, w.r.t.\ the $\mathcal B$
formulation) and from arithmetic constraints (efficiently handled by CLP(FD)).

\subsection{2-Dimensional Protein Folding Problem}
The problem we have encoded is a simplification of the protein structure
folding problem. The input is a chain $\alpha_1 \alpha_2 \cdots \alpha_n$ with $\alpha_i \in
\{0,1\}$,
initially placed in a vertical position, as in \Fig{PFfigure1}-left.
We will refer to each $\alpha_i$ as an \emph{amino acid}.
The permissible actions are the counter-clockwise/clockwise \emph{pivot moves}.
Once one point $i$ of the chain is selected, the points $\alpha_1,\alpha_2,\dots,\alpha_i$
will remain fixed, while the points $\alpha_{i+1},\dots,\alpha_n$ will perform a
rigid counter-clockwise/clockwise rotation.
Each conformation must be a \emph{self-avoiding-walk},
i.e., no two amino acids
are  in the same position. Moreover, the chain cannot be broken---i.e.,
two consecutive amino acids are always at points at distance 1
(i.e., in contact).
The goal is to perform a sequence of pivot moves leading to a
configuration where at least $k$ non-consecutive
amino acids of value 1 are in contact.
\Fig{PFfigure1} shows a possible plan to  reach a configuration
with 4 contacts. Table~\ref{PFfigure2} reports some execution times.
\Sect{PF-encoding} reports the \BMV\ action description
encoding this problem.
Since the goal is based on the notion of cost of a given state, for
which reified constraints are used extensively,
a direct encoding in $\mathcal{B}$
does not seem to be feasible.

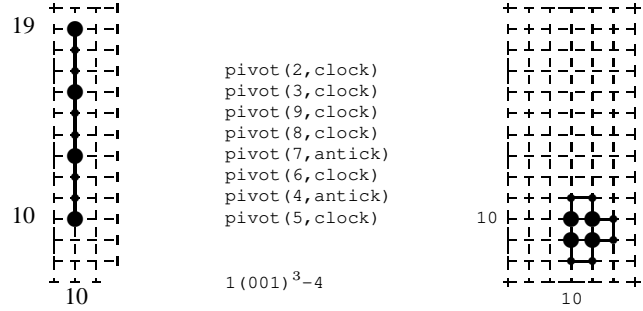
\begin{figure}
\begin{picture}(55,135)(-10,-10)
\setlength{\unitlength}{0.8pt}
\thinlines
\put(10,10){\dashbox{4}(30,0){}}
\put(10,20){\dashbox{4}(30,0){}}
\put(10,30){\dashbox{4}(30,0){}}
\put(10,40){\dashbox{4}(30,0){}}
\put(10,50){\dashbox{4}(30,0){}}
\put(10,60){\dashbox{4}(30,0){}}
\put(10,70){\dashbox{4}(30,0){}}
\put(10,80){\dashbox{4}(30,0){}}
\put(10,90){\dashbox{4}(30,0){}}
\put(10,100){\dashbox{4}(30,0){}}
\put(10,110){\dashbox{4}(30,0){}}
\put(10,120){\dashbox{4}(30,0){}}
\put(10,130){\dashbox{4}(30,0){}}
\put(10,0){\dashbox{4}(0,130){}}
\put(20,0){\dashbox{4}(0,130){}}
\put(30,0){\dashbox{4}(0,130){}}
\put(15,-10){10}
\put( -10,28){10}
\put( -10,118){19}
\thicklines
\put(20,30){\circle*{8}}
\put(20,30){\line(0,1){90}}
\put(20,40){\circle*{4}}
\put(20,50){\circle*{4}}
\put(20,60){\circle*{8}}
\put(20,70){\circle*{4}}
\put(20,80){\circle*{4}}
\put(20,90){\circle*{8}}
\put(20,100){\circle*{4}}
\put(20,110){\circle*{4}}
\put(20,120){\circle*{8}}
\end{picture}
~~~~~~~~~
{\protect\scriptsize\tt
\begin{minipage}[t]{.20\textwidth}
\begin{tabular}[b]{l}
pivot(2,clock) \\
pivot(3,clock) \\
pivot(9,clock) \\
pivot(8,clock) \\
pivot(7,antick) \\
pivot(6,clock) \\
pivot(4,antick) \\
pivot(5,clock)\\
\phantom{a}\\
\phantom{a}\\
\mbox{1(001)$^{3}$-4}\\
\phantom{a}
\end{tabular}
\end{minipage}
~~~
{\begin{picture}(75,135)(-25,-10)
\setlength{\unitlength}{0.8pt}
\thinlines
\put(-10,0){\dashbox{4}(60,0){}}
\put(-10,10){\dashbox{4}(60,0){}}
\put(-10,20){\dashbox{4}(60,0){}}
\put(-10,30){\dashbox{4}(60,0){}}
\put(-10,40){\dashbox{4}(60,0){}}
\put(-10,50){\dashbox{4}(60,0){}}
\put(-10,60){\dashbox{4}(60,0){}}
\put(-10,70){\dashbox{4}(60,0){}}
\put(-10,80){\dashbox{4}(60,0){}}
\put(-10,90){\dashbox{4}(60,0){}}
\put(-10,100){\dashbox{4}(60,0){}}
\put(-10,110){\dashbox{4}(60,0){}}
\put(-10,120){\dashbox{4}(60,0){}}
\put(-10,130){\dashbox{4}(60,0){}}
\put(-10,0){\dashbox{4}(0,130){}}
\put(0,0){\dashbox{4}(0,130){}}
\put(10,0){\dashbox{4}(0,130){}}
\put(20,0){\dashbox{4}(0,130){}}
\put(30,0){\dashbox{4}(0,130){}}
\put(40,0){\dashbox{4}(0,130){}}
\put(50,0){\dashbox{5}(0,130){}}
\put(15,-10){10}
\put( -25,28){10}
\thicklines
\put(20,30){\circle*{8}}
\put(20,30){\line(0,1){10}}
\put(20,40){\circle*{4}}
\put(20,40){\line(1,0){10}}
\put(30,40){\circle*{4}}
\put(30,40){\line(0,-1){10}}
\put(30,30){\circle*{8}}
\put(30,30){\line(1,0){10}}
\put(40,30){\circle*{4}}
\put(40,30){\line(0,-1){10}}
\put(40,20){\circle*{4}}
\put(40,20){\line(-1,0){10}}
\put(30,20){\circle*{8}}
\put(20,10){\line(1,0){10}}
\put(30,10){\circle*{4}}
\put(30,10){\line(0,1){10}}
\put(20,10){\circle*{4}}
\put(20,10){\line(0,1){10}}
\put(20,20){\circle*{8}}
\end{picture}
}
}
\caption{\label{PFfigure1}An instance of the HP-protein folding problem:
initial configuration, a plan, and final configuration with 4
contacts between 1-amino acids.}
\end{figure}

\begin{table}
\begin{tabular}{|c|c|c|r|}
Instance       & Len. & Ans. & $\mathcal{B}_{MV}^{FD}$ \\
\cline{1-4}
1$^{7}$-2      &   3   & Y  & 0.07+0.01
\\
1$^{7}$-2      &   4   & Y  & 0.09+0.01
\\
1$^{13}$-6     &   3   & N  & 0.42+19.91
\\
1$^{13}$-6     &   4   & Y  & 0.57+35.16
\\
1(001)$^{2}$-2 &   3   & N  & 0.06+0.09
\\
1(001)$^{2}$-2 &   4   & Y  & 0.07+0.01
\\
1(001)$^{3}$-4 &   7   & N  & ~0.47+7521.13
\\
1(001)$^{3}$-4 &   8   & Y  & 0.49+50.46
\\
1(001)$^{3}$-4 &   9   & ?  & --
\\
1(001)$^{3}$-4 &  10   & Y  & 0.63+603.37
\\
\cline{1-4}
\end{tabular}
\caption{\label{PFfigure2}The HP-protein folding problem:
some results for different sequences, and plan lengths (timeout 12000sec).}
\end{table}

Let us consider the resolution of the instance depicted in
\Fig{PFfigure1}, i.e., the folding of
the input chain $1001001001$ of $n=10$ amino acids.
Asking for a plan of $8$ (resp.~$10$) moves and for a solution with
cost $\geqslant 4$, our planner finds the 8-moves plan shown
in \Fig{PFfigure1}-center
in 50.46s (a 10-moves plan in found in 603.37s).
By removing the two constraints that keep fixed $\alpha_2$:
\begin{center}
\codetext{
always(x(2) eq 10).}\\
\codetext{
always(y(2) eq 11).}
\end{center}
the solutions are found in 52.72s and 617.68s, respectively.
On the other hand,
by keeping fixed $\alpha_2$ and adding the two  constraints
\begin{center}
\codetext{
holds(x(3) eq 11,1).}\\
\codetext{
holds(y(3) eq 11,1).}
\end{center}
the execution time is reduced to  4.06s and 52.97s.
Adding the additional  constraints
\begin{center}
\codetext{
holds(x(4) eq 11,2).}\\
\codetext{
   holds(y(4) eq 10,2).}
\end{center}
the plans are found in  only 0.37s and 4.62s.
This shows that the use of multi-valued fluents and the ability
to exploit domain-specific knowledge, in the form of symmetry-breaking
constraints,  allows \BMV\ to
effectively converge to a solution.

\subsection{The Community Problem}\label{community}

The \emph{Community} problem is formulated as follows.
There are $M$ individuals,
identified by the numbers $1,2,\ldots,M$.
At each time step, one of them, say $j$,  gives exactly~$j$ dollars to someone else,
provided she/he owns more than $j$ dollars.
Nobody can give away all of her/his money.
The goal consists of reaching a state in which all the
participants have the same amount of money.

Table~\ref{tabellaCommunity} lists some results for four variants of the problem:
the person $i$ initially owns $2*i$ dollars (instances $A_M$),
$i+1$ dollars (instances $B_M$),
$i*i$ dollars (instances $C_M$),
or $i*(1+i)$ dollars (instances $D_M$).

The representations of this problem are reported in \Sects{BCommunity} and~\ref{MVCommunity}.

Notice that the large number of  Boolean fluents that have to be introduced in
the $\mathcal{B}$ description  causes failures due to lack of
memory during the grounding phase
(these instances are marked ``mem'' in Table~\ref{tabellaCommunity}).
For all these experiments,
the bound on memory usage was 4~GB (for the grounder, the ASP-solvers, and
the CLP(FD) engine).
Observe that, in some cases, also the CLP(FD)-based solver
for~$\mathcal{B}$ runs out of memory, while the
failures of the CLP(FD) solver for~\BMV\
have been  caused by expiration of the
time limit.
In summary,  the constraint-based encodings provides better
performance in most of the instances,
especially considering their better scalability w.r.t.\ the size
of the instances.
This originates from the smaller number of  numerical fluents
and from the efficiency of the underlying constraint solver.

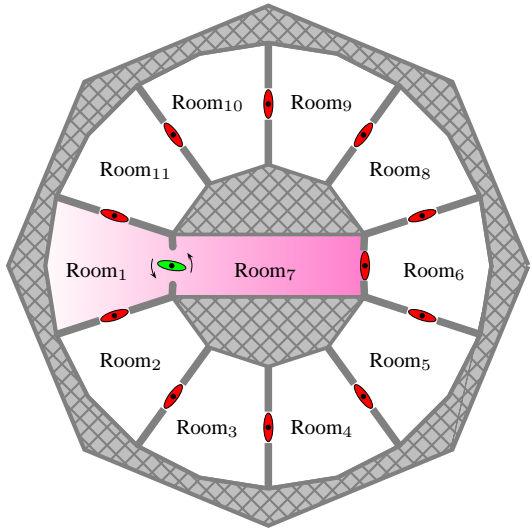
\begin{sidewaystable}
\begin{tabular}{|c|c|c|r|r|r|r|r|r|r||r|r|}
&
&
&
\multicolumn{8}{c|}{$\mathcal{B}$} &
\multicolumn{1}{c|}{\BMV} \\
\boxstandup{Instance} &
\boxstandup{Length}  &
\boxstandup{Answer}  &
\multicolumn{1}{c|}{\em lparse}   &
\multicolumn{1}{c|}{Smodels}   &
\multicolumn{3}{c|}{Cmodels} &
\multicolumn{1}{c|}{Clasp} &
\multicolumn{1}{c||}{Best}
&  \multicolumn{1}{c|}{CLP(FD)}
&  \multicolumn{1}{c|}{CLP(FD)}
\\
&
&
&
& & zchaff & relsat & minisat & &
\multicolumn{1}{c||}{ASP}
 & &  \multicolumn{1}{c|}{~}
\\
\cline{1-12}
A$_{\mathit{4}}$ & 5  & N &   34.34 &   11.12  &    1.78 &    11.68 &     0.67 &     0.45 &  0.45 & 0.71+14.14   & 0.01+3.31
\\
A$_{\mathit{4}}$ & 6  & Y &   34.90 &    1.43  &    0.26 &     7.38 &     0.57 &     0.09 &  0.09 & 0.82+0.10    & 0.03+0.00
\\
A$_{\mathit{4}}$ & 7  & Y &   35.44 &   15.72  &    0.39 &    47.74 &     0.80 &     0.10 &  0.10 & 0.94+0.12    & 0.03+0.01
\\
A$_{\mathit{5}}$ & 5  & N &  201.88 &  100.58  &    5.22 &   125.63 &     2.30 &     1.19 &  1.19 & 2.64+157.48  & 0.02+41.15
\\
A$_{\mathit{5}}$ & 6  & Y &  202.64 &   11.43  &    1.85 &   442.22 &     1.63 &     0.28 &  0.28 & 3.17+0.21    & 0.01+0.04
\\
A$_{\mathit{5}}$ & 7  & Y &  202.12 &   34.02  &    2.81 &   114.74 &     2.31 &     0.27 &  0.27 & 3.71+447.87  & 0.04+142.27
\\
B$_{\mathit{5}}$ & 5  & N &   51.87 &   30.04  &    4.24 &    44.49 &     1.49 &     0.69 &  0.69 & 1.03+77.06   & 0.03+23.13
\\
B$_{\mathit{5}}$ & 6  & Y &   52.04 &    2.07  &    1.32 &    37.96 &     0.99 &     0.14 &  0.14 & 1.31+0.11    & 0.04+0.02
\\
B$_{\mathit{5}}$ & 7  & Y &   52.94 &   13.49  &    0.80 &    41.86 &     1.27 &     0.42 &  0.42 & 1.40+0.17    & 0.05+0.04
\\
B$_{\mathit{7}}$ & 5  & N &   mem   &          &         &          &          &          &       & 7.67+3345.56 & 0.05+1421.54
\\
C$_{\mathit{5}}$ & 5  & N &   mem   &          &         &          &          &          &       & 16.98+85.71  & 0.02+49.83
\\
C$_{\mathit{5}}$ & 6  & N &   mem   &          &         &          &          &          &       & 20.44+1926.97& 0.04+888.30
\\
C$_{\mathit{7}}$ & 5  & N &   mem   &          &         &          &          &          &       &   mem        & 0.05+3186.34
\\
D$_{\mathit{4}}$ & 5  & N &  138.91 &    7.08  &    1.28 &    13.48 &     0.76 &     0.43 &  0.43 &  3.70+21.19  & 0.01+6.83
\\
D$_{\mathit{4}}$ & 6  & N &  139.88 &   90.32  &   11.56 &    87.11 &     3.62 &     3.72 &  3.72 &  4.32+0.50   & 0.02+0.74
\\
D$_{\mathit{4}}$ & 7  & N &  139.82 & 1015.44  &  104.36 &   788.94 &    33.70 &    22.86 & 22.86 &  5.17+5.55   & 0.04+7.64
\\
D$_{\mathit{5}}$ & 5  & N &   mem   &          &         &          &          &          &       & 24.64+24.12  & 0.05+93.88
\\
D$_{\mathit{5}}$ & 6  & N &   mem   &          &         &          &          &          &       & 29.60+1490.48& 0.02+1801.78
\\
\cline{1-12}
\end{tabular}
\caption{\label{tabellaCommunity}Experimental results for instances of the
Community problem.
 ``mem'' denotes out-of-memory failures.
Some results are missing for the ASP solvers, for those instances that
are unable to complete grounding.
}
\end{sidewaystable}

\newcounter{lopper}
\definecolor{molto}{rgb}{1,0.50,0.80}
\definecolor{poco}{rgb}{1,0.98,0.98}
\begin{figure}[th]
\noindent{\centering
\begin{tabular}[b]{cc}
\begin{pspicture}(-0.4,-1.4)(6.5,6.5)
\setlength{\unitlength}{85pt}
\psset{PstPicture=false}
\put(1,1){\PstPolygon[unit=0.43,PolyRotation=27.0,PolyName=Qa,PolyNbSides=10,linewidth=0mm,linecolor=white]}
\put(1,1){\PstPolygon[unit=0.43,PolyRotation=9.0,PolyName=Qb,PolyNbSides=10,linewidth=0mm,linecolor=white]}
\put(1,1){\PstPolygon[unit=0.43,PolyRotation=0.0,PolyName=Qc,PolyNbSides=10,linewidth=0mm,linecolor=white]}
\put(1,1){\PstPolygon[unit=0.45,PolyRotation=18.0,PolyName=A,PolyNbSides=10,linewidth=0mm,linecolor=white]}
\put(1,1){\PstPolygon[unit=0.65,PolyRotation=18.0,PolyName=Ain,PolyNbSides=10,linewidth=0mm,linecolor=white]}
\put(1,1){\PstPolygon[unit=0.72,PolyRotation=18.0,PolyName=Bmid,PolyNbSides=10,linewidth=0mm,linecolor=white]}
\put(1,1){\PstPolygon[unit=0.79,PolyRotation=18.0,PolyName=Bin,PolyNbSides=10,linewidth=0mm,linecolor=white]}
\put(1,1){\PstPolygon[unit=1.0,PolyRotation=18.0,PolyName=B,PolyNbSides=20,linewidth=1mm,linecolor=gray]}
\pspolygon[unit=1mm,linecolor=white,linewidth=0mm,fillstyle=gradient,gradbegin=molto,gradend=poco,gradangle=270,gradmidpoint=1.0](A1)(A5)(B9)(B10)(B11)(A6)(A10)(A1)
\put(1,1){\PstPolygon[unit=1.0,PolyRotation=18.0,PolyName=Bextra,PolyNbSides =20,linewidth=1mm,linecolor=gray]}
\setcounter{lopper}{108}
\multido{\i=1+1}{10}{
    \pscustom[fillstyle=solid,fillcolor=red,linewidth=0.1mm,linecolor=black]{\rotate{\thelopper}
        \psellipse[fillstyle=solid,fillcolor=red,linewidth=0.1mm,linecolor=black](Bmid\the\multidocount)(0.06,0.19)}
    \addtocounter{lopper}{36}}
\multido{\i=1+1}{10}{\ncline[linewidth=1mm,linecolor=gray]{A\the\multidocount}{Ain\i}}
\multido{\i=1+2}{10}{\ncline[linewidth=1mm,linecolor=gray]{Bin\the\multidocount}{B\i}}
\setcounter{lopper}{108}
\multido{\i=1+1}{10}{
    \pscustom[fillstyle=solid,fillcolor=black,linewidth=0.1mm,linecolor=black]{
        \pscircle[fillstyle=solid,fillcolor=black](Bmid\the\multidocount){0.01}}
    \addtocounter{lopper}{36}}
\ncline[linewidth=1mm,linecolor=gray]{A1}{Qb1}
\ncline[linewidth=1mm,linecolor=gray]{A10}{Qa10}
\psellipse[fillstyle=solid,fillcolor=red,linewidth=0.1mm,linecolor=black](Qc1)(0.06,0.19)
\pscustom[fillstyle=solid,fillcolor=black,linewidth=0.1mm,linecolor=black]{
    \pscircle[fillstyle=solid,fillcolor=black](Qc1){0.01}}
\pscustom[fillstyle=solid,fillcolor=green,linewidth=0.1mm,linecolor=black]{\rotate{80}
    \psellipse[fillstyle=solid,fillcolor=green,linewidth=0.1mm,linecolor=black](Qc6)(0.06,0.19)}
\pscustom[linewidth=0.1mm,linecolor=black]{\psarc{->}(Qc6){0.09}{-20}{40}}
\pscustom[linewidth=0.1mm,linecolor=black]{\psarc{->}(Qc6){0.09}{160}{220}}
\pscustom[fillstyle=solid,fillcolor=black,linewidth=0.1mm,linecolor=black]{
    \pscircle[fillstyle=solid,fillcolor=black](Qc6){0.01}}
\psline[linewidth=1mm,linecolor=gray]{-cc}(A5)(Qa5)
\ncline[linewidth=1mm,linecolor=gray]{-cc}{A6}{Qb6}
\pspolygon[unit=1mm,linecolor=gray,linewidth=0.5mm,fillstyle=crosshatch*,hatchcolor=gray,fillcolor=lightgray](A1)(A2)(A3)(A4)(A5)(A1)
\pspolygon[unit=1mm,linecolor=gray,linewidth=0.5mm,fillstyle=crosshatch*,hatchcolor=gray,fillcolor=lightgray](A6)(A7)(A8)(A9)(A10)(A6)
\put(1.45,0.55){\mbox{{\footnotesize Room$_{5}$}}}
\put(1.60,0.95){\mbox{{\footnotesize Room$_{6}$}}}
\put(1.45,1.40){\mbox{{\footnotesize Room$_{8}$}}}
\put(0.85,0.95){\mbox{{\footnotesize Room$_{7}$}}}
\put(0.25,0.55){\mbox{{\footnotesize Room$_{2}$}}}
\put(0.10,0.95){\mbox{{\footnotesize Room$_{1}$}}}
\put(0.25,1.40){\mbox{{\footnotesize Room$_{11}$}}}
\put(1.10,1.70){\mbox{{\footnotesize Room$_{9}$}}}
\put(1.10,0.25){\mbox{{\footnotesize Room$_{4}$}}}
\put(0.57,1.70){\mbox{{\footnotesize Room$_{10}$}}}
\put(0.59,0.25){\mbox{{\footnotesize Room$_{3}$}}}
\put(1,1){\PstPolygon[unit=1.15,PolyName=C,PolyNbSides=8,linewidth=1mm,linecolor=gray]}
\pspolygon[unit=1mm,linecolor=gray,linewidth=0mm,fillstyle=crosshatch*,hatchcolor=gray,fillcolor=lightgray](C8)(B20)(B1)(B2)(B3)(B4)(B5)(B6)(B7)(B8)(B9)(B10)(B11)(B12)(B13)(B14)(B15)(B16)(B17)(B18)(B19)(B20)(C8)(C7)(C6)(C5)(C4)(C3)(C2)(C1)(C8)
\end{pspicture}
&
$\begin{array}[b]{|c|c|c|r|}
\boxstandupbis{Instance}  & \boxstandupbis{Length} & \boxstandupbis{Answer} & \BMV \\
\cline{1-4}
A_{\mathit{1}} &  6 & N &         0.07+13.48
\\
A_{\mathit{1}} &  7 & Y &         0.10+5.35
\\
B_{\mathit{1}} & 10 & N &         0.17+3846.20
\\
B_{\mathit{1}} & 11 & Y &         0.14+1802.76
\\
B_{\mathit{1}} & 12 & Y &         0.15+933.35
\\
B_{\mathit{1}} & 13 & Y &         0.16+302.34
\\
B_{\mathit{1}} & 14 & Y &         0.14+4.60
\\
B_{\mathit{2}} & 10 & N &         0.13+11134.82
\\
B_{\mathit{2}} & 11 & Y &         0.15+4191.20
\\
B_{\mathit{2}} & 12 & Y &         0.16+2156.52
\\
B_{\mathit{2}} & 13 & Y &         0.18+710.53
\\
B_{\mathit{2}} & 14 & Y &         0.17+6.36
\\
B_{\mathit{3}} & 10 & N &         0.12+18763.27
\\
B_{\mathit{3}} & 11 & Y &         0.16+6124.91
\\
B_{\mathit{3}} & 12 & Y &         0.15+3148.43
\\
B_{\mathit{3}} & 13 & Y &         0.20+1145.04
\\
B_{\mathit{3}} & 14 & Y &         0.18+9.97
\\
B_{\mathit{4}} & 10 & N &         0.11+17109.05
\\
B_{\mathit{4}} & 11 & Y &         0.15+6173.49
\\
B_{\mathit{4}} & 12 & Y &         0.14+3180.53
\\
B_{\mathit{4}} & 13 & Y &         0.16+1159.27
\\
B_{\mathit{4}} & 14 & Y &         0.18+10.34
\\
\cline{1-4}
\end{array}$
\end{tabular}
}
\caption{\label{GasDiffusionExp}On the left: a simple schema of the 11 rooms for the Gas-diffusion problem.
The locked gates are in red color. The gas (in pink)
is flowing through the open gate (in green), from Room$_{7}$ to Room$_{1}$.
On the right: some results for different instances (i.e., different goals and initial allocations
of amounts of gas---see \Sect{derattization}).}
\end{figure}

\subsection{The Gas-diffusion Problem}\label{derattization}

The Gas-diffusion problem can be formulated as follows.
A building contains a number of rooms. Each room is connected to
(some) other rooms via gates. Initially, all gates are closed and some of the
rooms contain a quantity of gas---while the other rooms are  empty.
Each gate can be opened or closed---\codetext{open(x,y)} and \codetext{close(x,y)}
are the only possible actions, provided that there is
a gate between room~\codetext{x} and room~\codetext{y}.
When a gate between two rooms is open,
the gas contained in these rooms flows through the gate.
The gas diffusion continues until the pressure reaches an equilibrium.
The only condition to be always satisfied is that a gate in a room can be
opened only if all the other gates are closed.
The goal is to move a desired quantity of gas to one specified room.

We experimented with instances of the problem where the
building has a specific topology:
there are eleven rooms, all having the same physical volume.
Each room is connected to the other rooms via gates as depicted
in \Fig{GasDiffusionExp}.
Since all rooms have the same volume, when equilibrium is reached
between two rooms sharing an open gate, they will both
contain the same amount of gas.

A \BMV\ specification of this planning problem is given in
\Sect{GasDiffusion}.
We experimented with different instances of the Gas-diffusion problem
obtained by considering different goal states and by requiring that
some of the rooms have to be kept empty.
Moreover, we seek  plans of different length.
\Fig{GasDiffusionExp} (on the right) summarizes the  results obtained. In particular,
all instances share the same initial state: rooms 10 and 3 contain
128 moles of gas. All the other rooms are empty.
Moreover,
\begin{itemize}
\item
in the instance A$_1$ the goal state is: room 1 contains at least 32 moles of gas;
\item
in all the instances B$_i$ the goal is:
room 1 contains at least 50 moles of gas.
The B$_i$ instances differ in the constraints imposed on the desired plan:
\begin{itemize}
\item
in the instance B$_1$,
rooms 7, 9, and 4 must remain  empty.
This condition can be imposed by including in the action description the constraints
\begin{center}
\codetext{always(contains(7)~eq~0).}
\\\codetext{always(contains(9)~eq~0).}
\\\codetext{always(contains(4)~eq~0).}
\end{center}
\item
in the instance B$_2$,
rooms 7, 8, and 5 must be kept empty.
\item
in the instance B$_3$,
only room 6 must be kept empty.
\item
in the instance B$_4$, no constraint is imposed.
\end{itemize}
\end{itemize}
Observe that it is quite natural to design a \BMV\ encoding of this problem,
by exploiting the multi-valued fluents.
On the other hand, adopting the naive approach used for the three-barrel problem
would force the introduction of (at least) 128 distinct boolean
fluents for each multi-valued fluent.
Such a large number of boolean fluents generates
a large state space, making the task of any
solver for $\mathcal{B}$ considerably harder.

\subsection{Other Puzzles}\label{other}
We report results from two other planning problems.
The first---3x3-puzzle---is  an encoding of the 8-tile puzzle problem, where the goal
is to find a sequence of moves to
re-order the 8 tiles, starting from a random initial position. The performance
results for this puzzle are reported in Table~\ref{tabellaPuzzle}.
The second problem is the well-known \emph{Wolf-goat-cabbage} problem. The performance results
are reported in Table~\ref{tabellaWGC}.

Notice that these planning problems are predominantly Boolean.
The constraint-based encodings perform well in solving the instances
of the Wolf-goat-cabbage problem.
In contrast, for the 8-tile puzzle problem, the use
of numerical fluents allows us to achieve a compact
encoding, but it does not necessarily lead to a better
performance w.r.t. ASP.

\begin{sidewaystable}
\begin{tabular}{|c|c|c|r|r|r|r|r|r|r||r|r|}
& & & \multicolumn{8}{c|}{$\mathcal{B}$} &
\multicolumn{1}{c|}{\BMV} \\
\boxstandup{Instance} &
\boxstandup{Length}  &
\boxstandup{Answer}  &
\multicolumn{1}{c|}{\em lparse}   &
\multicolumn{1}{c|}{Smodels}   &
\multicolumn{3}{c|}{Cmodels} &
\multicolumn{1}{c|}{Clasp} &
\multicolumn{1}{c||}{Best}
&  \multicolumn{1}{c|}{CLP(FD)}
&  \multicolumn{1}{c|}{CLP(FD)}
\\
& &        & & & zchaff & relsat & minisat & &
\multicolumn{1}{c||}{ASP}
 & &  \multicolumn{1}{c|}{~}
\\
\cline{1-12}
I$_1$  &  9 & N &   41.49 &     0.94 &   2.06  &     3.36 &     1.54 &     0.52 &  0.52 & 0.64+4.42   & 0.25+2.64
\\
I$_1$  & 10 & Y &   41.80 &     2.02 &   2.52  &     7.36 &     2.06 &     0.70 &  0.70 & 0.73+5.43   & 0.29+3.64
\\
I$_2$  & 14 & N &   42.68 &    27.10 &  34.46  &    90.07 &     7.15 &     7.42 &  7.15 & 1.03+57.54  & 0.40+38.67
\\
I$_2$  & 15 & Y &   43.14 &    50.73 &  49.50  &   131.38 &     8.90 &     1.98 &  1.98 & 1.06+7.08   & 0.43+4.60
\\
I$_3$  & 19 & N &   43.76 &   739.39 &1255.46  &   911.82 &    91.75 &   268.69 & 91.75 & 1.39+967.26 & 0.54+673.66
\\
I$_3$  & 20 & Y &   44.52 &   368.28 &1090.66  &  1445.78 &    58.89 &   268.59 & 58.89 & 1.46+597.92 & 0.52+435.96
\\
I$_4$  & 24 & N &   51.59 & 10247.47 &   --    &  5613.98 & 7862.10  &  4185.42 &4185.42&1.70+13887.17& 0.71+10109.58
\\
I$_4$  & 25 & Y &   55.54 &  1430.43 & 954.68  &  1023.22 &  437.11  &   875.16 &437.11 & 1.84+79.20  & 0.73+57.00
\\
I$_5$  & 24 & N &   49.64 &  6936.39 &   --    &  6041.87 & 1239.72  &  4901.13 &1239.72&1.69+11092.48& 0.73+9155.79
\\
I$_5$  & 25 & N &   51.07 & 14079.78 &3747.96  &  8583.44 &11745.93  &  8557.94 &3747.96&1.84+18301.15& 0.73+14195.54
\\
\cline{1-12}
\end{tabular}
\caption{\label{tabellaPuzzle}Experimental results for instances of the 8-tile puzzle problem (timeout 36000 sec).}
\end{sidewaystable}

\begin{table}
\begin{center}
\begin{tabular}{|c|c|r|r|r|r|r|r|r||r|r|}
& & \multicolumn{8}{c|}{$\mathcal{B}$} &
\multicolumn{1}{c|}{\BMV} \\
\boxstandup{Length}  &
\boxstandup{Answer}  &
\multicolumn{1}{c|}{\footnotesize\em\!\!lparse}   &
\multicolumn{1}{c|}{\footnotesize\!\!\!\!Smodels}   &
\multicolumn{3}{c|}{\footnotesize Cmodels} &
\multicolumn{1}{c|}{\footnotesize Clasp} &
\multicolumn{1}{c||}{\footnotesize Best}
&  \multicolumn{1}{c|}{\footnotesize\!CLP(FD)}
&  \multicolumn{1}{c|}{\footnotesize\!CLP(FD)}
\\
&        & & &{\scriptsize\!\!\!zchaff}&{\scriptsize\!\!\!relsat}&{\scriptsize\!\!\!minisat} & &
\multicolumn{1}{c||}{\footnotesize ASP}
 & &  \multicolumn{1}{c|}{~}
\\
\cline{1-11}
21 & N &    0.10 &     0.19 &   1.38  &     1.89 &     0.67 &  0.19 & 0.19 & 0.10+0.20  &  0.09+0.15
\\
22 & N &    0.10 &     0.25 &   1.46  &     3.32 &     0.77 &  0.56 & 0.25 & 0.09+0.21  &  0.11+0.17
\\
23 & Y &    0.10 &     0.26 &   2.30  &     4.34 &     0.58 &  0.13 & 0.13 & 0.12+0.17  &  0.07+0.15
\\
24 & N &    0.11 &     0.43 &   3.10  &     4.75 &     0.67 &  1.09 & 0.43 & 0.07+0.32  &  0.06+0.25
\\
25 & Y &    0.12 &     0.27 &   1.15  &     4.92 &     0.74 &  0.42 & 0.27 & 0.12+0.06  &  0.08+0.08
\\
26 & N &    0.12 &     0.68 &   7.23  &    11.52 &     1.18 &  0.69 & 0.68 & 0.10+0.49  &  0.10+0.40
\\
27 & Y &    0.13 &     0.43 &   1.93  &     6.68 &     0.93 &  0.84 & 0.43 & 0.10+0.03  &  0.06+0.03
\\
28 & N &    0.14 &     1.24 &   9.44  &    18.72 &     1.59 &  2.15 & 1.24 & 0.10+0.80  &  0.08+0.69
\\
29 & Y &    0.14 &     0.41 &   1.75  &    15.55 &     1.10 &  0.60 & 0.41 & 0.11+0.01  &  0.07+0.03
\\
30 & N &    0.15 &     2.97 &  16.17  &    43.53 &     2.31 &  1.78 & 1.78 & 0.11+1.08  &  0.08+1.05
\\
31 & Y &    0.15 &     0.49 &   8.40  &     7.10 &     0.89 &  4.60 & 0.49 & 0.12+0.01  &  0.11+0.04
\\
32 & N &    0.16 &     2.78 &  23.76  &    38.58 &     2.20 &  5.37 & 2.20 & 0.13+1.35  &  0.09+1.32
\\
33 & Y &    0.16 &     1.06 &  31.92  &    26.67 &     1.23 &  0.57 & 0.57 & 0.10+0.07  &  0.14+0.06
\\
34 & N &    0.17 &     3.61 &  38.62  &    51.22 &     3.11 &  5.86 & 3.11 & 0.13+1.75  &  0.10+1.60
\\
35 & Y &    0.18 &     1.39 &  31.10  &    30.25 &     3.20 &  4.21 & 1.39 & 0.15+0.54  &  0.08+0.32
\\
36 & N &    0.18 &     4.55 &  43.97  &    57.21 &     4.24 & 12.68 & 4.24 & 0.13+1.87  &  0.11+1.79
\\
\cline{1-11}
\end{tabular}
\end{center}
\caption{\label{tabellaWGC}Experimental results for instances of the Wolf-goat-cabbage problem.}
\end{table}

\subsection{A Summary of the Experiments}\label{palloSection}

Table~\ref{pallogramma} pictorially summarizes some of the results relating the performance of
the different approaches. For each problem instance, we compare
the execution times obtained by the best ASP-solver and the CLP(FD) solvers for $\mathcal{B}$
and \BMV\ action description languages.
We considered only those instances for which at least one of the
solvers gave an answer.
A score of $1$ ($0$, $-1$) is assigned to
the fastest (second fastest, slowest) solver.
The scores of all instances of a problem have been summed together,
and this provides the radius of the circles in the figure. Instances have been
separated between \emph{``Yes''} instances (they admit a solution) and
\emph{``No''}
instances (they have no solutions).

\newgray{g00}{0}
\begin{table}[ht]
\begin{center}
\framebox{
\begin{pspicture}(-0.1,-3.2)(11.6,5.1)
\psset{runit=0.1pt}
\psset{xunit=0.5pt}
\psset{yunit=0.5pt}
\rput[bl]{0}(-8,85){\psframebox*{{Three-barrel}}}
\rput[bl]{0}(-8,10){\psframebox*{{Community}}}
\rput[bl]{0}(-8,-60){\psframebox*{{Goat\&C.}}}
\rput[bl]{0}(-8,-130){\psframebox*{{Tile-puzzle}}}
\rput[bl]{0}(135,150){\psframebox*{\standup{{\sc $\mathcal{B}$-Best-ASP}}}}
\rput[bl]{0}(192,150){\psframebox*{\standup{{\sc $\mathcal{B}$-CLP(FD)}}}}
\rput[bl]{0}(246,150){\psframebox*{\standup{{\sc \BMV-CLP(FD)}}}}
\rput[bl]{0}(319,150){\psframebox*{\standup{{\sc $\mathcal{B}$-Best-ASP}}}}
\rput[bl]{0}(375,150){\psframebox*{\standup{{\sc $\mathcal{B}$-CLP(FD)}}}}
\rput[bl]{0}(431,150){\psframebox*{\standup{{\sc \BMV-CLP(FD)}}}}
\rput[bl]{0}(502,150){\psframebox*{\standup{{\sc $\mathcal{B}$-Best-ASP}}}}
\rput[bl]{0}(558,150){\psframebox*{\standup{{\sc $\mathcal{B}$-CLP(FD)}}}}
\rput[bl]{0}(612,150){\psframebox*{\standup{{\sc \BMV-CLP(FD)}}}}
\qline(114,-189)(114,297)
\rput[bl]{0}(140,-175){\psframebox*{{``No'' instances}}}
\qline(295,-189)(295,297)
\rput[bl]{0}(320,-175){\psframebox*{{``Yes'' instances}}}
\qline(482,-189)(482,297)
\rput[bl]{0}(520,-175){\psframebox*{{All instances}}}
\qline(-16,137)(667,137)
 \pscircle[fillstyle=solid,fillcolor=g00](149,94){0}
 \pscircle[fillstyle=solid,fillcolor=g00](205,94){50}
 \pscircle[fillstyle=solid,fillcolor=g00](260,94){100}
 \pscircle[fillstyle=solid,fillcolor=g00](149,24){4.2}
 \pscircle[fillstyle=solid,fillcolor=g00](205,24){45.8}
 \pscircle[fillstyle=solid,fillcolor=g00](260,24){95.8}
 \pscircle[fillstyle=solid,fillcolor=g00](149,-50){5.6}
 \pscircle[fillstyle=solid,fillcolor=g00](205,-50){44.4}
 \pscircle[fillstyle=solid,fillcolor=g00](260,-50){100}
 \pscircle[fillstyle=solid,fillcolor=g00](149,-116){75}
 \pscircle[fillstyle=solid,fillcolor=g00](205,-116){8.3}
 \pscircle[fillstyle=solid,fillcolor=g00](260,-116){66.7}
 \pscircle[fillstyle=solid,fillcolor=g00](331,94){0}
 \pscircle[fillstyle=solid,fillcolor=g00](388,94){50}
 \pscircle[fillstyle=solid,fillcolor=g00](448,94){100}
 \pscircle[fillstyle=solid,fillcolor=g00](331,24){8.3}
 \pscircle[fillstyle=solid,fillcolor=g00](388,24){41.7}
 \pscircle[fillstyle=solid,fillcolor=g00](448,24){100}
 \pscircle[fillstyle=solid,fillcolor=g00](331,-50){7.1}
 \pscircle[fillstyle=solid,fillcolor=g00](388,-50){57.1}
 \pscircle[fillstyle=solid,fillcolor=g00](448,-50){85.7}
 \pscircle[fillstyle=solid,fillcolor=g00](331,-116){25}
 \pscircle[fillstyle=solid,fillcolor=g00](388,-116){37.5}
 \pscircle[fillstyle=solid,fillcolor=g00](448,-116){87.5}
 \pscircle[fillstyle=solid,fillcolor=g00](515,94){0}
 \pscircle[fillstyle=solid,fillcolor=g00](570,94){50}
 \pscircle[fillstyle=solid,fillcolor=g00](627,94){100}
 \pscircle[fillstyle=solid,fillcolor=g00](515,24){5.6}
 \pscircle[fillstyle=solid,fillcolor=g00](570,24){44.4}
 \pscircle[fillstyle=solid,fillcolor=g00](627,24){97.2}
 \pscircle[fillstyle=solid,fillcolor=g00](515,-50){6.3}
 \pscircle[fillstyle=solid,fillcolor=g00](570,-50){50}
 \pscircle[fillstyle=solid,fillcolor=g00](627,-50){93.8}
 \pscircle[fillstyle=solid,fillcolor=g00](515,-116){55}
 \pscircle[fillstyle=solid,fillcolor=g00](570,-116){20}
 \pscircle[fillstyle=solid,fillcolor=g00](627,-116){75}
\end{pspicture}}
\end{center}
\caption{\label{pallogramma}Relative performance of
the solvers for each set of instances (the radii of the circles are proportional
to the performance of the specific solver).}
\end{table}

The success of the constraint-based approach is evident.
However, it is interesting to observe that the planning problems that do not make
significant use of non-boolean fluents tend to perform better in the ASP-based
implementations---possibly due to the greater efficiency of ASP solvers in propagating
boolean knowledge during search for a solution.
Conversely, when numerical quantities are relevant in modeling a planning problem,
the use of multi-valued fluents and constraints not only reduces the modeling effort,
yielding more concise formalizations, but also requires a smaller number
of fluents (compared with the analogous Boolean encoding).
This, combined with the use of constraints, often translates into a smaller state space
to be explored in finding a solution. These seem to be  the main reasons for the
better behavior provided by the \BMV\ approach.

The distinction between \emph{``Yes''} and \emph{``No''}
instances is also very relevant. The CLP-based solvers tend to perform better
on the \emph{``Yes''} instances, especially for large instances.
It is interesting to observe
that a similar behavior has been observed  in recent
studies comparing performance of ASP and CLP solutions to combinatorial
problems~\cite{DFPiclp05,NECTAR07,JETAI08}.

\section{Related work}\label{sec:related}

The literature on planning and planning domain
description languages is extensive, and it would be
impossible to summarize it all in this context.
We focus our discussion and comparison to the papers that present languages
and techniques similar to ours.

\medskip

The language investigated in this work is a  variant
of the language $\mathcal{B}$ originally
introduced in~\cite{GL98}, as presented
in~\cite[Sect.~2]{SON01}.
Apart from minor syntactical differences, any action description $\mathcal{D}$
from the language of~\cite{SON01} can be embedded in our~$\mathcal{B}$.
The semantics for $\mathcal{B}$ presented here reproduces
the one of~\cite{GL98}.

\medskip

The language $\cal ADC$ has been introduced in \cite{BaralS02} to
model planning problems in presence of actions with duration and
delayed effects.
The language relies on multi-valued fluents, akin to those used in our
language. $\cal ADC$ actions have two types of effects:
\begin{enumerate}
\item Direct modification of fluent values, described by
    dynamic causal laws of the forms
    \begin{eqnarray}
        a  \textbf{ causes }  f = g(f,f_1,\dots,f_n,t)   \textbf{ from $t_1$ to $t_2$} \label{eq1}\\
        a  \textbf{ contributes }  g(f,f_1,\dots,f_n,t)   \textbf{ to $f$ from $t_1$ to $t_2$} \label{eq2}
    \end{eqnarray}
      The first axiom describes the value of the fluent $f$
        as a function, that modifies its value over the period of time from $t_1$ to $t_2$---these
            represent time units relative to the current point in time.
      The second axiom is similar, except that it denotes the quantity that should be added to the
            value of $f$ over the period of time.
    These axioms are important when describing actions whose effect has a known duration over time
        (i.e., the interval of length $t_2-t_1$).

\item Indirect modifications through the initiation and termination of \emph{processes}, that can
        modify fluents until explicitly stopped; the axioms involved are axioms for the
        creation and termination of processes:
        \begin{eqnarray}
            a_1  \textbf{ initiates }  p  \textbf{ from } t_1  \label{eq3}\\
            a_2  \textbf{ terminates }  p  \textbf{ at }  t_2 \label{eq4}
            \end{eqnarray}
        and axioms that describe how processes modify fluents
        \begin{eqnarray}
            p  \textbf{ is\_associated\_with }  f=g(f,f_1,\dots,f_n,t) \label{eq5}\\
            p  \textbf{ is\_associated\_with }  f \leftarrow g(f,f_1,\dots,f_n,t) \label{eq6}
        \end{eqnarray}
        The first axiom describes how the value of
            the fluent $f$ will change as a function of time once a process is started;
        the second axiom determines how the value of $f$ changes while the process $p$ is active.
\end{enumerate}
$\cal ADC$ has some similarities to \BMV; they both allow multi-valued fluents
and some forms of temporal references. \BMV\ has the flexibility of allowing
non-Markovian behavior and it allows references to values of fluents at different time
points, features that are missing in $\cal ADC$. On the other hand, $\cal ADC$ allows the representation
of  continuous time and the ability to describe continuous changes to the value of fluents.

Several features of $\cal ADC$ can be reasonably simulated in \BMV; we will focus on
the axioms of type (\ref{eq3})--(\ref{eq6}), since these subsume the capabilities of
axioms (\ref{eq1}) and (\ref{eq2}):
\begin{itemize}
\item we can represent each process $p$ using  a corresponding fluent;
\item the axioms (\ref{eq3}) and (\ref{eq4}) can be simulated by
\[\begin{array}{lcr}
   \codetext{causes}(a_1,p^{t_1-1}=1,\codetext{true}) & \hspace{.5cm} &
        \codetext{causes}(a_2, p^{t_2-1}=0,\codetext{true})
    \end{array}\]
\item the axiom (\ref{eq5}) can be simulated by introducing the static causal law
    \[ \codetext{caused}(p>0, f=g(f^{-1},f_1^{-1}, \dots, f_n^{-1},p^{-1}) \wedge p = p^{-1}+1)\]
\end{itemize}
Note that, due to the inability of \BMV\ to handle continuous time,
we are considering only discrete time measures.

\medskip

The language ${\cal C}^+$ proposed in \cite{Giun} also has some similarities
to the language \BMV.  ${\cal C}^+$ does not offer
capabilities for non-Markovian and temporal references, but supports multi-valued
fluents. The syntax of ${\cal C}^+$ builds on a language of fluent constants (each
with an associated domain) and action names (viewed as Boolean variables):
\begin{itemize}
\item Static causal laws
\[ \textbf{caused } F \textbf{ if } G \]
where $F$ and $G$ are fluent formulae (i.e., propositional combinations of
atoms of the form $f=v$ for $f$ fluent and $v \in \dom(f)$). The language
introduces syntactic restrictions that are effectively equivalent to preventing
cyclic dependencies among fluents. Static causal laws describe dependencies between
fluents within a state of the world.

\item Fluent dynamic laws
\[ \textbf{caused } F \textbf{ if } G \textbf{ after } H \]
where $F$ and $G$ are fluent formulae and $H$ is a formula that may also contain
action variables. The semantics of dynamic laws can be summarized as follows: if
$H$ holds in a state, then the implication $G \rightarrow F$ should hold in the
successive state.

\item Actions that can be freely generated are declared to be exogenous
\[ \textbf{exogenous } a\]

\item Fluents can be declared to be inertial (i.e., they satisfy the frame axiom)
\[ \textbf{inertial } f \]
\end{itemize}
The relationships between the two languages can be summarized as follows:
\begin{itemize}
\item ${\cal C}^+$ is restricted to non-cyclic dependencies among fluents, while
    \BMV\ lifts this restriction.
\item ${\cal C}^+$ is capable of identifying fluents as inertial or non-inertial,
    while \BMV\ focuses only on inertial fluents (though it is relatively
        simple to introduce an additional type of constraint to create non-inertial fluents).
\item ${\cal C}^+$ can describe domains where concurrent actions are
    allowed---by allowing occurrences of different action variables in the $H$
        component of the fluent dynamic laws; although \BMV\ does not
        currently supports this feature, a similar extension has been investigated in
        a recent paper~\cite{DFPlpnmr09}.
\end{itemize}
Subsets of \BMV\ and ${\cal C}^+$ can be shown to have the same expressive
power; in particular, let us consider the subset of ${\cal C}^+$ that contains
only domains that meet the following requirements:
\begin{itemize}
\item there are no concurrent actions---i.e., each $H$ contains exactly one occurrence
of an action variable; thus
\[ \textbf{caused } F \textbf{ if } G \textbf{ after } a \wedge H \]
where $H$ is a fluent formula;
\item for each action $a$, there is a declaration
$$\textbf{exogenous } a.$$
\end{itemize}
Under these restrictions, it is possible to map a ${\cal C}^+$ domain $D$
to an equivalent domain in \BMV. In particular:
\begin{itemize}
\item for each non-inertial fluent $f$, with default value $v$,
    we  introduce   the static law
    \[ \codetext{caused}(f^{-1}\neq v, f^0 = v)\]
\item for each static causal law {\bf caused} $F$ {\bf if} $G$ we introduce
    a causal law $\codetext{caused}(G,F)$
\item for each fluent dynamic law $r$ of the form
$\textbf{caused } F \textbf{ if } G \textbf{ after } a \wedge H$,
we introduce the following axioms (where $exec\_r$ is a fresh fluent):
\[\begin{array}{l}
    \codetext{causes}(a, exec\_r=1, H)\\
    \codetext{causes}(a, exec\_r^{1}=0, H)\\
    \codetext{caused}(exec\_r=1 \wedge G, F)
\end{array}\]
\end{itemize}

\medskip

Logic programming, and more specifically Prolog, has been also used to
implement the first prototype of GOLOG (as discussed in \cite{gogol}). GOLOG
is a programming language for describing agents and their capabilities
of changing the state of the world. The language builds on the
foundations of situation calculus. It provides high level
constructs for the definition of complex actions and for the
introduction of control knowledge in the agent specification. Prolog
is employed to create an interpreter, which enables, for example, to
answer projection queries (i.e., determine the properties that hold in
a situation after the execution of a sequence of actions). The goals of GOLOG
and the use of logic programming in that work are radically different from
the focus of our work.

\medskip

The work by~\cite{thielscher} takes a different perspective in using constraint programming
to handle problems in reasoning about actions and change. Thielscher's work builds on the
use of Fluent Calculus~\cite{fluentcalc} for the representation of actions and their effects.
Fluent calculus views states as sets of fluents, constructed using an operator $\circ$, and with
the ability to encode partially specified sets (e.g., $f_1 \circ f_2 \circ Z$ where $Z$ represents
the ``rest'' of the state).
In~\cite{thielscher}, an encoding of the fluent calculus axioms using Constraint Handling Rules (CHRs)
is presented; the encoding uses \emph{lists} to represent states, and it employs CHRs to explicitly
implement the operations on lists required to operate on states---e.g., truth or falsity of a fluent,
validation of disjunctions of fluents. The ability to code open lists enables reasoning with
incomplete knowledge. Experimental results (reported in~\cite{experfc}) denote a
good performance with respect to GOLOG. The framework is very suitable for dealing with incomplete
knowledge and sensing actions. Differently from our framework, it does not support non-Markovian
reasoning, multi-valued reasoning, and it does not bring the expressiveness of constraint programming
to the level of the action specification language. The use of constraints in the two approaches is
radically different---Thielscher's work develops new constraint solvers to implement reasoning about
states, while we use existing solvers as black boxes.

\medskip

A strong piece of work regarding the use of constraint programming in planning
is~\cite{vidalgeffner}. The authors use constraint programming, based on the
CLAIRE language~\cite{claire}, to encode temporal planning problems and
to search for minimal plans.
They also use a series of interesting heuristics for solving that problem.
This line of research is more accurate than ours from the implementation point
of view---although their heuristic strategies can be implemented in our system
and it would be interesting to exploit them during the labeling phase.
On the other hand, the proposal by Vidal and Geffner only deals with Boolean fluents and without
explicitly defined static causal laws.

Similar considerations can be done with respect to
the cited proposal by Lopez and Bacchus~\cite{lopezbacchus}.
The authors start from Graphplan and exploit constraints to encode \mbox{$k$-plan} problems.
Fluents are in this case only Boolean (not multi-valued) and the process is deterministic
once an action is chosen (instead, we  deal also with non-determinism, e.g.,
when we have consequences such as $f>5$). The proposal of Lopez and Bacchus does not
address the encoding of static causal laws.

\section{Conclusions and Future Work}\label{sec:endofit}

In this paper, we investigated the application
of constraint logic programming technology to the problem of
reasoning about actions and change and planning.
In particular,
we presented a modeling of the action language $\mathcal{B}$ using
constraints, developed an  implementation using
CLP(FD), and reported on its performance. We also presented the
action language \BMV, which allows the use of multi-valued fluents
and the use of constraints as conditions and consequences of
actions. Once again, the use of constraints is instrumental in
making these extensions possible.
We illustrated the application of both $\mathcal{B}$ and  \BMV\ to
several planning problems.
Both languages have been implemented using SICStus Prolog.

We consider the research and the results discussed in this
paper as a preliminary step in a very promising direction.
The experimental results, as well as the elegance of the encodings
of complex problems, shows the promise of constraint-based technology
to address the needs of complex planning domains. A number of
research directions are currently being pursued:
\begin{itemize}
\item we have introduced the use of global constraints to encode
        different forms of preferences (e.g., action costs)
        and control knowledge. Global constraints have been widely
        used in constraint programming to enhance efficiency, by
        providing more effective constraint propagations between
        sets of variables; we believe a similar use of global constraints
        can be introduced in the context of planning---e.g., the
        use of techniques used to efficiently handle the
        \codetext{alldifferent} global constraint to enforce non-repetition
        of states in a trajectory.

\item We also believe that significant improvements in efficiency
can be achieved by delegating parts of the constraint
solving process to an efficient dedicated solver (e.g., encoded using
a constraint platform such as GECODE, possibly enhanced with local search
moves).

\item
The encoding in CLP(FD) allow us to think of extensions
in several directions, such as the encoding of qualitative and
quantitative preferences (a preliminary study has been
presented in~\cite{TuSP07}),
and the use of
constraints to represent incomplete states---e.g., to determine
most general conditions for the existence of a plan and to
support conformant planning~\cite{sonconf}.

\item
An interesting line of research is represented by the application of the
approach discussed here to multi-agent systems. In that case, besides admitting
the execution of more that one action in each state transition
(cf., Remark~\ref{sequentiality}), other important issues have to be
addressed, since different agents may compete or collaborate
in order to reach the desired results. For instance, concurrency of actions
may be subject to constraints to model incompatibilities or interdependencies
among the occurrences/effects of different actions executed by different agents
(even in different points in time).
Hence, the action description language, as well as its CLP encoding,
has to be suitably enriched in order to deal with these aspects.
A first step in this direction has been presented in~\cite{DFPlpnmr09}.
\end{itemize}

\subsection*{Acknowledgments}
The authors would like to thank the following
researchers for their help, comments, and suggestions:
Son Cao Tran, Michael Gelfond, and the anonymous reviewers of ICLP 2007
and TPLP.

The research has been partially supported by
NSF Grants IIS-0812267,
HRD-0420407, and CNS-0220590,
by the FIRB grant RBNE03B8KK, and by
GNCS---\emph{Gruppo Nazionale per il Calcolo Scientifico}
(project \emph{Tecniche innovative per la programmazione con vincoli in applicazioni strategiche}).

\bibliographystyle{acmtrans}

\newpage

\appendix

\section{Some of the codes of the Experimental Section}

\subsection{The Three-Barrel Problem: \BMV\ description of the 12-7-5 barrels problem}
\label{MV_Barrels}

The \BMV\ encoding of the three barrels planning problem for $N=12$.
(\Fig{Bool_Barrels} presents an encoding using the language~$\mathcal{B}$.)

\begin{center}
\begin{minipage}[t]{1.0\textwidth}
\begin{codicenonum}{0}
barrel(5).\\
barrel(7).\\
barrel(12).\\
\\
fluent(cont(B),0,B) :- barrel(B).\\
\\
action(fill(X,Y)) :- barrel(X), barrel(Y), neq(X,Y).\\
\\
causes(fill(X,Y), cont(X) eq 0, [Y-cont(Y) geq cont(X)]) :-\\
\tab        action(fill(X,Y)).\\
causes(fill(X,Y), cont(Y) eq cont(Y)\^{}(-1)+cont(X)\^{}(-1),
\\\tab\tab
 [Y-cont(Y) geq cont(X)]) :-\\
\tab action(fill(X,Y)).\\
causes(fill(X,Y), cont(Y) eq Y, [Y-cont(Y) lt cont(X)]) :-\\
\tab        action(fill(X,Y)).\\
causes(fill(X,Y), cont(X) eq cont(X)\^{}(-1)-Y+cont(Y)\^{}(-1),
\\\tab\tab
 [Y-cont(Y) lt cont(X)]) :-\\
\tab        action(fill(X,Y)).\\
\\
executable(fill(X,Y), [cont(X) gt 0, cont(Y) lt Y]) :-\\
\tab action(fill(X,Y)).\\
\\
caused([], cont(12) eq 12-cont(5)-cont(7)).\\
\\
initially(cont(12) eq 12).\\
\\
goal(cont(12) eq  cont(7)).
\end{codicenonum}
\end{minipage}
\end{center}
\vfill

\newpage

\subsection{The HP Protein Folding Problem}
\label{PF-encoding}
\BMV\ encoding of the HP-protein folding problem with pivot moves
on input of the form 1001001001\dots\ starting from a vertical straight line.

\begin{center}
\begin{minipage}[t]{1.0\textwidth}
\begin{codicenonum}{0}
length(10). \\
amino(A) :-
   length(N), interval(A,1,N).\\
direction(clock). \\
direction(antick).\\
\\
fluent(x(A),1,M) :-\\
\tab   length(N), M is 2*N, amino(A).\\
fluent(y(A),1,M) :-\\
\tab   length(N), M is 2*N, amino(A).\\
fluent(type(A),0,1) :- \\
\tab amino(A).\\
fluent(saw,0,1).\\
\\
action(pivot(A,D)) :- \\
\tab length(N), amino(A),\\
\tab  1<A, A<N, direction(D).\\
\\
executable(pivot(A,D), []) :-
    action(pivot(A,D)).\\
\\
causes(pivot(A,clock), x(B) eq x(A)\^{}(-1)+y(B)\^{}(-1)-y(A)\^{}(-1), []) :-\\
\tab        action(pivot(A,clock)), amino(B), B > A.\\
causes(pivot(A,clock), y(B) eq y(A)\^{}(-1)+x(A)\^{}(-1)-x(B)\^{}(-1), []) :-\\
\tab        action(pivot(A,clock)), amino(B), B > A.\\
causes(pivot(A,antick), x(B) eq x(A)\^{}(-1)-y(B)\^{}(-1)+y(A)\^{}(-1), []) :-\\
\tab        action(pivot(A,antick)), amino(B), B > A.\\
causes(pivot(A,antick), y(B) eq y(A)\^{}(-1)-x(A)\^{}(-1)+x(B)\^{}(-1), []) :-\\
\tab        action(pivot(A,antick)), amino(B), B > A.\\
\\
caused([x(A) eq x(B), y(A) eq y(B)], saw eq 0) :-\\
\tab         amino(A), amino(B), A < B.\\
\\
initially(saw eq 1). \\
initially(x(A) eq  N) :-
     length(N), amino(A).\\
initially(y(A) eq  Y) :-
    length(N), amino(A), Y is N+A-1.\\
initially(type(X) eq 1) :-
    amino(X), X mod 3 =:= 1.\\
initially(type(X) eq 0) :-
    amino(X), X mod 3 ={\tt$\backslash$}= 1.\\
\\
goal(saw gt 0).\\
\\
state\_cost(FE) :-
    length(N),
    auxc(1,4,N,FE).\\
auxc(I,J,N,0) :-
    I > N-3,!.\\
auxc(I,J,N,FE) :- J > N, !, I1 is I+1,\\
\tab     J1 is I1+3, auxc(I1,J1,N,FE).\\
auxc(I,J,N,FE1+type(I)*type(J)*rei(abs(x(I)-x(J))+abs(y(I)-y(J)) eq 1)) :-\\
\tab    J1 is J+2, auxc(I,J1,N,FE1).\\
\\
always(x(1) eq 10).    always(y(1) eq 10).\\
always(x(2) eq 10).    always(y(2) eq 11).\\
\\
cost\_constraint(goal geq 4).
\end{codicenonum}
\end{minipage}
\end{center}

\vfill
\newpage

\subsection{The Community Problem}

\subsubsection{$\mathcal{B}$ description of the instance A$_4$}
\label{BCommunity}

\begin{center}
\begin{minipage}[t]{.90\textwidth}
\begin{codicenonum}{0}
max\_people(4).\\
person(X) :- max\_people(N), interval(X,1,N).\\
money(X) :- max\_people(N), M is N*(N+1), interval(X,1,M).\\
\\
fluent(owns(B,M)) :- person(B), money(M).\\
\\
action(gives(X,Y)) :-  \\
\tab    person(X), person(Y), neq(X,Y). \\
\\
executable(gives(X,Y), [owns(X,Mx)]) :-  \\
\tab    action(gives(X,Y)),\\
\tab    fluent(owns(X,Mx)), Mx > X.\\
\\
causes(gives(X,Y), owns(X,NewMx), [owns(X,Mx)]) :-\\
\tab    action(gives(X,Y)), money(Mx),\\
\tab    fluent(owns(X,NewMx)), fluent(owns(X,Mx)),\\
\tab    NewMx is Mx-X.\\
causes(gives(X,Y), owns(Y,NewMy), [owns(Y,My)]) :-\\
\tab    action(gives(X,Y)), money(My),\\
\tab    fluent(owns(Y,NewMy)), fluent(owns(Y,My)), \\
\tab    NewMy is My+X.\\
\\
caused([owns(X,Mx)], neg(owns(X,Other))) :-\\
\tab    fluent(owns(X,Mx)),
        fluent(owns(X,Other)),\\
\tab    person(X), money(Mx), money(Other), neq(Mx,Other).\\
\\
initially(owns(X,M)) :-\\
\tab    person(X), M is 2*X.\\
\\
goal(owns(X,Mid)) :-\\
\tab    person(X), max\_people(N), Mid is (N*(N+1))//N.
\end{codicenonum}
\end{minipage}
\end{center}
\vfill

\subsubsection{\BMV\ description of the instance A$_4$}
\label{MVCommunity}

\begin{center}
\begin{minipage}[t]{1.0\textwidth}
\begin{codicenonum}{0}
max\_people(4).  \\
person(X) :- max\_people(N), interval(X,1,N).\\
\\
fluent(owns(B),1,M) :-\\
\tab    person(B), max\_people(N), M is N*(N+1).\\
\\
action(gives(X,Y)) :-  \\
\tab    person(X), person(Y), neq(X,Y). \\
\\
executable(gives(X,Y), [owns(X) gt X]) :-  \\
\tab    action(gives(X,Y)).\\
\\
causes(gives(X,Y), owns(X) eq owns(X)\^{}(-1)-X, []) :-\\
\tab    action(gives(X,Y)).\\
causes(gives(X,Y), owns(Y) eq owns(Y)\^{}(-1)+X, []) :-\\
\tab    action(gives(X,Y)).\\
\\
initially(owns(X) eq M) :-\\
\tab    person(X), M is 2*X.\\
\\
goal(owns(X) eq Mid) :-\\
\tab    person(X), max\_people(N), Mid is (N*(N+1))//N.
\end{codicenonum}
\end{minipage}
\end{center}

\vfill

\newpage
\subsection{The 8-Tile Puzzle Problem}

\subsubsection{$\mathcal{B}$ description of the instance I$_1$}
\label{Btilepuzzle}

\begin{center}
\begin{minipage}[t]{1.0\textwidth}
\begin{codicenonum}{0}
cell(X) :- interval(X,1,9).\\
val(X) :- interval(X,1,9), neq(X,3).\\
near(1,2). ~near(1,4).\\
near(2,1). ~near(2,3). ~near(2,5).\\
near(3,2). ~near(3,6).\\
near(4,1). ~near(4,5). ~near(4,7).\\
near(5,2). ~near(5,4). ~near(5,6). ~near(5,8).\\
near(6,3). ~near(6,5). ~near(6,9).\\
near(7,4). ~near(7,8).\\
near(8,5). ~near(8,7). ~near(8,9).\\
near(9,6). ~near(9,8).
\\
\\
fluent(at(X,Y)) :- val(X), cell(Y).\\
fluent(free(Y)) :- cell(Y).\\
\\
action(move(X,Y)) :- val(X), cell(Y).\\
\\
executable(move(X,Y), [at(X,Z), free(Y)]) :- \\
\tab  val(X), cell(Y), cell(Z), near(Z,Y).\\
\\
causes(move(X,Y), at(X,Y), []) :- \\
\tab  val(X), cell(Y).\\
causes(move(X,Y), free(Z), [at(X,Z)]) :- \\
\tab  val(X), cell(Y), cell(Z).\\
caused([at(X,Y)], neg(free(Y))) :-\\
\tab   val(X), cell(Y).\\
\\
caused([at(X,Y)], neg(at(X,Z))) :-\\
\tab   val(X), cell(Y), cell(Z), neq(Y,Z).\\
caused([at(X,Y)], neg(at(W,Y))) :-\\
\tab   val(X), val(W), cell(Y), neq(X,W).
\\
\\
initially(at(1,1)). initially(at(2,3)). initially(at(4,8)).\\
initially(at(5,2)). initially(at(6,9)). initially(at(7,4)).\\
initially(at(8,6)). initially(at(9,7)). initially(free(5)).\\
initially(neg(at(1,X))) :- cell(X), neq(X,1).\\
initially(neg(at(2,X))) :- cell(X), neq(X,3).\\
initially(neg(at(4,X))) :- cell(X), neq(X,8).\\
initially(neg(at(5,X))) :- cell(X), neq(X,2).\\
initially(neg(at(6,X))) :- cell(X), neq(X,9).\\
initially(neg(at(7,X))) :- cell(X), neq(X,4).\\
initially(neg(at(8,X))) :- cell(X), neq(X,6).\\
initially(neg(at(9,X))) :- cell(X), neq(X,7).\\
initially(neg(free(X))) :- cell(X), neq(X,5).\\
\\
goal(at(X,X)) :- val(X).\\
goal(free(3)).\\
\end{codicenonum}
\end{minipage}
\end{center}
\vfill

\subsubsection{\BMV\ description of the instance I$_1$}
\label{MVtilepuzzle}

\begin{center}
\begin{minipage}[t]{1.0\textwidth}
\begin{codicenonum}{0}
cell(X) :- interval(X,1,9).\\
tile(X) :- interval(X,1,9), neq(X,3).\\
near(1,2). ~near(1,4).\\
\tab ...\%as for $\mathcal{B}$...\\
near(9,6). ~near(9,8).
\\
fluent(at(X),1,9) :- tile(X).\\
fluent(free,1,9).\\
\\
action(move(X,Y)) :- cell(Y), tile(X).\\
\\
executable(move(X,Y), [at(X) eq Z, free eq Y]) :-\\
\tab  tile(X), cell(Y), near(Z,Y).\\
causes(move(X,Y), at(X) eq Y, []) :- \\
\tab  tile(X), cell(Y).\\
causes(move(X,Y), free eq at(X)\^{}(-1), []) :-\\
\tab  tile(X), cell(Y).
\\
\\
initially(at(1) eq 1). initially(at(2) eq 3).\\
initially(at(4) eq 8). initially(at(5) eq 2).\\
initially(at(6) eq 9). initially(at(7) eq 4).\\
initially(at(8) eq 6). initially(at(9) eq 7).\\
initially(free eq 5).
\\
goal(at(X) eq X) :- tile(X).\\
goal(free eq 3).\\
\end{codicenonum}
\end{minipage}
\end{center}

\vfill

\newpage

\subsection{The Wolf-Goat-Cabbage Problem}

\subsubsection{$\mathcal{B}$ description of the  Wolf-goat-cabbage problem}
\label{Bgoat}

\begin{center}
\begin{minipage}[t]{1.0\textwidth}
\begin{codicenonum}{0}
obj(goat). \\
obj(cabbage). \\
obj(wolf). \\
obj(man).\\
side(left). side(right).\\
pos(X) :- side(X).  \\
pos(boat).\\
\\
fluent(is\_in(X,Y)) :- obj(X), pos(Y).\\
fluent(boat\_at(Y)) :- side(Y).\\
fluent(alive).\\
\\
action(sail(A,B)) :-  side(A), side(B), neq(A,B).\\
action(go\_aboard(A)) :- obj(A).\\
action(get\_off(A)) :- obj(A).\\
\\
executable(sail(A,B), [boat\_at(A), is\_in(man,boat)]) :- \\
\tab   side(A), side(B), neq(A,B).\\
executable(go\_aboard(A), [boat\_at(L), is\_in(A,L)]) :- \\
\tab   obj(A), side(L).\\
executable(get\_off(A), [is\_in(A,boat)]) :- \\
\tab   obj(A).\\
\\
causes(sail(A,B), boat\_at(B), []) :- \\
\tab   side(A), side(B), neq(A,B).\\
causes(go\_aboard(A), is\_in(A,boat), []) :-\\
\tab   obj(A).\\
causes(get\_off(A), is\_in(A,L), [boat\_at(L)]) :-\\
\tab   obj(A), side(L).\\
\\
caused([is\_in(Ogg,L1)], neg(is\_in(Ogg,L2))) :-\\
\tab      obj(Ogg), pos(L1), pos(L2), neq(L1,L2).  \\
caused([boat\_at(L1)], neg(boat\_at(L2))) :-\\
\tab      side(L1), side(L2), neq(L1,L2). \\
caused([is\_in(A,boat), is\_in(B,boat)], neg(alive)) :-\\
\tab      obj(A), obj(B), diff(A,B,man).\\
caused([is\_in(wolf,L), is\_in(goat,L), neg(is\_in(man,L))], neg(alive)) :-\\
\tab pos(L).\\
caused([is\_in(cabbage,L), is\_in(goat,L), neg(is\_in(man,L))], neg(alive)) :-\\
\tab pos(L).
\\
\\
initially(is\_in(A,left)) :- obj(A).\\
initially(alive).\\
initially(boat\_at(left)).\\
\\
goal(is\_in(A,right)) :- obj(A).\\
goal(alive).
\end{codicenonum}
\end{minipage}
\end{center}

\vfill

\subsubsection{\BMV\ description of the  Wolf-goat-cabbage problem}
\label{MVgoat}

\begin{center}
\begin{minipage}[t]{1.0\textwidth}
\begin{codicenonum}{0}
obj(goat). \\
obj(cabbage). \\
obj(wolf). \\
obj(man).\\
\% 0=boat,  1=on-the-left,  2=on-the-right:\\
\\
fluent(is\_in(X),0,2) :- obj(X).\\
fluent(boat\_at,1,2).\\
fluent(alive,0,1).\\
\\
action(sail).\\
action(go\_aboard(A)) :- obj(A).\\
action(get\_off(A)) :- obj(A).\\
\\
executable(sail, [is\_in(man) eq 0]).\\
executable(go\_aboard(A), [boat\_at eq is\_in(A)]) :- \\
\tab       obj(A).\\
executable(get\_off(A), [is\_in(A) eq 0]) :- \\
\tab       obj(A).\\
\\
causes(sail, boat\_at eq 1, [boat\_at eq 2]).\\
causes(sail, boat\_at eq 2, [boat\_at eq 1]).\\
causes(go\_aboard(A), is\_in(A) eq 0, []) :- \\
\tab obj(A).\\
causes(get\_off(A), is\_in(A) eq boat\_at\^{}(-1), []) :-\\
\tab obj(A).\\
\\
caused([is\_in(A) eq 0, is\_in(B) eq 0], alive eq 0) :-\\
\tab  obj(A), obj(B), diff(A,B,man).\\
caused([is\_in(wolf) eq is\_in(goat),
\\\tab\tab
is\_in(man) neq is\_in(wolf)], alive eq 0).\\
caused([is\_in(cabbage) eq is\_in(goat),
\\\tab\tab
 is\_in(man) neq is\_in(cabbage)], alive eq 0).
\\
\\
initially(is\_in(A) eq 1) :- obj(A).\\
initially(boat\_at eq 1).\\
initially(alive eq 1).\\
\\
goal(is\_in(A) eq 2) :- obj(A).\\
goal(alive eq 1).
\end{codicenonum}
\end{minipage}
\end{center}

\newpage

\subsection{The Gas-diffusion Problem: \BMV\ description of the instance A$_4$}
\label{GasDiffusion}

\begin{center}
\begin{minipage}[t]{1.0\textwidth}
\begin{codicenonum}{0}
room(N) :- interval(N,1,11).\\
gate(1,2).\\
gate(1,7).\\
gate(1,11).\\
gate(2,3).\\
gate(3,4).\\
gate(4,5).\\
gate(5,6).\\
gate(6,7).\\
gate(6,8).\\
gate(8,9).\\
gate(9,10).\\
gate(10,11).\\
\\
fluent(contains(N),0,255) :- room(N).\\
fluent(is\_open(X,Y),0,1) :- gate(X,Y).\\
\\
action(open(X,Y)) :- gate(X,Y).\\
action(close(X,Y)) :- gate(X,Y).\\
\\
executable(open(X,Y),L) :-\\
\tab    action(open(X,Y)),\\
\tab    findall((is\_open(X,Z) eq 0), gate(X,Z),L1),\\
\tab    findall((is\_open(Z,X) eq 0), gate(Z,X),L2),\\
\tab    findall((is\_open(Y,Z) eq 0), (gate(Y,Z),neq(Z,X)),L3),\\
\tab    findall((is\_open(Z,Y) eq 0), (gate(Z,Y),neq(Z,X)),L4),\\
\tab    append(L1,L2,La),append(L3,L4,Lb),append(La,Lb,L).\\
executable(close(X,Y), [is\_open(X,Y) eq 1]) :-\\
\tab     action(close(X,Y)).\\
\\
causes(open(X,Y),
\\ \tab\tab
contains(Y) eq (contains(X)\^{}(-1)+contains(Y)\^{}(-1))/2,
\\ \tab\tab
[]) :-\\
\tab action(open(X,Y)).\\
causes(open(X,Y),
\\ \tab\tab
contains(X) eq (contains(X)\^{}(-1)+contains(Y)\^{}(-1))/2,
\\ \tab\tab
[]) :-\\
\tab action(open(X,Y)).\\
causes(open(X,Y), is\_open(X,Y) eq 1, []) :-\\
\tab     action(open(X,Y)).\\
causes(close(X,Y), is\_open(X,Y) eq 0, []) :-\\
\tab     action(close(X,Y)).\\
\\
initially(is\_open(X,Y) eq 0) :- gate(X,Y).\\
initially(contains(10) eq 128). \\
initially(contains(3) eq 128).  \\
initially(contains(A) eq 0) :- room(A), diff(A,3,10).\\
\\
goal(contains(1) gt 50).
\end{codicenonum}
\end{minipage}
\end{center}

\end{document}